%% file: main.tex
\documentclass[english,american]{article}

\usepackage[a4paper, 
includeheadfoot, margin = 2.54cm]{geometry}

\input{preamble}

\input{math_commands}

\begin{document}
\title{The surrogate Gibbs-posterior of a corrected stochastic MALA: Towards uncertainty quantification for neural networks}
\author{\!\!\!Sebastian Bieringer$^{1,*}$\!, Gregor Kasieczka$^{1,*}$\!, Maximilian F.\ Steffen$^{2,*}$ and Mathias
Trabs$^{2,}$\thanks{The authors would like to thank Botond Szab\'o and two anonymous referees for helpful comments.
SB is supported by DASHH (Data Science in Hamburg - HELMHOLTZ Graduate School for the Structure of Matter) with the grant HIDSS-0002.
SB and GK acknowledge support by the Deutsche Forschungsgemeinschaft (DFG) under Germany's Excellence Strategy - EXC 2121 Quantum Universe - 390833306. MS and MT acknowledge  support by the DFG through project TR 1349/3-1.
The empirical studies were enabled by the Maxwell computational resources operated at Deutsches Elektronen-Synchrotron DESY, Hamburg, Germany.
}}
\date{$^{1}$Universit\"at Hamburg and $^{2}$Karlsruhe Institute of Technology}
\maketitle

\begin{abstract}
\noindent 
MALA is a popular gradient-based Markov chain Monte Carlo method to access the Gibbs-posterior distribution. Stochastic MALA (sMALA) scales to large data sets, but changes the target distribution from the Gibbs-posterior to a surrogate posterior which only exploits a reduced sample size. We introduce a corrected stochastic MALA (csMALA) with a simple correction term for which distance between the resulting surrogate posterior and the original Gibbs-posterior decreases in the full sample size while retaining scalability. In a nonparametric regression model, we prove a PAC-Bayes oracle inequality for the surrogate posterior. Uncertainties can be quantified by sampling from the surrogate posterior. Focusing on Bayesian neural networks, we analyze the diameter and coverage of credible balls for shallow neural networks and we show optimal contraction rates for deep neural networks. Our credibility result is independent of the correction and can also be applied to the standard Gibbs-posterior.
A simulation study in a high-dimensional parameter space demonstrates that an estimator drawn from csMALA based on its  surrogate Gibbs-posterior indeed exhibits these advantages in practice.
\end{abstract}


\smallskip{}


\section{Introduction\label{sec:intro}}

An essential feature in modern data science, especially in machine learning as well as high-dimensional statistics, are large sample sizes and large parameter space dimensions. As a consequence, the design of methods for uncertainty quantification is characterized by a tension between numerically feasible and efficient algorithms and approaches which satisfy theoretically justified statistical properties. Motivated by this tension, we introduce a simple correction to stochastic MALA to achieve both: the method is scalable, \emph{i.e.}, it is computationally feasible for large samples, and we can prove an optimal bound for the prediction risk as well as uncertainty statements for the underlying posterior distribution. While the focus of this work is on the statistical behavior of the posterior distribution of the corrected stochastic MALA (csMALA), our simulation study demonstrates the practical benefits of our algorithm.

\medskip{}

Bayesian methods enjoy high popularity for quantifying uncertainties in complex models. The classical approach to sample from the posterior distribution are Markov Chain Monte Carlo methods (MCMC). For large parameter spaces gradient-based Monte Carlo methods are particularly useful with, \emph{e.g.},  Langevin dynamics serving as a prototypical example. State-of-the-art methods such as Metropolis adjusted Langevin (MALA) \cite{Besag1994,roberts1996b} and Hamiltonian Monte Carlo \cite{duaneEtAl1987,neal2011} equip a Metropolis-Hastings (MH) step to accept or reject the proposed next state of the chain. From the practical point of view, the MH step improves robustness with respect to the choice of the tuning parameters and in theory MH speeds up the convergence of the Markov chain.

If the sample size is large, the computational costs of gradient-based MCMC methods can be reduced by replacing the gradient of the full loss over all observations by a stochastic gradient. This is standard in empirical risk minimization and has been successfully applied for Langevin dynamics as well \citep{ssgMCMC2022, Li2016SGLD, Patterson2013SGRLD, Welling2011bayesianSGLD}. In this case, the MH steps remain as a computational bottleneck: Since the target distribution depends on the full dataset, we have to compute the loss on the full sample to calculate the acceptance probabilities. Among the approaches to circumvent this problem, see \citet{BardenetEtAl2017} for a review, a \emph{stochastic MH} step is presumably the most natural one. There, the full loss in the acceptance probability is replaced by a (mini-)batch approximation which reduces the computational cost of the resulting \emph{stochastic MALA} (sMALA) considerably, see \citet{wu2022}. 

\citet[Section 6.1]{BardenetEtAl2017} have argued heuristically that the naive stochastic MH step reduces the effective sample size, which determines, for instance, contraction rates of the posterior distribution, to the size of the batch. To rigorously understand the statistical consequences of a stochastic MH step, we apply the pseudo-marginal Metro\-polis-Hastings perspective by \citet{AndrieuRoberts2009} and \citet{MaclaurinAdams2014}. It turns out that a Markov Chain with a stochastic MH step does not converge to the original target posterior distribution, but a different distribution, which we call \emph{surrogate posterior} and whose statistical performance is indeed determined to the batch size only. However, we show that there is a simple correction term in the risk such that the resulting stochastic MH chain converges to a surrogate posterior which achieves the full statistical power in terms of optimal contraction rates. In combination with the MALA methodology and stochastic gradients in the proposal distribution we obtain our \emph{corrected stochastic MALA} (csMALA).

\medskip{}

In a nonparametric regression problem under a quadratic loss and under classical assumptions, we investigate the distance of the surrogate posterior associated to the stochastic MH algorithm and the corrected stochastic MH algorithm to the original posterior distribution in terms of the Kullback-Leibler divergence. While these approximation results could be used to analyze the surrogate posteriors based on properties of the original posterior as done for variational Bayes methods, see \citet{raySzabo2022}, we will instead investigate the surrogate posteriors directly which will allow for sharp rates. Still, our bounds for the Kullback-Leibler divergences indicate that our correction of the surrogate posterior is beneficial since our corrected surrogate posterior is closer to the classical Gibbs-posterior than the surrogate posterior without the correction.

We prove oracle inequalities for the surrogate posteriors of the stochastic MH method and its corrected modification. Based on that we can conclude contraction rates as well as rates of convergence for the surrogate posterior mean. These findings reveal that indeed the surrogate posterior of csMALA has a high concentration around the true regression function compared to the surrogate posterior of sMALA. Moreover, we investigate the size and coverage of credible balls from the surrogate posterior. These results apply also to the original Gibbs-posterior as a special case, which might be of interest independent of the discussion of the surrogate posterior. 

We apply the oracle inequality and the credibility theorem in the context of shallow as well as deep neural networks. 
For shallow neural networks the oracle inequality yields optimal convergence rates and credible ball diameters for H\"older-regular functions up to a logarithmic factor. Due to the complex and non-invertible relationship between a Bayesian neural network and its parameters, establishing coverage guarantees remains a longstanding open problem. Towards this aim, we show that credibility can be verified when the critical value is computed in the parameter space rather than in the prediction space. For deep neural networks we show that the contraction rate of the corrected stochastic MH procedure coincides with the minimax rate by \citet{schmidthieber2020} (up to a logarithmic factor) for H\"older-regular hierarchical regression functions. While the latter paper has analyzed sparse deep neural networks with ReLU activation function, similar results for fully connected networks are given by \citet{KohlerLanger2021} and we exploit their main approximation theorem. A mixing approach in the prior distribution, compatible with, \emph{e.g.}, \citet{alquier2013} and \citet{Guedj2013}, leads to a fully adaptive method.  

\medskip{}

A simulation study demonstrates the merit of the correction term for sampling from a $10401$ dimensional parameter space for a low-dimensional regression task.
The approximate samples from the surrogate posterior of csMALA, as well as their mean, show a significant improvement in terms of the empirical prediction risk and size of credible balls over those taken from the surrogate posterior of the naive sMALA. 
The correction term cancels the bias on the size of accepted batches introduced by the stochastic setting.
The Python code of the numerical example is available on GitHub.\footnote{ \url{https://github.com/sbieringer/csMALA.git}}

\paragraph{Related literature.}

In view of possibly better scaling properties, variational Bayes methods have been studied intensively in recent years. Instead of sampling from  the posterior distribution itself, variational Bayes methods approximate the posterior within a parametric distribution class which can be easily sampled from, see \citet{BleiEtal2017} for a review. The theoretical understanding of variational Bayes methods is a current research topic, see \cite{ZhangZhou2020,ZhangGao2020,raySzabo2022} and references therein.

\medskip{}

Our oracle inequalities rely on PAC-Bayes theory which provides \emph{probably approximately correct} error bounds and goes back to \citet{ShaweTaylor1997}
and \citet{McAllester1999a,McAllester1999b}. We refer to the review papers by \citet{Guedj2019} and \citet{Alquier2021}, the monograph by \cite{hellstroem2025} and the pioneering works by \cite{catoni2004,catoni2007}. PAC-Bayes bounds in a regression setting have been studied, see, \emph{e.g.},  \citet{Audibert2004,Audibert2009,Audibert2011} and the references therein. Note that applying PAC-Bayes techniques is not straightforward, as the underlying loss of the surrogate posterior differs from the standard $L^2$-loss with respect to which our oracle inequality is formulated.

While we focus on oracle inequalities, PAC-Bayes bounds could be used to derive numerical risk certificates and to derive cost functions for training, as demonstrated, \emph{e.g.}, by \cite{Dziugaite2017}, \cite{Zhou2018}, \cite{perezEtAl2021}, \cite{biggs2021,biggs2022,biggs2023} in the contexts of deep and shallow learning.

Our analysis of the Bayesian procedure from a frequentist point of view embeds into the nonparametric Bayesian inference, see \citet{GhosalvanderVaart2017}.
Coverage of credible sets has been studied, for instance, by \citet{SzaboEtAl2015}
and \citet{RousseauSzabo2020} and based on the Bernstein-von Mises
theorem in \citet{CastilloNickl2014} among others. We would like to point out that our proof strategy to verify credibility with PAC-Bayes techniques seems novel and provides an alternative and possibly simpler way compared to the previously mentioned results.

While contraction
rates for Bayes neural networks have been studied by \citet{polsonRockova2018}, \citet{cherief2020} and \cite{castillo2024posterior}, the theoretical properties of credible sets
are not well understood so far. \citet{franssenSzabo2022} have studied an empirical Bayesian approach where only the last layer of the network is Bayesian while the remainder of the network remains fixed.

\medskip{}

For an introduction to neural networks see, \emph{e.g.},   \citet{goodfellowEtAl2016} and \citet{schmidhuber2015}. While early theoretical foundations for neural nets are summarized by \citet{AnthonyBartlett1999}, the excellent approximation properties of deep neural nets, especially with the ReLU activation function, have been discovered in recent years, see, \emph{e.g.},  \citet{yarotsky2017} and the review paper \citet{devoreEtAl2021}. In addition to these approximation properties, an  explanation of the empirical capabilities of neural networks has been given by \citet{schmidthieber2020} as well as \citet{BauerKohler2019}: While classical regression methods suffer from the curse of dimensionality, deep neural network estimators can profit from a hierarchical structure of the regression function and a possibly much smaller intrinsic dimension. 

\medskip{}

Tailoring Markov chains to the needs of current neural network applications is a field of ongoing investigation. 
Different efforts to improve efficiency by improve mixing, that is transitioning between modes of the posterior landscape, exist.
\citet{Zhang2020cyclicSGMCMC} employ a scheduled step-size to help the algorithm move between different modes of the posterior, while contour stochastic gradient MCMC \cite{Deng2020contourSGMCMC, Deng2022ISCGLD} uses a piece-wise continuous function to flatten the posterior landscape which is itself determined through MCMC sampling or from parallel chains.
Parallel chains of different temperature are employed by \citet{Deng2020replicaexchangeSGMCMC} at the cost of memory space during computation.
Only limited research on scaling MCMC for large data has been done. Most recently, \citet{cobb2020scaling} introduced a splitting scheme for Hamiltonian Monte Carlo maintaining the full Hamiltonian.

\paragraph{Organization.}
The paper is organized as follows: In \ref{sec:estimation}, we derive the stochastic MH procedure, introduce the stochastic MH correction and study the Kullback-Leibler divergences of the corresponding surrogate posteriors from the Gibbs posterior. In \ref{sec:Oracle-inequality}, we state the oracle inequality and the resulting contraction rates in a general regression setting and we investigate credible sets. In \ref{sec:NN} we apply the methodology to the setting of Bayesian neural networks.
 The numerical performance of the method is studied in \ref{sec:Numerics} and our key takeaways are summarized in  \ref{sec:conclusion}.
 All proofs have been postponed to \ref{sec:Proofs} which starts with an overview of how the main results and the auxiliary results relate to each other in their proofs.

\section{Surrogate Gibbs-posteriors for regression\label{sec:estimation}}

The aim is to estimate a regression function $f\colon\R^{\inputdim}\to\R$, $\inputdim\in\N$, under a quadratic loss and under classical assumptions
based on a training sample $\mathcal{D}_{n}\coloneqq(\mathbf{X}_{i},Y_{i})_{i=1,\dots,n}\subset\R^{\inputdim}\times\R$
given by $n\in\N$ i.i.d.\ copies of generic random variables $(\mathbf{X},Y)\in\R^{\inputdim}\times\R$
on some probability space $(\Omega,\mathcal{A},\P)$ with $Y=f(\mathbf{X})+\eps$ and
observation error $\eps$ satisfying $\E[\eps\mid\mathbf{X}]=0$
almost surely (a.s.). Equivalently, $f(\mathbf{X})=\E[Y\mid\mathbf{X}]$
a.s. For any estimator $\hat f$, the prediction risk and its empirical
counterpart are given by
\begin{equation}
R(\hat f)\coloneqq\E_{(\mathbf{X},Y)}\big[\big(Y-\hat f(\mathbf{X})\big)^{2}\big]\qquad\text{and}\qquad R_{n}(\hat f)=\frac{1}{n}\sum_{i=1}^{n}\big(Y_{i}-\hat f(\mathbf{X}_{i})\big)^{2},\label{eq:R}
\end{equation}
respectively, where $\E$ denotes the expectation under $\P$ and $\E_Z$ is the (conditional) expectation only with respect to a random variable $Z$. The
accuracy of the estimation procedure will be quantified in terms of
the excess risk 
\begin{equation}
\mathcal{E}(\hat f)\coloneqq R(\hat f)-R(f)=\E_{\mathbf{X}}\big[\big(\hat f(\mathbf{X})-f(\mathbf{X})\big)^{2}\big]=\Vert \hat f-f\Vert_{L^{2}(\P^{\mathbf{X}})}^{2},\label{eq:excess}
\end{equation}
where $\P^{\mathbf{X}}$ denotes the distribution of $\mathbf{X}$. A common alternative from the aforementioned literature on risk certificates would be to consider the generalization gap which allows for the use of PAC-Bayes generalization bounds as training objectives, see, \emph{e.g.}, \cite{Guedj2019}, \cite{Alquier2021}, \cite{hellstroem2025}.

We consider a parametric class of potential estimators $\mathcal{F}=\{f_{\theta}:\theta\in\Theta\}$, where the finite dimensional parameter space is fixed as $\Theta=[-B,B]^{\pardim}$ for simplicity
with some $B\ge1$ and a potentially large parameter dimension $\pardim\in\N$. For $f_{\theta}\in\mathcal{F}$
we abbreviate $R(\theta)=R(f_{\theta})$ and 
\[
R_{n}(\theta)=R_{n}(f_{\theta})=\frac{1}{n}\sum_{i=1}^{n}\ell_{i}(\theta)\qquad\text{with}\qquad\ell_{i}(\theta)=\big(Y_{i}-f_{\theta}(\mathbf{X}_{i})\big)^{2}.
\]
Throughout, $\vert x\vert_{\qnorm}$ denotes the $\ell^{\qnorm}$-norm of a vector $x\in\R^{\inputdim},\,\qnorm\in[1,\infty]$. For brevity, $\vert\cdot\vert\coloneqq \vert \cdot\vert_{2}$ is the Euclidean norm. We write $a\lor b\coloneqq \max\{a,b\}$ and $a\land b\coloneqq\min\{a,b\}$ for $a,b\in\R$. The identity
matrix in $\R^{d\times d}$ is denoted by $\id{d}$ and $\mathcal{O}_\P$ is the stochastic big $O$ Landau notation.

\subsection{Prior and posterior distribution}

As prior on the parameter set of the class $\mathcal{F}$ we choose a uniform distribution $\Pi=\mathcal{U}([-B,B]^{ \pardim})$. The corresponding
\emph{Gibbs posterior} $\Pi_{\lambda}(\cdot\mid\mathcal{D}_{n})$
is defined as the solution to the minimization problem
\[
\inf_{\nu}\Big(\int R_{n}(\theta)\,\nu(\d\theta)+\frac{1}{\lambda}\KL(\nu\mid \Pi)\Big)
\]
where the infimum is taken over all probability distributions $\nu$
on $\R^{ \pardim}$. Hence, $\Pi_{\lambda}(\cdot\mid\mathcal{D}_{n})$ will
concentrate at parameters $\theta$ with a small empirical risk $R_{n}(\theta)$,
but it takes into account a regularization term determined by the
Kullback-Leibler divergence (denoted by $\KL$, see \ref{eq:KL} for a definition) to
the prior distribution $\Pi$ and weighted via the \emph{inverse temperature
parameter} $\lambda>0$. Owing to the change of measure inequality, or, the Legendre transform of the Kullback-Leibler divergence, this optimization problem has a unique solution given by 
\begin{equation}
\Pi_{\lambda}(\mathrm{d}\theta\mid\mathcal{D}_{n})\propto\exp\big(-\lambda R_{n}(\theta)\big)\Pi(\mathrm{d}\theta),\label{eq:posterior}
\end{equation}
see \citet{csiszar1975}, \citet{donsker1976}, \cite{catoni2004,catoni2007} or \ref{lem:classicallemma} below. While
\ref{eq:posterior} coincides with the classical Bayesian posterior
distribution if $Y_{i}=f_{\theta}(\mathbf{X}_{i})+\eps_{i}$ with
i.i.d.\ $\eps_{i}\sim\mathcal{N}(0,\frac{n}{2\lambda})$, the so-called tempered likelihood, see, \emph{e.g.}, \citet{dalalyan2007}, \citet{Alquier2011}, 
\citet{Bissiri2016}, \citet{Guedj2018}, \citet{Guedj2019}, 
$\exp(-\lambda R_{n}(\theta))$ serves as a proxy for the unknown
distribution of the observations given $\theta$. As we will see,
the method is indeed applicable under quite general assumptions on
the regression model. 

Based on the Gibbs posterior distribution the regression function can be estimated via a random draw from the posterior
\begin{equation}
\widehat{f}_{\lambda}\coloneqq f_{\widehat{\theta}_{\lambda}}\qquad\text{for}\qquad\widehat{\theta}_{\lambda}\mid\mathcal{D}_{n}\sim\Pi_{\lambda}(\cdot\mid\mathcal{D}_{n}),\label{eq:Estimator}
\end{equation}
or via the posterior mean
\begin{equation}
\bar{f}_{\lambda}\coloneqq\E\big[f_{\hat{\theta}_{\lambda}}\,\big\vert\,\mathcal{D}_{n}\big]=\int f_{\theta}\,\Pi_{\lambda}(\d\theta\mid\mathcal{D}_{n}).\label{eq:postMean}
\end{equation}
Another popular approach is to use the maximum a posteriori (MAP)
estimator, but we will focus on the previous two estimators. 

\medskip{}

To apply the estimators $\hat{f}_\lambda$ and $\bar{f}_\lambda$ in practice, we need to sample from the Gibbs posterior. The MCMC approach is to construct
a Markov chain $(\theta^{(k)})_{k\in\N_{0}}$ with stationary distribution
$\Pi_{\lambda}(\cdot\mid\mathcal{D}_{n})$, see \cite{RobertCasella2004}. In particular, the \emph{Langevin} MCMC sampler is given by
\begin{equation}
\theta^{(k+1)}=\theta^{(k)}-\gamma\nabla_{\theta}R_{n}(\theta^{(k)})+s W_{k},\label{eq:Langevin}
\end{equation}
where $\nabla_{\theta}R_{n}(\theta)$ denotes the  gradient of $R_{n}(\theta)$ with respect to $\theta$,  $\gamma>0$ is the learning rate and $sW_{k}\sim\mathcal{N}(0,s^2\id{ \pardim})$ is i.i.d.\ white noise with noise level $s>0$. This
approach can also be interpreted as a noisy version of the gradient
descent method commonly used to train neural networks. In practice
this approach requires careful tuning of the procedural parameters
and Langevin-MCMC suffers from relatively slow polynomial convergence
rates of the distribution of $\theta^{(k)}$ to the target distribution
$\Pi_{\lambda}(\cdot\mid\mathcal{D}_{n})$, see \cite{nickl2022,cheng18a}.
Only in special cases, the convergence
rates are faster, see, \emph{e.g.},  \citet{Freund2022} for an overview and \citet{Dalalyan2020} for the case of log-concave densities. This convergence rate
can be considerably improved by adding an MH step
resulting in the \emph{Metropolis-adjusted Langevin algorithm} (MALA), see \cite{roberts1996b}.

Applying the generic MH algorithm to $\Pi_{\lambda}(\cdot\mid \mathcal{D}_{n})$
and taking into account that the prior $\Pi$ is uniform, we obtain
the following iterative method: Starting with some initial choice
$\theta^{(0)}\in\R^{ \pardim}$, we successively generate $\theta^{(k+1)}$
given $\theta^{(k)}$, $k\in\N_{0}$, by 
\[
\theta^{(k+1)}=\begin{cases}
\theta' & \text{with probability }\alpha(\theta'\mid \theta^{(k)})\\
\theta^{(k)} & \text{with probability }1-\alpha(\theta'\mid \theta^{(k)})
\end{cases},
\]
where $\theta'$ is a random variable drawn from some conditional
proposal density $q(\cdot\mid\theta^{(k)})$ and the \emph{acceptance
probability} is chosen as
\begin{equation}
\alpha(\theta'\mid \theta)=\exp\big(-\lambda R_{n}(\theta')+\lambda R_{n}(\theta)\big)\1_{[-B,B]^{ \pardim}}(\theta')\frac{q(\theta\mid\theta')}{q(\theta'\mid\theta)}\land 1.\label{eq:alpha}
\end{equation}
In view of \ref{eq:Langevin} the probability density $q$ of the
proposal distribution is given by 
\begin{equation}
q(\theta'\mid\theta)=\frac{1}{(2\pi s^{2})^{ \pardim/2}}\exp\Big(-\frac{1}{2s^{2}}\big\vert\theta'-\theta+\gamma\nabla_{\theta}R_{n}(\theta)\big\vert^{2}\Big).\label{eq:proposal}
\end{equation}
The standard deviation $s$ should not be too large as otherwise
the acceptance probability might be too small. As a result the proposal
would rarely be accepted, the chain might not be sufficiently randomized
and the convergence to the invariant target distribution would be
too slow in practice. On the other hand, $s$ should not be smaller
than the shift $\gamma\nabla_{\theta}R_{n}(\theta)$ in the mean,
since otherwise $q(\theta\mid\theta')$ might be too small. The MH step ensures that $(\theta^{(k)})_{k\in\N_{0}}$ is a Markov chain
with invariant distribution $\Pi_{\lambda}(\cdot\mid\mathcal{D}_{n})$
(under rather mild conditions on $q$). The convergence to the invariant
distribution follows from \citet[Theorem 2.2]{Roberts1996} with geometric
rate. 

To calculate the estimators $\hat f_{\lambda}$ and $\bar{f}_{\lambda}$
from \ref{eq:Estimator} and \ref{eq:postMean}, respectively, one
chooses a \emph{burn-in} time $b\in\N$ to let the distribution of
the Markov chain stabilize at its invariant distribution and then
approximates
\[
\hat f_{\lambda}\approx f_{\theta^{(b)}}\qquad\text{and}\qquad\bar{f}_{\lambda}\approx \frac{1}{N}\sum_{k=1}^{N}f_{\theta^{(b+ck)}}.
\]
In practice, $b$ can be calibrated by plotting the training loss against the number of iterations and looking for the point at which the training loss stabilizes.

A sufficiently large \emph{gap length} $c\in\N$ ensures the necessary
variability and reduced dependence between $\theta^{(b+ck)}$
and $\theta^{(b+c(k+1))}$, whereas $N\in\N$ has to be large enough
for a good approximation of the expectation by the empirical mean. 

\subsection{Surrogate posterior of stochastic MALA}

The gradient has to be calculated only once in each MALA iteration.
Hence, using the full gradient $\nabla_{\theta}R_{n}(\theta)=\frac{1}{n}\sum_{i=1}^{n}\nabla_{\theta}\ell_{i}(\theta)$,
the additional computational price of MALA compared to training a
standard neural network by empirical risk minimization only comes
from a larger number of necessary iterations due to the rejection
with probability $1-\alpha(\theta'\mid \theta^{(k)})$. For large datasets
however the standard training of a neural network would rely on a
stochastic gradient method, where the gradient $\frac{1}{m}\sum_{i\in \mathcal{B}}\nabla_{\theta}\ell_{i}(\theta)$
is only calculated on (mini-)batches $\mathcal{B}\subset\{1,\dots,n\}$ of size
$m<n$. While we could replace $\nabla_{\theta}R_{n}(\theta)$ in
\ref{eq:proposal} by a stochastic approximation without any additional
obstacle, the MH step still requires the calculation of the loss $\ell_{i}(\theta')$
for all $1\le i\le n$ in \ref{eq:alpha}. 

To avoid a full evaluation of the empirical risk $R_{n}(\theta)$,
a natural approach is to replace the empirical risks in $\alpha(\theta'\mid \theta)$
by a batch-wise approximation, too. To study the consequences
of this approximation we follow a pseudo-marginal MH 
approach, see \cite{AndrieuRoberts2009,MaclaurinAdams2014,BardenetEtAl2017,wu2022}. 

We augment our target distribution by a set of auxiliary random variables
$Z_{1},\dots,Z_{n}\overset{\mathrm{i.i.d.}}{\sim}\mathrm{Ber}(\rho)$ with some $\rho\in(0,1]$
and aim for a reduction of the empirical risk $R_{n}(\theta)$ to
the stochastic approximation 
\[
R_{n}(\theta,Z)\coloneqq\frac{1}{n\rho}\sum_{i=1}^{n}Z_{i}\ell_{i}(\theta)
\]
in the algorithm. Hence, we define the joint target distribution by
\begin{align*}
\bar{\Pi}_{\lambda,\rho}(\theta,z\mid\mathcal{D}_{n}) & \propto\prod_{i=1}^{n}\rho^{z_{i}}(1-\rho)^{1-z_{i}}\exp\big(-\lambda R_{n}(\theta,z)\big)\Pi(\d\theta)\\
 & \propto\exp\Big(-\lambda R_{n}(\theta,z)+\log\Big(\frac{\rho}{1-\rho}\Big)\sum_{i=1}^{n}z_{i}\Big)\Pi(\d\theta),\qquad z\in\{0,1\}^{n}.
\end{align*}
The marginal distribution in $\theta$ is then given by
\begin{align}
\bar{\Pi}_{\lambda,\rho}(\theta\mid\mathcal{D}_{n})=\sum_{z\in\{0,1\}^{n}}\bar{\Pi}_{\lambda,\rho}(\theta,z\mid\mathcal{D}_{n}) & \propto\prod_{i=1}^{n}\Big(\rho \e^{-\frac{\lambda}{n\rho}\ell_{i}(\theta)}+1-\rho\Big)\Pi(\d\theta).\label{eq:PiBar}
\end{align}
As proposal for the MH algorithm we use
\begin{align}
\bar{q}(\theta',z'\mid\theta,z) & =q_{\mathrm{s}}(\theta'\mid\theta,z)\prod_{i=1}^{n}\rho^{z'_{i}}(1-\rho)^{1-z'_{i}}\qquad\text{with}\label{eq:qBar}\\
q_{\mathrm{s}}(\theta'\mid\theta,z) & =\frac{1}{(2\pi s^{2})^{ \pardim/2}}\exp\Big(-\frac{1}{2s^{2}}\big\vert\theta'-\theta+\gamma\nabla_{\theta}R_{n}(\theta,z)\big\vert^{2}\Big).\nonumber 
\end{align}
Hence, the proposed $Z'=z'$ is indeed a vector of independent $\mathrm{Ber}(\rho)$-random
variables and $q_{\mathrm{s}}(\theta'\mid\theta,z)$ is the stochastic
analogue to $q$ from \ref{eq:proposal} with a stochastic gradient.
The resulting acceptance probabilities are given by 
\begin{align*}
\alpha(\theta',z'\mid\theta,z) & =\frac{\bar{q}(\theta,z\mid\theta',z')\bar{\Pi}_{\lambda,\rho}(\theta',z'\mid\mathcal{D}_{n})}{\bar{q}(\theta',z'\mid\theta,z)\bar{\Pi}_{\lambda,\rho}(\theta,z\mid\mathcal{D}_{n})}\wedge1\\
 & =\frac{q_{\mathrm{s}}(\theta\mid\theta',z')}{q_{\mathrm{s}}(\theta'\mid\theta,z)}\1_{[-B,B]^{ \pardim}}(\theta')\e^{-\lambda R_{n}(\theta',z')+\lambda R_{n}(\theta,z)}\wedge1.
\end{align*}
We observe that $\alpha(\theta',z'\mid\theta,z)$ corresponds to a stochastic
MH step where we have to evaluate the loss $\ell_{i}(\theta')$
for the new proposal $\theta'$ only if $z_{i}'=Z'_{i}\sim\mathrm{Ber}(\rho)$
is one, \emph{i.e.}, with probability $\rho$. Calculating $\alpha(\theta',z'\mid\theta,z)$
thus requires only few evaluations of $\ell_{i}(\theta)$ for small
values of $\rho$. The expected number of data points on which the
gradient and the loss have to be evaluated is $n\rho$ and corresponds
to a batch size of $m=n\rho$. 

Generalizing \ref{eq:Estimator}, we define the stochastic MH estimator
\begin{equation}
\widehat{f}_{\lambda,\rho}\coloneqq f_{\widehat{\theta}_{\lambda,\rho}}\qquad\text{for}\qquad\widehat{\theta}_{\lambda,\rho}\mid\mathcal{D}_{n}\sim\bar{\Pi}_{\lambda,\rho}(\cdot\mid\mathcal{D}_{n}).\label{eq:EstimatorRhoStochMH}
\end{equation}
For $\rho=1$ we recover the standard MALA.

As discussed by \citet{BardenetEtAl2017}, the previous derivation
reveals that the stochastic MH step leads to a different
invariant distribution of the Markov chain, namely \ref{eq:PiBar}
instead of the Gibbs posterior from \ref{eq:posterior}. Writing
\begin{equation}
\bar{\Pi}_{\lambda,\rho}(\theta\mid\mathcal{D}_{n})\propto\exp\big(-\lambda\bar{R}_{n,\rho}(\theta)\big)\Pi(\d\theta)\qquad\text{with}\qquad\bar{R}_{n,\rho}(\theta)\coloneqq-\frac{1}{\lambda}\sum_{i=1}^{n}\log\big(\rho \e^{-\frac{\lambda}{n\rho}\ell_{i}(\theta)}+1-\rho\big),\label{eq:rBar}
\end{equation}
we observe that $\bar{\Pi}_{\lambda,\rho}(\cdot\mid\mathcal{D}_{n})$
is itself a Gibbs posterior distribution, the \emph{surrogate posterior}, corresponding to the modified
risk $\bar{R}_{n,\rho}(\theta)$. Note that $\bar{\Pi}_{\lambda,\rho}(\cdot\mid\mathcal{D}_{n})$
coincides with $\Pi_{\lambda}(\cdot\mid\mathcal{D}_{n})$ for $\rho=1$
and thus $\hat f_{\lambda}=\hat f_{\lambda,1}$ and $\bar{f}_{\lambda}=\bar{f}_{\lambda,1}$
in distribution. Whether $\bar{\Pi}_{\lambda,\rho}(\cdot\mid\mathcal{D}_{n})$ also 
behaves as our original target distribution $\Pi_{\lambda}(\cdot\mid\mathcal{D}_{n})$ for $\rho<1$ depends on the choice of $\lambda$ and $\rho$:
\begin{lem}
\label{lem:KLApprox}If $f$ and all $f_{\theta}$ are bounded by
some constant $\Cf>0$, then we have
\[
\frac{1}{n\rho}\KL\big(\bar{\Pi}_{\lambda,\rho}(\cdot\mid\mathcal{D}_{n})\,\big\vert\,\Pi_{\lambda}(\cdot\mid\mathcal{D}_{n})\big)\le\Big(\frac{\lambda}{n\rho}\Big)^{2}\Big(64\Cf^{4}+\frac{4}{n}\sum_{i=1}^{n}\eps_{i}^{4}\Big).
\]
For $\rho<1$ and the probability distribution $\varpi_{\lambda,\rho}(\theta\mid\mathcal{D}_{n}):\propto\exp\big(\rho\sum_{i=1}^{n}\e^{-\frac{\lambda}{n\rho}\ell_{i}(\theta)}\big)\Pi(\d\theta)$
we moreover have
\[
\frac{1}{n\rho}\KL\big(\bar{\Pi}_{\lambda,\rho}(\cdot\mid\mathcal{D}_{n})\,\big\vert\,\varpi_{\lambda,\rho}(\cdot\mid\mathcal{D}_{n})\big)\le\frac{\rho}{1-\rho}.
\]
\end{lem}

On the one hand, if $\frac{\lambda}{n\rho}$ is sufficiently small,
then the surrogate posterior $\bar{\Pi}_{\lambda,\rho}(\cdot\mid\mathcal{D}_{n})$ is indeed
a good approximation for the Gibbs posterior $\Pi_{\lambda}(\cdot\mid\mathcal{D}_{n})$.
On the other hand, for $\rho\to0$ the distribution $\bar{\Pi}_{\lambda,\rho}(\cdot\mid\mathcal{D}_{n})$
behaves as the distribution $\varpi_{\lambda,\rho}(\cdot\mid\mathcal{D}_{n})$
with density proportional to 
\[
\exp\Big(\rho\sum_{i=1}^{n}\e^{-\frac{\lambda}{n\rho}\ell_{i}(\theta)}\Big)\Pi(\d\theta).
\]
For large $\frac{\lambda}{n\rho}$ the terms $\e^{-\frac{\lambda}{n\rho}\ell_{i}(\theta)}$
rapidly decay for all $\theta$ with $\ell_{i}(\theta)>0$, \emph{i.e.},  $\varpi_{\lambda,\rho}(\cdot\mid\mathcal{D}_{n})$
emphasizes interpolating parameter choices. For all $\theta$ where
$\frac{\lambda}{n\rho}\ell_{i}(\theta)$ is relatively large the density
converges to a constant. Therefore, in the extreme case $\text{\ensuremath{\rho}}\to0$
and $\frac{\lambda}{n\rho}\to\infty$ the distribution $\varpi_{\lambda,\rho}(\cdot\mid\mathcal{D}_{n})$
and thus $\bar{\Pi}_{\lambda,\rho}(\cdot\mid\mathcal{D}_{n})$ converge
to the uninformative prior with interpolating spikes at parameters
where $\ell_{i}(\theta)$ are zero.

We illustrate \ref{lem:KLApprox} in a simple setting where $Y_{i}=\mathcal{N}(0,0.5)$
and $f_{\theta}(x)=\theta$ for $\theta\in[-1,1]$. The densities
of the measures $\Pi(\cdot\mid\mathcal{D}_{n})$, $\bar{\Pi}(\cdot\mid\mathcal{D}_{n})$
and $\varpi(\cdot\mid\mathcal{D}_{n})$ are shown in \ref{fig:posteriors}
for different choices of $\lambda$ and $\rho$. \ref{fig:posteriors}
confirms the predicted approximation properties: $\bar\Pi_{\lambda,\rho}(\cdot\mid\mathcal D_n)$ behaves similarly to $\Pi_\lambda(\cdot\mid\mathcal D_n)$ if $\lambda$ is not too large (orange lines) or $\rho$ is not too small (left figure). Additionally, we observe
that $\bar{\Pi}_{\lambda,\rho}(\cdot\mid\mathcal{D}_{n})$ is still informative
if $\lambda$ is in the order $n\rho$ even if it is not close to
the Gibbs posterior at all. 

\begin{figure}
\centering{}\input{posteriors.tex}\caption{\label{fig:posteriors}\emph{Points}: $Y_{1},\dots,Y_{n}\sim\mathcal{N}(0,0.5)$
for $n=10$. \emph{Solid lines}: densities of $\bar{\Pi}_{\lambda,\rho}(\cdot\mid\mathcal{D}_{n})$
with $\lambda=10n$ (blue) and $\lambda=n/2$ (orange) and $\rho=0.9$
(left) and $\rho=0.1$ (right). \emph{Dashed lines}: corresponding
densities of $\Pi_{\lambda}(\cdot\mid\mathcal{D}_{n})$. Dotted lines:
corresponding densities of $\varpi_{\lambda,\rho}(\cdot\mid\mathcal{D}_{n})$.}
\end{figure}

The scaling of the Kullback-Leibler distance with $n\rho$ in \ref{lem:KLApprox}
is quite natural in this setting. In particular, applying an approximation
result from the variational Bayes literature by \citet[Theorem 5]{raySzabo2022}
we obtain for the two reference measures $\mathbb{Q}\in\{\Pi_{\lambda}(\cdot\mid\mathcal{D}_{n}),\varpi_{\lambda,\rho}(\cdot\mid\mathcal{D}_{n})\}$
and a high probability parameter set $\Theta_{n}$ with $\mathbb{Q}(\Theta_{n}^{\comp})\le C\e^{-n\rho}$
for some constant $C>0$ that
\begin{equation}\label{eq:RaySzabo}
\E\big[\bar{\Pi}_{\lambda,\rho}(\Theta_{n}\mid\mathcal{D}_{n})\big]\le\frac{2}{n\rho}\E\big[\KL\big(\bar{\Pi}_{\lambda,\rho}(\cdot\mid\mathcal{D}_{n})\,\big\vert\,\mathbb{Q}\big)\big]+C\e^{-n\rho/2}.
\end{equation}
Hence, for $\frac{\lambda}{n\rho}\to0$ we could analyze the surrogate posterior
via the Gibbs posterior itself at the cost of the approximation error
$\frac{1}{n\rho}\KL\big(\bar{\Pi}_{\lambda,\rho}(\cdot\mid\mathcal{D}_{n})\,\big\vert\,\Pi_{\lambda}(\cdot\mid\mathcal{D}_{n})\big)$.
Instead of this route, we will directly investigate $\bar{\Pi}_{\lambda,\rho}(\cdot\mid\mathcal{D}_{n})$
which especially allows for $\lambda$ in the order of $n\rho$.

\subsection{Surrogate posterior of corrected stochastic MALA}\label{ssec:csMALA}
The computational advantage of the stochastic MH algorithm due to the reduction of the information parameter from $n$ to $\rho n$ comes at the cost of a worse convergence rate of an estimator based on its surrogate posterior, as we will demonstrate with \ref{thm:oracleinequalityStochMH} in \ref{subsec:oracle}. This motivates the introduction of a corrected stochastic MALA (csMALA) and the analysis of its surrogate posterior.

To remedy the loss of information while retaining scalability, we define another joint
target distribution as
\begin{align*}
\tilde{\Pi}_{\lambda,\rho}(\theta,z\mid\mathcal{D}_{n}) & \propto\prod_{i=1}^{n}\big(\e^{-\frac{\lambda}{n}\ell_{i}(\theta)z_{i}}(1-\rho)^{1-z_{i}}\big)\Pi(\d\theta)\\
 & \propto\exp\Big(-\frac{\lambda}{n}\sum_{i=1}^{n}z_{i}\ell_{i}(\theta)-\log(1-\rho)\sum_{i=1}^{n}z_{i}\Big)\Pi(\d\theta),\qquad z\in\{0,1\}^{n},
\end{align*}
with marginal distribution in $\theta$ given by 
\begin{align*}
\tilde{\Pi}_{\lambda,\rho}(\theta\mid\mathcal{D}_{n})=\sum_{z\in\{0,1\}^{n}}\tilde{\Pi}_{\lambda,\rho}(\theta,z\mid\mathcal{D}_{n}) & \propto\prod_{i=1}^{n}\Big(\rho\frac{\e^{-\frac{\lambda}{n}\ell_{i}(\theta)}}{\rho}+1-\rho\Big)\Pi(\d\theta)=\exp\big(-\lambda\tilde R_{n,\rho}(\theta)\big)\Pi(\d\theta)
\end{align*}
with
\begin{equation}
\tilde R_{n,\rho}(\theta)\coloneqq-\frac{1}{\lambda}\sum_{i=1}^{n}\log\big(\e^{-\frac{\lambda}{n}\ell_{i}(\theta)}+1-\rho\big).\label{eq:rTilde}
\end{equation}
\begin{rem}
As with sMALA, we recover MALA for $\rho=1$. Therefore, all of our upcoming results, especially those regarding uncertainty quantification, also hold for MALA.
\end{rem}

Compared to $\bar{R}_{n,\rho}$ from \ref{eq:rBar} there is no $\rho$
in the first term in the logarithm. In line with \ref{eq:Estimator} and \ref{eq:postMean}, we obtain the
estimators
\begin{equation}
\tilde f_{\lambda,\rho}\coloneqq f_{\tilde{\theta}_{\lambda,\rho}}\qquad\text{for}\qquad\tilde{\theta}_{\lambda,\rho}\mid\mathcal{D}_{n}\sim\tilde{\Pi}_{\lambda,\rho}(\cdot\mid\mathcal{D}_{n})\label{eq:EstimatorRho}
\end{equation}
and
\begin{equation}
\bar{f}_{\lambda,\rho}\coloneqq\E\big[f_{\tilde{\theta}_{\lambda,\rho}}\,\big\vert\,\mathcal{D}_{n}\big]=\int f_{\theta}\,\tilde{\Pi}_{\lambda,\rho}(\d\theta\mid\mathcal{D}_{n}).\label{eq:postMeanRho}
\end{equation}

To sample from $\tilde\Pi_{\lambda,\rho}(\cdot\mid\mathcal D_n)$ the MH algorithm with proposal density $q(\theta',z'\mid\theta,z)=q_{\mathrm{s}}(\theta'\mid\theta,z)\prod_{i=1}^{n}\rho^{z'_{i}}(1-\rho)^{1-z'_{i}}$
as in \ref{eq:qBar} leads to the acceptance probabilities
\begin{align*}
\alpha(\theta',z'\mid\theta,z) & =\frac{q_{\mathrm{s}}(\theta\mid\theta',z')}{q_{\mathrm{s}}(\theta'\mid\theta,z)}\1_{[-B,B]^{ \pardim}}(\theta')\exp\Big(-\sum_{i=1}^{n}z_{i}'\big(\tfrac{\lambda}{n}\ell_{i}(\theta')+\log\rho\big)+\sum_{i=1}^{n}z_{i}\big(\tfrac{\lambda}{n}\ell_{i}(\theta)+\log\rho\big)\Big)\wedge1.
\end{align*}
To take the randomized batches into account, we thus introduce a small \emph{correction term} $\frac{\log\rho}{\lambda}\vert Z\vert_1=\mathcal{O}_{\P}(\frac{n}{\lambda}\rho\log\rho)$
in the empirical risks. 
The resulting surrogate posterior
$\tilde{\Pi}_{\lambda,\rho}(\theta\mid\mathcal{D}_{n})$ achieves a considerably improved
approximation of the Gibbs distribution $\Pi_\lambda(\cdot\mid\mathcal D_n)$: 
\begin{lem}\label{lem:KLtilde}
If $f$ and all $f_{\theta}$ are bounded by some constant $\Cf>0$, then we have
\[
\frac{1}{n}\KL\big(\tilde{\Pi}_{\lambda,\rho}(\cdot\mid\mathcal{D}_{n})\mid\Pi_{\lambda/(2-\rho)}(\cdot\mid\mathcal{D}_{n})\big)\le\Big(\frac{\lambda}{n}\Big)^{2}\Big(32\Cf^{4}+\frac{2}{n}\sum_{i=1}^{n}\eps_{i}^{4}\Big).
\]
\end{lem}
Compared to \ref{lem:KLApprox}, the approximation error of $\tilde\Pi_{\lambda,\rho}(\cdot\mid\mathcal D_n)$ in terms of the Kullback-Leibler distance is now determined by the full sample size $n$ instead of the possibly much smaller batch size $\rho n$ as for the surrogate posterior of sMALA. The only price to pay is a reduction of the inverse temperature parameter $\lambda$ by the factor $(2-\rho)^{-1}\in[\frac12,1]$. As already mentioned in \ref{eq:RaySzabo}, we can conclude contraction and coverage results for $\tilde\Pi_{\lambda,\rho}(\cdot\mid\mathcal D_n)$ by combining \citet[Theorem 5]{raySzabo2022} with \ref{lem:KLtilde} if $\lambda/n\to0$. A direct analysis of $\tilde\Pi_{\lambda,\rho}(\cdot\mid\mathcal D_n)$ will even allow for $\lambda$ of the order $n$ in our main results and thus lead to results as good as we can hope for the Gibbs measure itself.

\medskip{}

The corrected stochastic MALA (csMALA) is summarized in \ref{alg:MCMCV2}.
The implementation omits
the restriction of the proposed network weights to $[-B,B]^{ \pardim}$
which is practically negligible for sufficiently large constant $B$ and the correction term $\frac{\log\rho}{\lambda}\vert Z\vert=\mathcal{O}_{\P}(\frac{n}{\lambda}\rho\log\rho)$ in the empirical risk  is weighted by some tuning parameter $\zeta\ge0$. For $\zeta=0$ we recover the uncorrected method. In theory we always set $\zeta=1$, but in practice the flexibility gained from choosing $\zeta$ was beneficial. The convergence of the algorithm for general $\zeta>0$ and its dependence on $\rho$ is an open research question which could be studied using state-of-the-art Markov chain Monte Carlo theory, see, \emph{e.g.}, \cite{chewi2024}.

\begin{algorithm}
\begin{lyxcode}
\textbf{Input:~}inverse~temperature~~\mbox{$\lambda>0$},~learning~rate~\mbox{$\gamma>0$},~standard~deviation~\mbox{$s>0$},\\
\hspace*{0.2cm}correction~parameter~\mbox{$\zeta\ge0$},~average~batch~proportion~\mbox{$\rho\in(0,1]$}~of~data~used,\\
\hspace{0.2cm}burn-in~\mbox{$b\in\N$}, gap~length~\mbox{$c\in\N$},~number~of~draws~$N\in\N$.
\begin{enumerate}
\item Initialize $\theta^{(0)}\in\R^{ \pardim}$ and
$Z^{(0)}\sim\mathrm{Ber}(\rho)^{\otimes n}$.
\item Calculate $R_{n}^{(0)}=\frac{1}{n}\sum_{i=1}^{n}Z_{i}^{(0)}\ell_{i}(\theta^{(0)})+\zeta\frac{\log\rho}{\lambda}\vert Z^{(0)}\vert_1$
and $\nabla R_{n}^{(0)}=\nabla_{\theta}R_{n}(\theta^{(0)},Z^{(0)})$.
\item For $k=0,\dots,b+cN$ do:
\begin{enumerate}
\item Draw $Z'\sim\mathrm{Ber}(\rho)^{\otimes n}.$
\item Draw $\theta'\sim\mathcal{N}(\theta^{(k)}-\gamma\nabla R_{n}^{(k)},s^{2}\id{\pardim})$
and calculate $R_{n}'=\frac{1}{n}\sum_{i=1}^{n}Z_{i}'\ell_{i}(\theta')+\zeta\frac{\log\rho}{\lambda}\vert Z'\vert_1$
and $\nabla R_{n}'=\nabla_{\theta}R_{n}(\theta',Z')$.
\item Calculate acceptance probability~\\
\parbox[t]{0.8\textwidth}{
\begin{align*}
\alpha^{(k+1)} & =\exp\Big(\lambda R_{n}^{(k)}+\frac{1}{2s^{2}}\big\vert\theta'-\theta^{(k)}+\gamma\nabla R_{n}^{(k)}\big\vert^{2}-\lambda R_{n}'-\frac{1}{2s^{2}}\big\vert\theta^{(k)}-\theta'+\gamma\nabla R_{n}'\big\vert^{2}\Big).
\end{align*}
}
\item Draw $u\sim\mathcal{U}([0,1])$. If $u\le\alpha^{(k+1)}$, ~\\
then set $\theta^{(k+1)}=\theta',R_{n}^{(k+1)}=R_{n}',\nabla R_{n}^{(k+1)}=\nabla R_{n}^{'},$~\\
else set $\theta^{(k+1)}=\theta^{(k)},R_{n}^{(k+1)}=R_{n}^{(k)},\nabla R_{n}^{(k+1)}=\nabla R_{n}^{(k)}.$
\end{enumerate}
\end{enumerate}
\textbf{Output:~$\tilde f_{\lambda,\rho}=f_{\theta^{(b)}}$,}~$\bar{f}_{\lambda,\rho}=\frac{1}{N}\sum_{k=1}^{N}f_{\theta^{(b+ck)}}$

\caption{\label{alg:MCMCV2}csMALA - corrected stochastic Metropolis adjusted Langevin Algorithm}
\end{lyxcode}
\end{algorithm}

\section{Oracle inequality and uncertainty quantification\label{sec:Oracle-inequality}}

In this section we state the statistical guarantees for the estimators
defined in terms of the surrogate posterior distributions. It is worth noting that our analysis is independent of the choice of the proposal distribution. We will derive
oracle inequalities for the estimators $\widehat{f}_{\lambda,\rho}$ (\ref{thm:oracleinequalityStochMH}) 
and $\tilde f_{\lambda,\rho}$ (\ref{thm:oracleinequality}) and as a consequence an analogous oracle inequality for $\bar{f}_{\lambda,\rho}$ (\ref{cor:mean}), which
verify that these estimators are not much worse than the optimal choice
for $\theta$. Afterwards, we discuss the properties of credible balls with respect to the posterior distributions.

\subsection{Oracle inequality}\label{subsec:oracle}

Our first main result compares the performance of the estimator $\tilde f_{\lambda,\rho}$
from \ref{eq:EstimatorRho} to the best possible estimator $f_{\theta^{*}}\in \mathcal{F}=\{f_{\theta}:\theta\in[-B,B]^{\pardim}\}$
for the \emph{oracle choice} 
\begin{equation}
\theta^{\ast}\in\argmin{\theta\in[-B,B]^{ \pardim}}R(\theta)=\argmin{\theta\in[-B,B]^{ \pardim}}\mathcal{E}(\theta).\label{eq:oracle}
\end{equation}
The oracle is not accessible to the practitioner because $R(\theta)$
depends on the unknown distribution of $(\mathbf{X},Y)$. Instead, the oracle serves as a benchmark against which the performance of the estimators can be assessed. A solution
to the minimization problem in \ref{eq:oracle} always exists since
$[-B,B]^{ \pardim}$ is compact and we will assume $\theta\mapsto R(\theta)$ to be continuous. Throughout we assume w.l.o.g.\@ that $\theta^*\in(-B,B)^\pardim$.
If there is more than one solution, we choose one of them. We need
some mild assumption on the regression model.
\begin{assumption}
\label{assu:bounded}~
\begin{enumerate}
\item \textbf{Bounded regression function: }For some $\Cf\ge1$ we have $\Vert f\Vert_{\infty}\le \Cf$.
\item \textbf{Conditional sub-Gaussianity of observation noise: }There are
constants $\sigma,\Ceps>0$ such that 
\[
\E[\vert\eps\vert^{k}\mid\mathbf{X}]\le\frac{k!}{2}\sigma^{2}\Ceps^{k-2}\as,\qquad\text{for all }k\ge2.
\]
\item \textbf{Conditional symmetry of observation noise: }$\eps$ is conditionally
on $\mathbf{X}$ symmetric.
\end{enumerate}
\end{assumption}

Note that neither the loss function nor the data are assumed to be
bounded. For the function class we require the following:
\setcounter{assumptioncounter}{2}
\begin{assumption}\label{assu:F} 
Let $\mathcal F=\{f_{\theta}:\theta\in\Theta\}$ with $\Theta=[-B,B]^\pardim$ satisfy:
\begin{enumerate}
  \item\label{assu:bounded_functions} \textbf{Bounded functions:} There is some $\Cf\ge1$ such that $\|f_\theta\|_\infty\le \Cf$ for all $f_\theta\in\mathcal F$.
  \item\label{assu:lipschitz} \textbf{Lipschitz dependence on the parameters:} There is a norm $|\cdot|_\Theta$ on $\Theta$ and some constant $\Delta>0$ such that 
  \begin{equation}
  \| f_{\theta}-f_{\tilde{\theta}}\|_{L^2(\P^{\mathbf X})}\le \Delta\vert \theta-\tilde{\theta}\vert_\Theta\qquad\text{ for all }\qquad\theta,\tilde\theta\in\Theta.
  \end{equation}
\end{enumerate}
\end{assumption}
\noindent We obtain the following non-asymptotic oracle inequality for our estimator $\tilde f_{\lambda,\rho}$ from \ref{eq:EstimatorRho}:
\begin{thm}[PAC-Bayes oracle inequality for the surrogate posterior of csMALA]
\label{thm:oracleinequality} Under \ref{assu:bounded} and \ref{assu:F} there are
constants $\K{0},\K{1}>0$ depending only on $\Cf,\Ceps,\sigma$ such that for $\lambda=n/\K{0}$ and sufficiently
large $n$ we have for all $\delta\in(0,1)$ with probability at least $1-\delta$ that
\begin{equation}
\mathcal{E}(\tilde f_{\lambda,\rho})\le12\mathcal{E}(f_{\theta^{*}})+\frac{\K{1}}{n}\big( \pardim \log(B \Delta n)-\log\operatorname{vol}(\mathcal B_1)+\log(2/\delta)\big),\label{eq:oracleIneq}
\end{equation}
where $\operatorname{vol}(\mathcal B_1)$ denotes the volume of the $|\cdot|_\Theta$-unit ball in $\R^\pardim$.
\end{thm}

\begin{rem}
For $\rho=1$ we do not need the conditional symmetry condition in
\ref{assu:bounded}. The remaining parts of \ref{assu:bounded} are in line with the literature, see, \emph{e.g.}, \citet{alquier2013,Guedj2013}. An explicit admissible choice for $\lambda$
is 
$\lambda=n/\big(2^5\Cf(\Ceps\lor (2\Cf))+2^7(\Cf^2+\sigma^2)+2^3(\sigma \Cf+\sigma^2)\big)$.
The dependence of $\K{1}$ on $\Cf,\Ceps,\sigma$ is at most quadratic. While the oracle inequality is not sharp (the leading constant $12$ in front of $\mathcal{E}(f_{\theta^{*}})$ is larger than $1$), it leads to optimal contraction rates (up to logarithms) in terms of the sample size as we will demonstrate in the upcoming \ref{prop:rate}.
\end{rem}

The term $\log\operatorname{vol}(\mathcal B_1)$ depends on the geometry induced by the $|\cdot|_\Theta$-norm. If $|\cdot|_\Theta$=$|\cdot|_\qnorm$ for some $\qnorm\in[1,\infty]$ and with the gamma function $\Gamma$, we have
\begin{equation}
  -\log\operatorname{vol}(\mathcal B_1)=\log(\Gamma(1+\pardim/\qnorm))-\pardim\log\big(2\Gamma(1+1/\qnorm)\big),
\end{equation}
which is $-\pardim\log 2$ for $\qnorm=\infty$ and of order $\pardim\log \pardim$ for $\qnorm\in[1,\infty)$. Hence, as long as $\pardim$ does not grow faster than polynomially in $n$, this term does not affect the order of the stochastic error term.

The right-hand side of \ref{eq:oracleIneq} can be interpreted similarly
to the classical bias-variance decomposition in nonparametric statistics.
The first term $\mathcal{E}(f_{\theta^{\ast}})=\E[(f_{\theta^{\ast}}(\mathbf{X})-f(\mathbf{X}))^{2}]$
quantifies the approximation error while the second term is an upper bound
for the stochastic error. A key implication is that the excess risk is small if the methodology is applied to a setting where an oracle choice of $\theta$ balances the order of corresponding the oracle risk and the model complexity. For this, we will later consider neural networks, see \ref{sec:NN}. 
\ref{thm:oracleinequality}
is in line with classical PAC-Bayes oracle inequalities, see \citet{Bissiri2016},
\citet{Guedj2013}, \citet{zhang2006}. In particular, \citet{cherief2020} has obtained
a similar oracle inequality for a variational approximation of the
Gibbs posterior distribution. 
A main step in the proof of \ref{thm:oracleinequality}
is to verify the compatibility between the risk $\tilde R_{n,\rho}$
from \ref{eq:rTilde} and the empirical risk $R_{n}$ as we will establish
in \ref{prop:compatibility}.

We obtain a similar result for $\widehat{f}_{\lambda,\rho}$ from
\ref{eq:EstimatorRhoStochMH}. The key difference is that the stochastic error
term is of order $\mathcal{O}(\frac{\pardim }{n\rho})$ instead of $\mathcal{O}(\frac{\pardim}{n})$
as in \ref{thm:oracleinequality} (up to the logarithm). Consequently, an estimator based on the surrogate Gibbs-posterior of csMALA benefits from the full sample size thanks to the correction terms, whereas without the correction, it does not.
\begin{thm}[PAC-Bayes oracle inequality for the surrogate posterior of sMALA]
\label{thm:oracleinequalityStochMH} Under \ref{assu:bounded} and \ref{assu:F} there
are constants $\K{0}{'},\K{1}{'}>0$ depending only on 
$\Cf,\Ceps,\sigma$ such that for $\lambda=n\rho/\K{0}{'}$ and sufficiently large $n$ we have for all $\delta\in(0,1)$ with probability at least $1-\delta$ that
\[
\mathcal{E}(\hat f_{\lambda,\rho})\le4\mathcal{E}(f_{\theta^{*}})+\frac{\K{1}{'}}{n\rho}\big( \pardim \log(B\Delta n)-\log\operatorname{vol}(\mathcal B_1)+\log(2/\delta)\big).
\]

\end{thm}

In view of \ref{thm:oracleinequalityStochMH} the following
results are also true for the stochastic MH estimator
if $n$ is replaced by $n\rho$. However, we subsequently focus only on the analysis
of $\tilde{\Pi}_{\lambda,\rho}(\cdot\mid\mathcal{D}_{n})$ for the sake of clarity.

\medskip{}

The $1-\delta$ probability in \ref{thm:oracleinequality} takes into account the randomness
of the data and of the estimate. Denoting
\begin{equation}
\rate^{2}\coloneqq 12\Vert f_{\theta^{*}}-f\Vert_{L^{2}(\P^{\mathbf{X}})}^{2}+\frac{\K{1}}{n}\big( \pardim\log(B\Delta n)-\log\operatorname{vol}(\mathcal B_1)\big),\label{eq:r2}
\end{equation}
we can rewrite \ref{eq:oracleIneq} as
\[
\E\big[\tilde{\Pi}_{\lambda,\rho}\big(\Vert f_{\tilde{\theta}_{\lambda,\rho}}-f\Vert_{L^{2}(\P^{\mathbf{X}})}^{2}>\rate^{2}+t^{2}\,\big\vert\,\mathcal{D}_{n}\big)\big]\le2\e^{-nt^{2}/\K{1}{} },\qquad t>0,
\]
which is a \emph{contraction rate} result in terms of a frequentist
analysis of the nonparametric Bayes method.

An immediate consequence is an oracle inequality for the posterior mean $\bar{f}_{\lambda,\rho}$ from \ref{eq:postMeanRho}.
\begin{cor}[Posterior mean]
\label{cor:mean}Under the conditions of \ref{thm:oracleinequality}
we have with probability at least $1-\delta$ that
\[
\mathcal{E}(\bar{f}_{\lambda,\rho})\le 12\mathcal{E}(f_{\theta^{\ast}})+\frac{\K{2}{}}{n}\big( \pardim \log(B\Delta n)-\log\operatorname{vol}(\mathcal B_1)+\log(2/\delta)\big)
\]
with a constant $\K{2}{}$ only depending on $\Cf,\Ceps,\sigma$ from \ref{assu:bounded}.
\end{cor}

\subsection{Credible sets}

In addition to the contraction rates, the Bayesian approach offers a gateway to uncertainty quantification. To this end, we will study \emph{credible balls} of the form
\begin{equation}\label{eq:credibleSet}
\widehat{C}(\tau_{\alpha}) \coloneqq\{h\in L^{2}:\Vert h-\bar{f}_{\lambda,\rho}\Vert_{L^{2}(\P^{\mathbf{X}})}\le\tau_{\alpha}\},\qquad\alpha\in(0,1),
\end{equation}
with critical values 
\[
\tau_{\alpha}\coloneqq\argmin{\tau>0}\big\{\tilde{\Pi}_{\lambda,\rho}(\theta:\Vert f_{\theta}-\bar{f}_{\lambda,\rho}\Vert_{L^{2}(\P^{\mathbf{X}})}\le\tau\mid\mathcal{D}_{n})>1-\alpha\big\}.
\]
By construction $\hat C(\tau_{\alpha})$ is the smallest $L^{2}$-ball around $\bar{f}_{\lambda,\rho}$ that contains $1-\alpha$ mass of
the surrogate posterior measure. 
While we assume that the distribution $\P^{\mathbf{X}}$ of $\mathbf{X}$ is known, to calculate this ball in practice, the distribution of $\mathbf{X}$ could be replaced by its empirical analog.

Despite the posterior belief,
it is not necessarily guaranteed that the true regression function
is contained in $\hat C(\tau_{\alpha})$. More precisely, the posterior
distribution might be quite certain, in the sense that the credible
ball is quite narrow, but suffers from a significant bias. In general,
it might happen that $\P(f\in\hat C(\tau_{\alpha}))\to0$, see, \emph{e.g.}, 
\citet[Theorem 4.2]{KnapiketAl2011} in a Gaussian model. To circumvent this, \citet{RousseauSzabo2020} have introduced
inflated credible balls where the critical value is multiplied with
a slowly diverging factor. 
While they proved that this method works in several classical nonparametric models with a sieve prior, aiming for a 
neural network setting causes an additional problem. 

In order to prove
coverage, \emph{i.e.}, to prove that the credible ball contains the true regression function with high probability, we would like to compare norms in the intrinsic parameter
space, \emph{i.e.}, the space of the network weights, with the norm of the
resulting predicted regression function. While the fluctuation of $f_\theta$ can be controlled via the fluctuation of $\theta$, more precisely we have $\Vert f_{\theta}-f_{\theta'}\Vert_{L^{2}(\P^{\mathbf{X}})}\le \Delta|\theta-\theta'|_\Theta$ by \ref{assu:F}, the converse direction does not hold in general. For instance, even locally around an
oracle choice $\theta^{*}$ we cannot hope to control $| \theta|_\qnorm $
via $\Vert f_{\theta}\Vert_{L^{2}(\P^{\mathbf{X}})}$ for any $1\le \qnorm\le\infty$ in view of the ambiguous
network parametrization. As a consequence, we define another critical
value at the level of the parameter space 
\[
\tau_{\alpha}^{\theta_0}\coloneqq\argmin{\tau>0}\big\{\tilde{\Pi}_{\lambda,\rho}(\theta:\vert \theta-\theta_0\vert_{\Theta}\le\Delta^{-1}\tau\mid\mathcal{D}_{n})>1-\alpha\big\}
\]
for some center $\theta_0\in(-B, B)^\pardim$. 
Both critical values measure the fluctuation of the posterior. The theoretical properties
of the credible ball are summarized in the following theorem:
\begin{thm}[Credible balls]\label{thm:credibility}
\label{thm:credibleSet} Under \ref{assu:bounded}, \ref{assu:F} and with constants $\K{0},\K{1},\K{2}>0$ from above we have for $\lambda=n/(\qnorm\K{0})$ with $\qnorm\ge2$, $\rate^2$ from \ref{eq:r2} and  sufficiently large $n$ that
\[
\P\Big(\diam\big(\widehat{C}(\tau_{\alpha})\big)\le4\sqrt{\qnorm\rate^{2}+\frac{2\qnorm(\K{1}\vee\K{2})}{n}\log\frac{2}{\alpha}}\Big)\ge1-\alpha.
\]
If $\mathcal E(f_{\theta^*})= \mathcal{O}( \pardim/\lambda)$,
then we have for sufficiently large $n\in\N$
\[
\P\big(f\in\widehat{C}(\xi\tau_{\alpha}^{\theta_0})\big)\ge1-\alpha,
\]
where $\xi=\Cxi\sqrt{\log(B\Delta n)}$ with some sufficiently large $\Cxi>0$ if $\theta_0$ does not depend on $\mathcal D_n$ and $\xi=(\rate/\Delta)^{-2/p}$ otherwise.
\end{thm}

Therefore, the order of the diameter of $\hat C(\tau_{\alpha})$ coincides with the contraction rate deduced from \Cref{thm:oracleinequality}. On the other hand, the larger
credible set $\widehat{C}(\xi\tau_{\alpha}^{\theta})$ defines an
honest confidence set for a fixed class of the regression functions if $\xi$ is chosen sufficiently large. That is, $f$
is contained in $\widehat{C}(\xi\tau_{\alpha}^{\theta_0})$ with probability at least $1-\alpha$. 
If $B$ and $\Delta$ grow at most polynomially in $n$ and $\theta_0$ is data independent, we can conclude from \ref{thm:credibleSet} that
\[
\P\big(f\in\widehat{C}(\tau_{\alpha}^{\theta_0}\log n)\big)\ge1-\alpha\qquad\text{for sufficiently large }n,
\]
which is in line with the inflation factor by \citet{RousseauSzabo2020}. 

The condition $\mathcal E(f_{\theta^*})= \mathcal{O}( \pardim/\lambda)=o(\pardim\log(B\Delta n)/\lambda)$ for
the coverage result means that the rate is dominated by the stochastic
error term, cf. \ref{eq:oracleIneq}. This can be achieved with a slightly larger parameter class than an optimal choice which balances the approximation error $\mathcal E(f_{\theta^*})$ and the stochastic error. This guarantees that the posterior is not underfitting
and that the posterior's bias is covered by its dispersal. The necessity of under-smoothing is well known for statistical inference.

\begin{example}[Linear model] We illustrate the above credibility statement for a linear model $f_\theta(\mathbf x)=\mathbf x^\top\theta $ with $\mathbf x\in[0,1]^\inputdim,\theta\in[-B,B]^\inputdim$ and true regression function $f=f_{\theta^*}$. Hence, the $L^2$-loss is 
\[
\mathcal E(f_\theta)=\E[(\mathbf X^\top\theta-\mathbf X^\top\theta^*)^2]=(\theta-\theta^*)^\top \E[\mathbf{X}\mathbf{X}^\top](\theta-\theta^*).
\]
In particular, \ref{assu:F} is satisfied for $|\cdot|_\Theta=|\cdot|_2$ and $\Delta^2=\lambda_{\max}(\Sigma)$ being the largest eigenvalue of the design matrix $\Sigma=\E[\mathbf{X}\mathbf{X}^\top]$. If $\Sigma$ is positive definite with smallest eigenvalue $\lambda_{\min}(\Sigma)>0$, the norms $\|f_\theta\|_{L^2(\P^{\mathbf X})}$ and $|\theta|_2$ are equivalent. \ref{thm:oracleinequality} then yields $\mathcal E(\tilde f_{\lambda,\rho})=\mathcal O(\frac \inputdim n\log n)$ with high probability.

Due to the linearity of $\theta\mapsto f_\theta$, the mean predictor from \ref{eq:postMeanRho} satisfies $\bar f_{\lambda,\rho}(\mathbf{X})=\mathbf{X}\bar\theta_{\lambda,\rho}$ with $\bar\theta:=\bar\theta_{\lambda,\rho}:=\E[\tilde\theta_{\lambda,\rho}\mid \mathcal D_n]$. Therefore,
\begin{align*}
   \tau_\alpha&=\argmin{\tau>0}\big\{\tilde{\Pi}_{\lambda,\rho}(\theta:\Vert \mathbf{X}(\theta-\bar{\theta}_{\lambda,\rho})\Vert_{L^{2}(\P^{\mathbf{X}})}\le\tau\mid\mathcal{D}_{n})>1-\alpha\big\}\\
  &\le\argmin{\tau>0}\big\{\tilde{\Pi}_{\lambda,\rho}(\theta:| \theta-\bar{\theta}_{\lambda,\rho}|\le\Delta^{-1}\tau\mid\mathcal{D}_{n})>1-\alpha\big\}\\
  &=\tau_\alpha^{{\bar\theta}_{\lambda,\rho}}\le\argmin{\tau>0}\Big\{\tilde{\Pi}_{\lambda,\rho}\Big(\theta:\Vert \mathbf{X}(\theta-\bar{\theta}_{\lambda,\rho})\Vert_{L^{2}(\P^{\mathbf{X}})}\le\frac{\lambda_{\min}(\Sigma)}{\Delta}\tau\,\Big|\,\mathcal{D}_{n}\Big)>1-\alpha\Big\}
  =\frac{\lambda_{\max}(\Sigma)}{\lambda_{\min}(\Sigma)}\tau_\alpha.
\end{align*}
\Cref{thm:credibleSet} thus yields  $\P\big(f\in\hat C\big(\big(\tfrac{\inputdim}{n}\big)^\kappa\tau_\alpha\big)\big)\ge 1-\alpha$ for an arbitrary small $\kappa>0$.

To illustrate the uncertainty quantification on both the function level and the parameter level, we consider a small simulation example with i.i.d.\ data $\mathcal{D}_n=(\mathbf{x}_i,y_i)_{1\le i\le n}$ for $n=200,\inputdim=2$ drawn from $Y=\mathbf{X}^\top\theta^\ast+\eps$ with $\theta^\ast=(0.5,0.5)^\top$, $\mathbf{X}=(1,U), U\sim\mathcal{U}(0,1),\eps\sim\mathcal{N}(0,0.1^2)$. We applied MALA with $\lambda=n,\gamma=0.01$, sMALA with $\rho=0.1,\lambda=n$ and csMALA with $\rho=0.1,\lambda=n/\rho,\gamma=0.01/\rho,\zeta=0.2$. We used a standard deviation of the proposals of $s=0.1$, a burn-in time of $b=1000$ steps and a gap length of $c=100$ to draw $N=100$ samples from each algorithm and calculate the critical values. \ref{fig:linearmodel} illustrates a typical run. All credible sets achieved a coverage of $100\%$ on both the function level and the parameter level, even without an inflating factor $\xi$ needed in theory. The radii of the credible sets generated by MALA and csMALA are comparable to each other and much smaller than that of sMALA, which suggests that the algorithm benefits from the correction term.

This impression is confirmed by a Monte Carlo simulation with $100$ iterations to approximate the coverage and radii of the credible sets. There, all credible sets achieve a coverage of $100\%$ and the average radii on the function level (parameter level) are $0.1176$ $(0.2867)$ for MALA, $0.1548$ $(0.3802)$ for sMALA and $0.1167$ $(0.2838)$ for csMALA.

\begin{figure}[h]
\includegraphics[width=\textwidth]{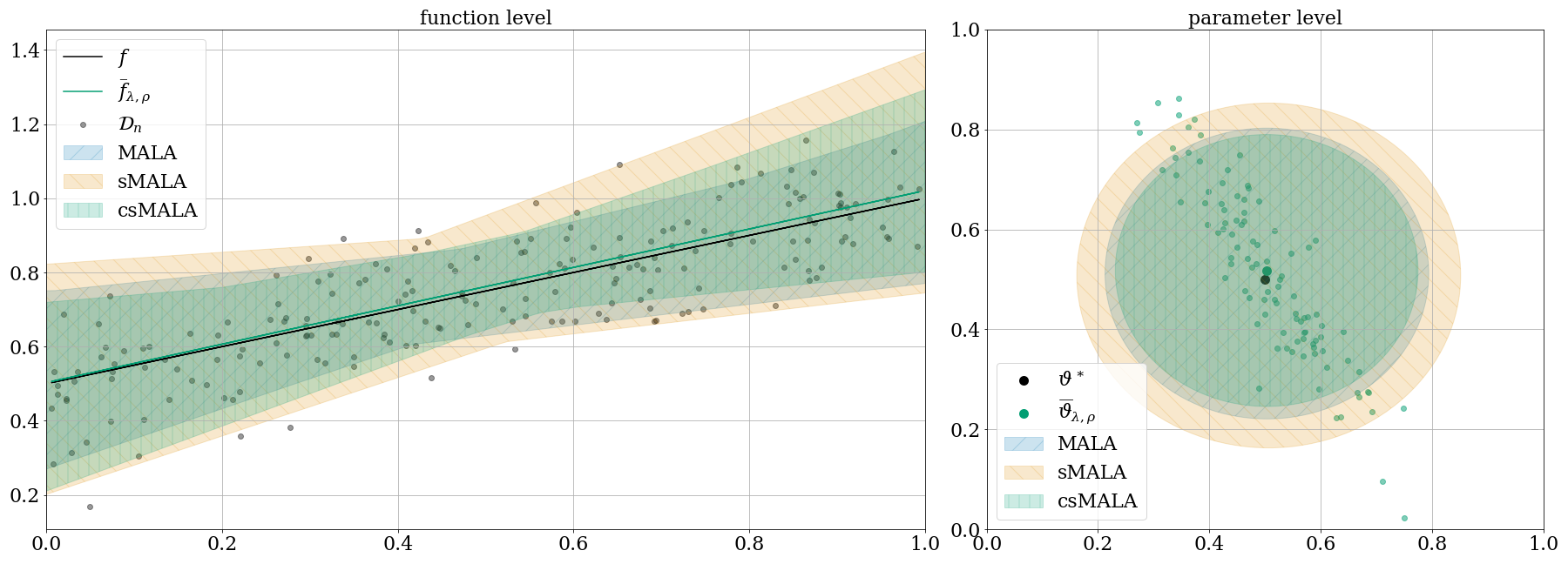}
\caption{Illustration of credible balls in a linear model. The left hand side displays the data, the true regression function and the (approximate) posterior mean of csMALA. The hatched areas are bounded by the pointwise minimum and maximum sampled prediction from MALA, sMALA and csMALA, respectively. The right hand side shows the credible balls on the parameter level, centered at their respective posterior means. The faint dots indicate $100$ parameters sampled using csMALA.}
\label{fig:linearmodel}
\end{figure}

\end{example}

\section{Application to Bayesian neural networks}\label{sec:NN}

In the sequel the estimator $\hat f$ is chosen as a neural network such that we provide statistical guarantees for stochastic neural networks.
More precisely, we consider a \emph{feedforward multilayer perceptron}
with $\inputdim\in\N$ inputs, $L\in\N$ hidden layers and constant width $r\in\N$. The latter restriction is purely for notational convenience.
The \emph{rectified linear unit} (ReLU) $\phi(x)\coloneqq\max\{x,0\},x\in\R,$ is used as activation
function. We write $\phi_{v}x\coloneqq\big(\phi(x_{i}+v_{i})\big)_{i=1,\dots,d}$
for vectors $x,v\in\R^{d}$. With this notation we can represent such neural networks as
\begin{equation}
g_{\theta}(\mathbf{x})\coloneqq W^{(L+1)}\phi_{v^{(L)}}W^{(L)}\phi_{v^{(L-1)}}\cdots W^{(2)}\phi_{v^{(1)}}W^{(1)}\mathbf{x}+v^{(L+1)},\qquad\mathbf{x}\in\R^{\inputdim},\label{eq:network}
\end{equation}
where the parameter vector $\theta$ contains all entries of the weight
matrices $W^{(1)}\in\R^{r\times \inputdim}, W^{(2)},\dots,W^{(L)}\in\R^{r\times r}, W^{(L+1)}\in\R^{1\times r}$ and
the shift (`bias') vectors $v^{(1)},\dots,v^{(L)}\in\R^r, v^{(L+1)}\in\R$.
We thus set throughout
\[
\Theta=[-B,B]^\pardim\qquad\text{with}\qquad \pardim\coloneqq(\inputdim+1)r+(L-1)(r+1)r+r+1.
\]
A layer-wise representation of $g_{\theta}$ is given by
\begin{align}
\mathbf{x}^{(0)} & \coloneqq\mathbf{x}\in\R^{\inputdim},\\
\mathbf{x}^{(l)} & \coloneqq\phi(W^{(l)}\mathbf{x}^{(l-1)}+v^{(l)}),\,l=1,\dots,L,\label{eq:neurons}\\
g_{\theta}(\mathbf{x})\coloneqq\mathbf{x}^{(L+1)} & \coloneqq W^{(L+1)}\mathbf{x}^{(L)}+v^{(L+1)},
\end{align}
where the activation function is applied coordinate-wise. We denote
the class of all such functions $g_{\theta}$ by $\mathcal{G}(\inputdim,L,r)$.
For some $\Cf\ge1$, we also introduce the class of clipped networks
\[
\mathcal{F}(\inputdim,L,r,\Cf)\coloneqq\big\{ f_{\theta}=(-\Cf)\lor(g_{\theta}\land \Cf)\,\big\vert\, g_{\theta}\in\mathcal{G}(\inputdim,L,r)\big\}.
\]

\subsection{Shallow networks}\label{subsec:shallownets}

We start with shallow neural networks where $L=1$ and thus $\pardim=(\inputdim+2)r+1$. The study of the approximation properties of shallow networks has a long history going back to \cite{barron1993universal} and is closely related to their variational spaces \cite{bach2017breaking}. Building on this branch of literature, the following result is a corollary of \cite[Theorem 2]{yang2024nonparametric}, where we formulate the approximation properties of shallow networks with uniformly bounded weights.
Let $\mathcal{C}_p^{\beta}([0,1]^{\inputdim},\cholder)$ denote classical H\"older balls of functions $f\colon[0,1]^\inputdim\to\R$ with H\"older regularity $\beta>0$ and H\"older norms bounded by $\cholder>0$.

\begin{lem}\label{lem:ApproxShallow} For any $0<\beta<(\inputdim+3)/2,\cholder>0$ there is a constant $\cshallow=\cshallow(\beta,\cholder)>0$ such that for any $f\in \mathcal C^\beta([0,1]^\inputdim,\cholder)$ there is a shallow network $g_\theta\in\mathcal G(\inputdim,1,r)$ with weights uniformly bounded by $B=1\vee r^{(3-\inputdim-2\beta)/(2 \inputdim)}$ such that
\[
\|g_\theta-f\|_{L^\infty([0,1]^\inputdim)}\le \cshallow r^{-\beta/\inputdim}.
\]
\end{lem}

If $\inputdim\ge3$ or $\beta>\frac{3-\inputdim}{2}$, all network weights can thus be uniformly bounded by one. In combination with a Lipschitz bound of the neural networks with respect to the parameters (see \ref{lem:Lipschitz} below), \ref{thm:oracleinequality} then yields the following result:
\begin{thm}[PAC-Bayes oracle inequality for shallow neural networks]
\label{thm:oracleinequalityShallow} Let $\E[\vert\mathbf{X}\vert^2]<\infty$. Under \ref{assu:bounded} there are
constants $\K{0},\K{1}>0$ depending only on $\Cf,\Ceps,\sigma$ such that for $\lambda=n/\K{0}$ and sufficiently large $n$ we have for all $\delta\in(0,1)$ with probability  at least $1-\delta$ that
\begin{equation}
\mathcal{E}(\tilde f_{\lambda,\rho})\le12\mathcal{E}(f_{\theta^{*}})+\frac{\K{1}}{n}\big( \pardim \log(rBn)+\log(2/\delta)\big).\label{eq:oracleIneq2}
\end{equation}
If $\mathbf{X}\in[0,1]^\inputdim$ and $f\in \mathcal C^\beta([0,1]^\inputdim,\cholder)$ for some $0<\beta<(\inputdim+3)/2,\cholder>0$ we obtain for some $\K{2}>0$
\begin{equation}
\mathcal{E}(\tilde f_{\lambda,\rho})\le \K{2}\big(r^{-2\beta/\inputdim}+\frac{r}{n}\log (rBn)+\log(2/\delta)\big).
\end{equation}
A shallow network with number of neurons $r$ of order $n^{\inputdim/(2\beta+\inputdim)}$ thus yields $\mathcal{E}(\tilde f_{\lambda,\rho})=\mathcal{O}_\P(n^{-2\beta/(2\beta+\inputdim)}\log n)$. All bounds hold true for $\bar{f}_{\lambda,\rho}$, too.
\end{thm}
Therefore, the shallow Bayesian neural network achieves the minimax optimal rate of convergence up to the factor $\log n$. The result is in line with \cite{tinsiDalalyan2022} who use generally aggregated shallow networks but with Gaussian priors. For empirical risk minimizers the rate has been verified in \cite{yang2024nonparametric}. Towards uncertainty quantification, we can deduce the following statement from \ref{thm:credibleSet}.

\begin{prop}[Credible balls for shallow neural networks]\label{prop:UQ}
    Let $\mathbf{X}\in[0,1]^\inputdim$. Under \ref{assu:bounded} let $f\in \mathcal C^\beta([0,1]^\inputdim,\cholder)$ for some $0<\beta<(\inputdim+3)/2,\cholder>0$. Then there are constants $\K{0},\K{1},\K{2}>0$ such that for $\lambda=n/\K{0}$, $r=\K{1}n^{\inputdim/(2\beta +\inputdim)}$ and any $\alpha\in(0,1)$ the credible set from \ref{eq:credibleSet} satisfies for sufficiently large $n$ that
\[
\P\big(\diam\big(\widehat{C}(\tau_{\alpha})\big)\le \K{2}n^{-\beta/(2\beta+\inputdim)}\big)\ge1-\alpha.
\]
For $\tau_{\alpha}^{\theta_0}=\argmin{\tau>0}\big\{\tilde{\Pi}_{\lambda,\rho}(\theta:\vert \theta-\theta_0\vert_{1}\le\tau/B\mid\mathcal{D}_{n})>1-\alpha\big\}$
for any (non-random) $\theta_0\in(-B,B)^\pardim$ we have for $n\to\infty$
\[
\P\big(f\in\widehat{C}\big(\log(n) \tau_{\alpha}^{\theta_0}\big)\big)\ge1-\alpha.
\]
\end{prop}

\subsection{Deep neural networks}

For deep neural networks with $L\ge2$ we obtain the following corollary of \ref{thm:oracleinequality}:
\begin{thm}[PAC-Bayes oracle inequality for deep neural networks]
\label{thm:oracleinequalityNN} Let $\E[\vert\mathbf{X}\vert^2]<\infty$. Under \ref{assu:bounded} there are
constants $\K{0},\K{1}>0$ depending only on $\Cf,\Ceps,\sigma$ such that for $\lambda=n/\K{0}$ and sufficiently
large $n$ we have for all $\delta\in(0,1)$ with probability at least $1-\delta$ that
\begin{equation}
\mathcal{E}(\tilde f_{\lambda,\rho})\le12\mathcal{E}(f_{\theta^{*}})+\frac{\K{1}}{n}\big( \pardim L\log(rBn)+\log(2/\delta)\big).\label{eq:oracleIneq3}
\end{equation}
Moreover, we have with probability at least $1-\delta$ that
\[
\mathcal{E}(\bar{f}_{\lambda,\rho})\le 12\mathcal{E}(f_{\theta^{\ast}})+\frac{\K{2}}{n}\big( \pardim L\log(rBn)+\log(2/\delta)\big)
\]
with a constant $\K{2}$ only depending on $\Cf,\Ceps,\sigma$ from \ref{assu:bounded}.
\end{thm}

Using the approximation properties of deep neural networks, the oracle
inequality yields the optimal rate of convergence (up to a logarithmic
factor) over the following class of hierarchical functions:
\begin{align*}
\mathcal{H}(\hdepth,\mathbf{d},\mathbf{t},\beta,\cholder) & \coloneqq\Big\{ g_{\hdepth}\circ\dots\circ g_{0}\colon[0,1]^{\inputdim}\to\R\,\Big\vert\,g_{i}=(g_{ij})_{j}^{\top}\colon[a_{i},b_{i}]^{d_{i}}\to[a_{i+1},b_{i+1}]^{d_{i+1}},\\
 & \qquad\qquad\qquad\qquad\qquad\qquad\qquad g_{ij}\text{ depends on at most \ensuremath{t_{i}} arguments,}\\
 & \qquad\qquad\qquad\qquad\qquad\qquad\qquad g_{ij}\in\mathcal{C}_{t_{i}}^{\beta_{i}}([a_{i},b_{i}]^{t_{i}},\cholder),\text{ for some }\vert a_{i}\vert,\vert b_{i}\vert\le \cholder\Big\},
\end{align*}
where $\mathbf{d}\coloneqq(\inputdim,d_{1},\dots,d_{\hdepth},1)\in\N^{\hdepth+2},\mathbf{t}\coloneqq(t_{0},\dots,t_{\hdepth})\in\N^{\hdepth+1},\beta\coloneqq(\beta_{0},...,\beta_{q})\in(0,\infty)^{\hdepth+1}$. For a detailed discussion of $\mathcal{H}(\hdepth,\mathbf{d},\mathbf{t},\beta,\cholder) $, see \cite{schmidthieber2020}.
\ref{thm:oracleinequalityNN} reveals the following convergence rate which is in line with the upper
bounds by \citet{schmidthieber2020} and \citet{KohlerLanger2021}:

\begin{prop}[Rates of convergence]
\label{prop:rate} Let $\mathbf{X}\in[0,1]^{\inputdim}$. In the situation of \ref{thm:oracleinequalityNN},
there exists a network architecture $(L,r)=(C_1\log n,C_2(n/(\log n)^3)^{t^\ast/(4\beta^\ast+2t^\ast)})$ with $C_{1},C_{2}>0$ only depending on upper bounds for $\hdepth,\vert \mathbf{d}\vert_\infty,\vert\beta\vert_\infty ,C_0$ such that the estimators $\tilde f_{\lambda,\rho}$ and
$\bar{f}_{\lambda,\rho}$ with $B=C_3n$ for some $C_3>0$ satisfy for sufficiently large
$n$ uniformly over all hierarchical functions $f\in\mathcal{H}(\hdepth,\mathbf{d},\mathbf{t},\beta,C_{0})$
\begin{align*}
\mathcal{E}(\tilde f_{\lambda,\rho}) & \le\K{3}\Big(\frac{(\log n)^3}{n}\Big)^{2\beta^{\ast}/(2\beta^{\ast}+t^{\ast})}+\K{3}\frac{\log(2/\delta)}{n}\qquad\text{and}\\
\mathcal{E}(\bar{f}_{\lambda,\rho}) & \le\K{4}\Big(\frac{(\log n)^3}{n}\Big)^{2\beta^{\ast}/(2\beta^{\ast}+t^{\ast})}+\K{4}\frac{\log(2/\delta)}{n}
\end{align*}
with probability  at least $1-\delta$, respectively, where
$\beta^{*}$ and $t^{*}$ are given by 
\[
\beta^\ast\coloneqq\beta^\ast_{i^\ast},\qquad t^\ast\coloneqq t^\ast_{i^\ast}\qquad\text{for}\qquad i^\ast\in \argmin{i=0,\dots,\hdepth}\frac{2\beta_{i}^{*}}{2\beta_{i}^{*}+t^{\ast}_{i}}\qquad\text{and}\qquad 
\beta_{i}^{\ast}\coloneqq\beta_{i}\prod_{l=i+1}^{\hdepth}(\beta_{l}\land 1).
\]
The constants $\K{3}$ and $\K{4}$ only depend on upper bounds for $q,\mathbf{d},\beta$
and $C_{0}$ as well as the constants from \ref{assu:bounded}.
\end{prop}

It has been proved by \citet{schmidthieber2020} that this is the minimax optimal rate of convergence for the nonparametric estimation of $f$ up to logarithmic factors. Studying the special case of classical H\"older balls 
$\mathcal{C}_{\inputdim}^{\beta}([0,1]^{\inputdim},C_{0})$, a contraction rate of order $n^{-2\beta/(2\beta+\inputdim)}$ has been derived for deep neural networks by \citet{polsonRockova2018}, \citet{cherief2020}. Hierarchical regression functions have also been studied by \cite{castillo2024posterior}.

\subsection{Adaptive choice of the network width\label{sec:Adaptation}}
To balance the approximation error term and the stochastic error term in \ref{eq:r2}, we have to choose an optimal network width. In this section we present a fully data-driven approach to this hyperparameter optimization problem which avoids evaluating competing network architectures on a validation set.
To account for the model selection problem, we augment the approach with a mixing prior, which prefers narrower neural networks. Equivalently, this approach can be understood as a hierarchical Bayes method where we put a geometric distribution on the hyperparameter $r$. While this method has interesting theoretical properties, an efficient implementation is challenging and left for future research. Still, it can be seen as a link to the literature on model selection through hierarchical priors, see, \emph{e.g.}, \cite{Guedj2013} for additive models and \cite{SteffenTrabs2025} for sparse neural networks.

We set
\[
\widecheck{\Pi}=\sum_{r=1}^n 2^{-r}\Pi_r\Big/(1-2^{-n}),
\]
where $\Pi_{r}=\mathcal{U}([-B,B]^{\pardim_r})$ with 
\[
\pardim_r\coloneqq (\inputdim+1)r+(L-1)(r+1)r+r+1.
\]
The basis $2$ of the geometric weights is arbitrary and can be replaced by a larger constant to assign even less weight to wide networks, but the theoretical results remain the same up to constants.

We obtain our adaptive estimator $\widecheck{f}_{\lambda,\rho}$ by drawing a parameter $\theta$  from the surrogate-posterior distribution with respect to this prior, \emph{i.e.}, 
\begin{equation}
    \widecheck{f}_{\lambda,\rho}\coloneqq f_{\widecheck{\theta}_{\lambda,\rho}} \qquad \text{for} \qquad \widecheck{\theta}_{\lambda,\rho}\mid\mathcal{D}_n \sim \widecheck{\Pi}_{\lambda,\rho}(\cdot\mid\mathcal{D}_n)\qquad\text{with}\qquad \widecheck{\Pi}_{\lambda,\rho}(\theta\mid\mathcal{D}_n)\propto\exp\big(-\lambda \tilde{R}_{n,\rho}(\theta)\big)\widecheck{\Pi}(\d \theta).
\end{equation}
This modification allows the estimator to adapt to the optimal network width and we can compare its performance against that of the network corresponding the oracle choice of the parameter
\[\theta^{\ast}_r\in\argmin{\theta\in[-B,B]^{\pardim_r}}R(\theta)\label{eq:oracleL}
\]
given any width $r$. We obtain the following adaptive version of \ref{thm:oracleinequality}:

\begin{thm}[Width-adaptive oracle inequality]\label{thm:learningthewidth}
Under \ref{assu:bounded} and if $\E[\vert\mathbf{X}\vert^2]<\infty$, there is a constant
 $\K{1}>0$ depending only on $\Cf,\Ceps,\sigma$ such that for $\lambda=n/\K{0}$ (with $\K{0}$ from \ref{thm:oracleinequality}) and sufficiently
large $n$ we have for all $\delta\in(0,1)$ with probability  at least $1-\delta$ that
\begin{equation}
\mathcal{E}(\widecheck{f}_{\lambda,\rho})\le\min_{r=1,\dots,n} \Big(12\mathcal{E}(f_{\theta^{\ast}_r})+\frac{\K{1}}{n}\big({\pardim_r} L\log(rBn)+\log(2/\delta)\big)\Big).
\end{equation}

\end{thm}

Since the modified estimator mimics the performance of the optimal network choice regardless of width, we obtain the following width-adaptive version of \ref{prop:rate} with no additional loss in the convergence rate:
\begin{cor}[Width-adaptive rates of convergence]\label{cor:rate_adaptive}
Let $\mathbf{X}\in[0,1]^{\inputdim}$. In the situation of \ref{thm:learningthewidth},
there exists a network depth $L=C_4\log n$ with $C_4>0$ only depending on upper bounds for $\hdepth,\vert \mathbf{d}\vert_\infty,\vert\beta\vert_\infty ,\cholder$ such that the estimator $\widecheck{f}_{\lambda,\rho}$ satisfies for sufficiently large
$n$ and $B=C_5n$ uniformly over all hierarchical functions $f\in\mathcal{H}(\hdepth,\mathbf{d},\mathbf{t},\beta,\cholder)$
\begin{align*}
\mathcal{E}(\widecheck{f}_{\lambda,\rho}) & \le\K{2}\Big(\frac{(\log{n})^3}{n}\Big)^{2\beta^{\ast}/(2\beta^{\ast}+t^{\ast})}+\K{2}\frac{\log(2/\delta)}{n}
\end{align*}
with probability
  at least $1-\delta$, where $\beta^{\ast}$ and $t^{\ast}$ are as in \ref{prop:rate}.
The constant $\K{2}$ only depends on upper bounds for $\hdepth,\vert\mathbf{d}\vert_\infty,\vert \beta\vert_\infty$
and $\cholder$ as well as the constants $\Cf,\Ceps,\sigma$ from \ref{assu:bounded}.
\end{cor}

For sparse neural networks, contraction rates for hierarchical Bayes procedures have been analyzed by \citet{polsonRockova2018} and \citet{SteffenTrabs2025}. Adaptivity with heavy-tailed weights has been achieved by \cite{castillo2024posterior}. It should be noted that we cannot hope to construct credible sets with coverage as in \ref{thm:credibleSet} based on the adaptive posterior
distribution. It is well known that adaptive honest confidence sets
are only possible under additional assumptions, \emph{e.g.},  self-similarity
or polished tail conditions, on the regularity of the regression function, see \citet{HoffmannNickl2011}
and we remark that such conditions with respect to the network parametrization
seem infeasible.

\section{Numerical examples\label{sec:Numerics}}

\begin{figure}[h]
\includegraphics[width=\textwidth]{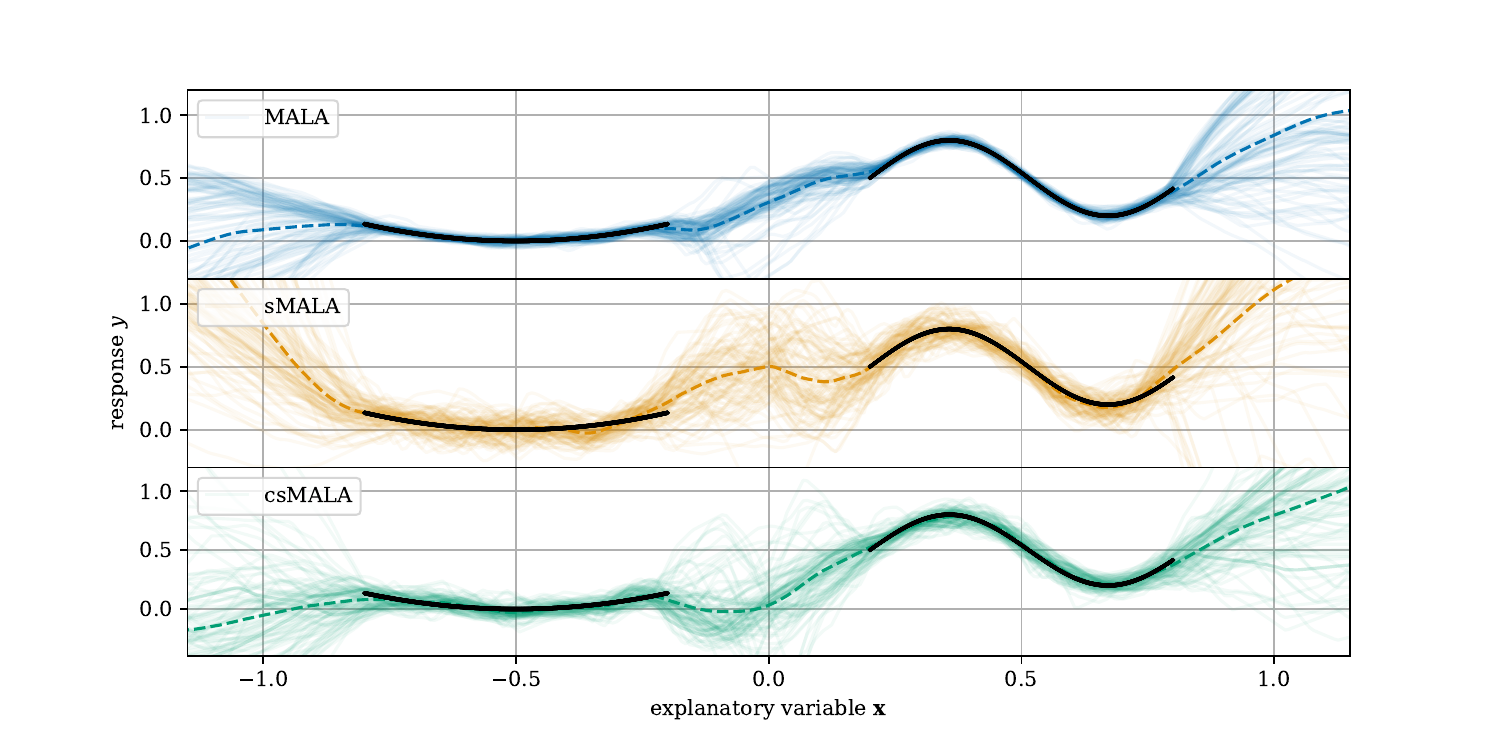}
\caption{$100$ samples drawn from the different MALA chains, given a training sample (black markers) of $10000$ points. Random variables are drawn for $\rho = 0.1$. The dashed line shows the corresponding posterior mean $\bar{f}_\lambda$.}
\label{fig:fit}
\end{figure}

\ref{ssec:csMALA} introduces a correction to the batch-wise approximation of the empirical risk when calculating the MH step.
In the following, we will show the merit of this correction for learning a one-dimensional regression function using a feed-forward neural network of $L=2$ layers of $r=100$ nodes each and ReLU activation.
The neural network has a total number of $10401$ parameters.
The training sample of size $10000$ consist of two equally populated intervals $ [-0.8, -0.2]$ and $ [0.2, 0.8]$ with $\mathbf{X}_i\sim\mathcal{U}([-0.8,-0.2]\cup [0.2,0.8])$ and true regression function 
\begin{equation*}
    f(x) = \begin{cases*}
            1.5 (x+0.5)^2  & for  $x < 0$  \\
            0.3 \sin{(10x-2)} + 0.5 & for $x\ge 0$.
            \end{cases*}
\end{equation*}
We generate the training sample following $Y=f(\mathbf{X})+\eps$ by adding an observation error $\eps \sim \mathcal{N}(0,0.02^2)$. 
In the interval between $-0.2$ and $0.2$ no data is produced in order to illustrate whether the methods recover the resulting large uncertainty due to missing data. 
For a sufficiently flexible model we expect a large spread between samples from each Markov chain in this region.
\ref{fig:fit} depicts exactly this behaviour, as well as the training sample.

To compare the convergence of MALA, stochastic MALA (sMALA) and our corrected stochastic MALA (csMALA) within reasonable computation time, we initialize the chains with network parameters obtained through optimization of the empirical risk with stochastic gradient descent for $2000$ steps. 
For this pre-training, we use a learning rate of $10^{-3}$. 
The hyperparameters of the subsequent chains are listed in \ref{tab:params}.
The inverse temperature is chosen to counteract the different normalization terms of the risk for (s)MALA and csMALA, as well as the reduction of the learning rate by $(2-\rho)$ through the correction term from \ref{ssec:csMALA}.
The proposal noise level per parameter dimension is normalized with respect to the number of network parameters such that the total length of the noise vector is independent of the parameter space dimension.

To further improve the efficiency of the sampling, we restart \ref{alg:MCMCV2} with $\vartheta^{(0)}$ set to the last accepted parameters whenever no proposal has been accepted for $100$ steps.
Especially for small $\rho$ and large $\eps$, the stochastic MH algorithms exhibit the tendency to get stuck after accepting an outlier batch with low risk.

\begin{figure}[t]
\centering
\begin{minipage}{0.4\linewidth}
    \begin{tabular}{l|ccc}
     & MALA               & sMALA              & csMALA             \\ \hline
    $\lambda$    & $n$ & $n \cdot \rho$  & $n \cdot (2-\rho)$  \\
    $\gamma$    & $10^{-4}$ & $10^{-4}$ & $10^{-4}/\rho$  \\
    $s$ & & $0.2/\sqrt{\pardim}$ & \\
    $b$ & & $100000/\rho$ &  \\
    $c$ & & $5000$ & \\
    $N$ & & $20$ & 
    \end{tabular}
    \vspace{1.25cm}
    \captionof{table}{Parameter choice for the different MALA chains. For $\rho = 0.1$, we chose a burn-in of $b = 50000$ to keep computation costs low.}
    \label{tab:params}
\end{minipage}
\hfill
\begin{minipage}{0.57\linewidth}\centering
    \includegraphics[width=1.\linewidth]{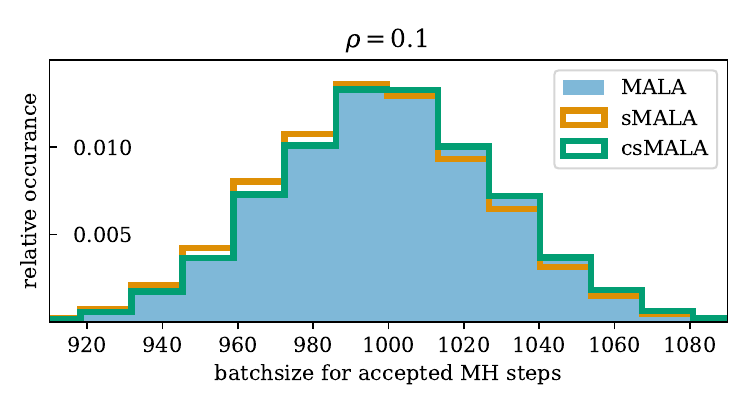}
  \caption{Histogram of the summed auxiliary variables, that is the number of training samples contributing to the stochastic risk, for all accepted steps.
  For MALA the MH acceptance step is calculated on the full sample and the distribution of the samples contribution to the risk gradients is thus unbiased by the batch size.}
  \label{fig:correction}
\end{minipage}
\end{figure}

It is also important to adapt $\zeta$ such that 
\begin{equation*}
    -\zeta\frac{\log{\rho}}{\lambda} \approx \frac{1}{n}\sum_{i=1}^{n}\ell_{i}(\theta^{(k)}).
\end{equation*}
For $\zeta$ lower than this, a bias is introduced towards accepting updates where many points of the data sample contributed to the stochastic risk approximation due to the Bernoulli distributed auxiliary variables.
Conversely, for higher values updates are preferably accepted for low amounts of points in the risk approximation.
This bias to small batches, note the minus sign due to $\log{\rho}$, can also be observed for the uncorrected sMALA.
It arises from the dependence of $R_n$ on the sum of the drawn auxiliary variables $Z_i$.
\ref{fig:correction} shows a histogram of this sum for all accepted steps.
A clear bias for sMALA towards small batches can be seen.
To achieve a good correction, we update $\zeta$ every $100$ steps to fulfill the preceding correspondence.
Over the chain, the correction factor thus falls like the empirical risk with $\zeta \ll 1$ due to the proportionality to $n^{-1}$.

\begin{figure}[t]
\includegraphics[width=\textwidth]{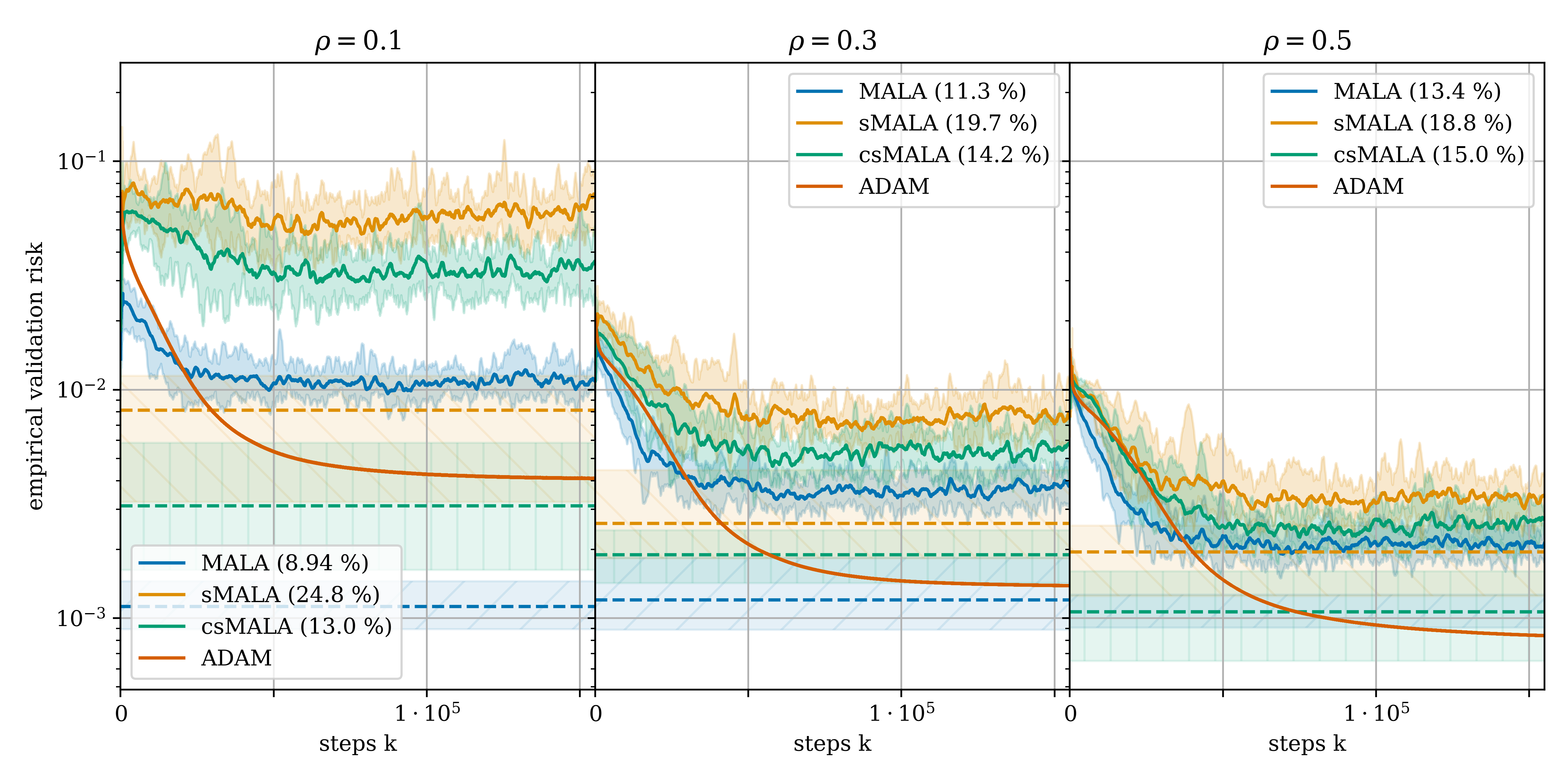}
\caption{Average empirical risk on a validation set of $10000$ points during running of the MALA chains. We show different batch probabilities $\rho$, as well as the values of the posterior mean (dashed lines). Uncertainties correspond to the minimum and maximum values of $10$ identical chains. For clarity, a the simple moving average over 1501 steps is plotted. In the legend, the average acceptance probability over all $10$ chains is given. For easier interpretation of the risk values, we also show the behavior of a gradient-based optimization using \texttt{ADAM}. For a fair comparison, we ran all algorithms, including MALA, with stochastic gradients in the proposals.}
\label{fig:risks}
\end{figure}

We quantify the performance of the estimators gathered from the different chains with an independent validation sample $\mathcal{D}^\text{val}_{n_\text{val}}\coloneqq(\mathbf{X}_i^\text{val},Y^\text{val}_{i})_{i=1,\dots,n_\text{val}}\subset\R^{\inputdim}\times\R$ of size $n_\text{val} = 10000$ drawn from the same intervals as the training sample and calculate the empirical validation risk
\begin{equation*}
R_{n}(\hat f)=\frac{1}{n_\text{val}}\sum_{i=1}^{n_\text{val}}\big(Y^\text{val}_{i}-\hat f(\mathbf{X}_i^\text{val})\big)^{2}
\end{equation*}
during running of the chain.
\ref{fig:risks} illustrates the behaviour of the empirical validation risk for the different MALA algorithms, as well as for a simple inference fit using \texttt{ADAM} \cite{Adam2014} with a learning rate of $10^{-3}$. 
For a fair comparison, we calculate the gradient updates for all algorithms, including MALA and \texttt{ADAM}, from Bernoulli drawn batches, and only calculate the MH step for MALA using the full training sample.
We can  see, the individual samples of MALA outperform those of the sMALA chains, while the samples from the corrected chain achieve substantially better values than those of the uncorrected stochastic algorithm.
On a level of individual samples, all chains are outperformed by the gradient-based optimization using \texttt{ADAM}.
Investigating the posterior means, MALA outperforms \texttt{ADAM} for small $\rho$ where our corrected algorithm reaches similar risk values as the gradient-based optimization.
For moderate values of $\rho$ the corrected stochastic MALA restores the performance of the full MH step for both, posterior samples and posterior means, at a level similar to \texttt{ADAM}.
While the acceptance rates of MALA decrease for low $\rho$ and those of sMALA increase, the acceptance rates of the corrected algorithm are stable under variation of the average batch size.

To study the empirical coverage properties, we calculate $10$ individual chains per algorithm and $\rho$ and estimate the credible sets and their average radii.
As radius of our credible balls, we approximate the $99.5\%$ quantile $q_{1-\alpha}$ of the mean squared distance to the posterior mean via
\begin{equation*}
    \tau_{\alpha,N} = q_{1-\alpha}\big((h_1,...,h_N)\big) 
    \qquad\text{ with }\qquad
    h_k^2 =\frac{1}{n_\text{val}}\sum_{i=1}^{n_\text{val}} \left| f_{\theta^{(b+ck)}}(\mathbf{X}_i^\text{val}) -\bar{f}_{\lambda,\rho}(\mathbf{X}_i^\text{val})\right|^2.
\end{equation*}
To determine the coverage probability, we then calculate the number of chains with a mean squared distance of the posterior mean to the true regression function not exceeding this radius.
The results are shown in \ref{tab:radii}. 
While the uncertainty estimates of all algorithms remain conservative, we find the correction term leads to considerably more precise credible sets.

To illustrate \ref{thm:oracleinequality} and \ref{thm:oracleinequalityStochMH}, we also investigate the scaling behavior of the empirical validation risk of the posterior means with the training sample size $n$ while keeping $n\rho$ constant.
We expect the risk of MALA to fall with growing $n$, while sMALA should not decay due to the constant $n \rho$.
The numerical simulation of \ref{fig:scaling} coincides with the theoretical expectations.
For our corrected algorithm, we regain the scaling behavior of MALA as expected.

\begin{figure}[t]
\centering
\begin{minipage}{.53\textwidth}
  \centering
  \includegraphics[width=1.\linewidth]{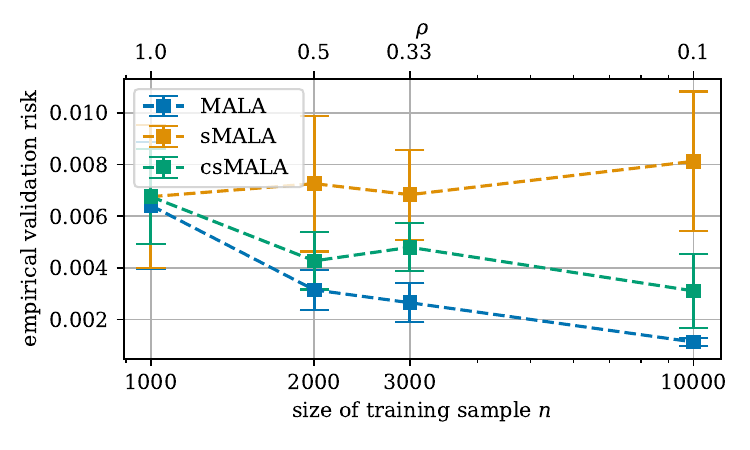}
  \caption{Scaling of the empirical risk of the posterior mean $\bar{f}$ on a $10000$ point validation set with the size of the training sample. We scale $\rho$ to keep the average batch size $n\rho =1000$ constant. Errorbars report the standard deviation of $10$ identical chains.}
  \label{fig:scaling}
\end{minipage}
\hfill
\begin{minipage}{0.45\linewidth}
\vspace{2.2cm}
\begin{tabular}{l|lll}
$\rho$ & MALA               & sMALA              & csMALA             \\ \hline
$0.1$    & $1.42 ± 0.16$ & $13.5 ± 1.4$  & $7.72 ± 0.82$  \\
$0.3$    & $1.10 ± 0.15$ & $3.70 ± 0.51$ & $2.15 ± 0.23$  \\
$0.5$    & $1.28 ± 0.11$ & $2.76 ± 0.19$ & $1.91 ± 0.36$
\end{tabular}
\vspace{.8cm}
\captionof{table}{Average radii $\tau_\alpha \cdot 10^3$ of credible sets for $\alpha = 0.005$ calculated from $10$ Monte Carlo chains. All sets show a coverage probability $\hat{C}(\tau_\alpha)$ of $100\%$.}
\label{tab:radii}
\end{minipage}
\end{figure}

\section{Conclusion}\label{sec:conclusion}
Motivated by MALAs lack of scalability, we considered a stochastic variant, sMALA. It turned out that the corresponding surrogate Gibbs-posterior does not benefit from the full sample size, which is in line with the literature. To remedy this drawback, we introduced a simple correction to sMALA, namely csMALA. Subsequently, we studied its surrogate Gibbs-posterior and verified that it does indeed take advantage of the full sample. Thereby, we refute the conjecture in the literature that a stochastic Metropolis-Hastings step necessarily reduces the effective sample size.

We quantified this phenomenon in terms of the distance to MALAs posterior as well as with oracle inequalities in a nonparametric regression under classical assumptions. Further, we investigated credible sets based on the surrogate Gibbs-posterior which are theoretically valid even without the correction term since MALA is a special case of csMALA. Overall, the quality of the uncertainty quantification depends on the correspondence between the parameters of the network and its output. In a simple linear model, the size of our credible sets depends on the eigenvalues of the design matrix. 
With a focus on Bayesian neural networks, we derived credible sets for shallow neural networks and optimal contraction rates for deep neural networks. 
In our simulation study, we demonstrated that the theoretically desirable properties of the surrogate Gibbs-posterior of csMALA carry over to an estimator directly drawn from csMALA.

\section{Proofs\label{sec:Proofs}}
Before proving our results, we provide an overview of their interplay with respect to the proofs. The proofs of our main results use our PAC-Bayes oracle inequality for the surrogate Gibbs-posterior of csMALA (\ref{thm:oracleinequality}) as a starting point. To prove \ref{thm:oracleinequality}, we first need to ensure compatibility between the corrected risk $\tilde{R}_{n,\rho}$ and the true excess risk, see \ref{prop:compatibility} in \ref{subsec:compatibility}. Together with the Legendre transform of the Kullback-Leibler divergence (\ref{lem:classicallemma}), this leads to a PAC-Bayes bound, see \ref{prop:mainpaclemma} in \ref{subsec:pacbayesbound}. The proof of \ref{thm:oracleinequality} is then completed in \ref{subsec:ProofOracle} by balancing the terms from this PAC-Bayes bound. 
A simplified version of this overall proof strategy leads to the PAC-Bayes oracle inequality for the surrogate Gibbs-posterior of sMALA (\ref{thm:oracleinequalityStochMH}) as we sketch in \ref{subsec:uncorrectedposteriorproof}. 

With \ref{thm:oracleinequality} at hand, we can verify the remaining main results. As previously mentioned, our analogous results to those following \ref{thm:oracleinequality} could also be proved for the original Gibbs-posterior of MALA as well as the surrogate Gibbs-posterior of sMALA, but we focus on the surrogate Gibbs-posterior of csMALA for clarity.
A first application of \ref{thm:oracleinequality} is its extension to the posterior mean, see \ref{cor:mean}. Together with this corollary, we are ready to  prove our first contribution towards uncertainty quantification (\ref{thm:credibility}) in \ref{subsec:proofofcredibility}. We postpone the proof of \ref{cor:mean} as well as the remaining results from \ref{sec:Oracle-inequality} to \ref{sec:remainingProofs}.

In \ref{subsec:proofsNN}, we apply \ref{thm:oracleinequality} and \ref{thm:credibility} to Bayesian neural networks by verifying \ref{assu:F} \labelcref{assu:lipschitz} (\ref{lem:Lipschitz}) and exploiting the approximation properties of ReLU nets. The remaining proofs of auxiliary lemmas used along the way are postponed to \ref{sec:AuxProofs}.

\subsection[Compatibility between the empirical risk and the excess risk]{Compatibility between $\protect\tilde R_{n,\rho}$ and the excess risk}\label{subsec:compatibility}

The first step in our analysis is to verify that the corrected empirical risk
$\tilde{R}_{n,\rho}$ which arises from the stochastic MH step is compatible with the excess risk $\mathcal{E}(\theta)=\E\big[\big(f(\mathbf{X}_{1})-f_{\theta}(\mathbf{X}_{1})\big)^{2}\big]$.
More precisely, we require the following concentration inequality.
A concentration inequality for the empirical risk $R_{n}(\theta)-R_{n}(f)$
follows as the special case where $\rho=1$. 
\begin{prop}
\label{prop:compatibility}Grant \ref{assu:bounded} and \ref{assu:F} \labelcref{assu:bounded_functions}. Define 
\[
\tilde{\mathcal{E}}_{n}(\theta)\coloneqq\tilde R_{n,\rho}(\theta)-\tilde R_{n,\rho}(f),
\]
and set $C_{n,\lambda}\coloneqq\frac{\lambda}{n}\frac{8(\Cf^{2}+\sigma^{2})}{1-w\lambda/n}$,
$w\coloneqq16\Cf(\Ceps\vee2\Cf)$. Then for all $\lambda\in[0,n/w)\cap \big[0, \frac{n\log{2}}{8(\Cf^2+\sigma^2)}\big]$, $\rho\in(0,1]$ and $n\in\N$ we have
\begin{align*}
\E\big[\exp\big(\lambda\big(\tilde{\mathcal{E}}_{n}(\theta)-\mathcal{E}(\theta)\big)\big)\big] & \le\exp\big(\big(C_{n,\lambda}+\tfrac{\lambda}{n}(\sigma \Cf+\sigma^{2})\big)\lambda\mathcal{E}(\theta)\big)\qquad\text{and}\\
\E\big[\exp\big(-\lambda\big(\tilde{\mathcal{E}}_{n}(\theta)-\mathcal{E}(\theta)\big)\big)\big] & \le\exp\big(\big(C_{n,\lambda}+\tfrac{3}{4}+\tfrac{\lambda}{n}(\sigma \Cf+\sigma^{2})\big)\lambda\mathcal{E}(\theta)\big).
\end{align*}
\end{prop}

\begin{proof}
Define $\psi_{\rho}(x)\coloneqq-\log\big(\e^{-x}+1-\rho\big)$ such that
\[
\tilde{\mathcal{E}}_{n}(\theta)=\frac{1}{\lambda}\sum_{i=1}^{n}\big(\psi_{\rho}\big(\tfrac{\lambda}{n}\ell_{i}(\theta)\big)-\psi_{\rho}\big(\tfrac{\lambda}{n}\ell_{i}(f)\big)\big).
\]
We have
\[
\tilde{\mathcal{E}}_{n}(\theta)=\frac{1}{n}\sum_{i=1}^{n}\big(\ell_{i}(\theta)-\ell_{i}(f)\big)\psi_{\rho}'\big(\xi_{i}\tfrac{\lambda}{n}\ell_{i}(\theta)+(1-\xi_{i})\tfrac{\lambda}{n}\ell_{i}(f)\big)
\label{eq:excess_tilde}\]
with some random variables $\xi_{i}\in[0,1]$. Using $\ell_{1}(\theta)-\ell_{1}(f)=\big(f(\mathbf{X}_{1})-f_{\theta}(\mathbf{X}_{1})\big)^{2}+2\eps_{1}\big(f(\mathbf{X}_{1})-f_{\theta}(\mathbf{X}_{1})\big)$,
we can decompose the expectation of \ref{eq:excess_tilde}:
\begin{align*}
\E\big[\tilde{\mathcal{E}}_{n}(\theta)\big]&=  \E\big[\big(f(\mathbf{X}_{1})-f_{\theta}(\mathbf{X}_{1})\big)^{2}\psi_{\rho}'\big(\xi_{1}\tfrac{\lambda}{n}\ell_{1}(\theta)+(1-\xi_{1})\tfrac{\lambda}{n}\ell_{1}(f)\big)\big]\\
 & \qquad\qquad +2\E\big[\eps_{1}\big(f(\mathbf{X}_{1})-f_{\theta}(\mathbf{X}_{1})\big)\psi_{\rho}'\big(\xi_{1}\tfrac{\lambda}{n}\ell_{1}(\theta)+(1-\xi_{1})\tfrac{\lambda}{n}\ell_{1}(f)\big)\big]\\
& \eqqcolon  E_{1}+E_{2}.
\end{align*}
We treat both terms separately. We have 
\begin{align*}
1\ge\psi_{\rho}'(x) & =(1+(1-\rho)\e^{x})^{-1}\\
 & \ge\frac{1}{1+2(1-\rho)}\ge\frac{1}{3}\qquad\text{for }x\in[0,\log2]
\end{align*}
and $\psi_{\rho}'(x)\in(0,1]$ for all $x\ge0$. In particular,
we observe
\[
E_{1}\le\E\big[\big(f_{\theta}(\mathbf{X}_{1})-f(\mathbf{X}_{1})\big)^{2}\big]=\mathcal{E}(\theta).
\]
 If $\vert \eps_{1}\vert\le2\sigma$, we have $\frac{\lambda}{n}\ell_{1}(\cdot)\le\frac{\lambda}{n}8(\Cf^{2}+\sigma^{2})\le\log2$
for $\frac{\lambda}{n}\le\frac{\log2}{8(\Cf^{2}+\sigma^{2})}$. Hence,
\begin{align*}
E_{1} & \ge\E\big[\big(f(\mathbf{X}_{1})-f_{\theta}(\mathbf{X}_{1})\big)^{2}\psi_{\rho}'\big(\xi_{1}\tfrac{\lambda}{n}\ell_{1}(\theta)+(1-\xi_{1})\tfrac{\lambda}{n}\ell_{1}(f)\big)\1_{\{\vert\eps_{1}\vert\le2\sigma\}}\big]\\
 & \ge\frac{1}{3}\E\big[\big(f(\mathbf{X}_{1})-f_{\theta}(\mathbf{X}_{1})\big)^{2}\P(\vert\eps_{1}\vert\le2\sigma\mid \mathbf{X}_{1})\big]\\
 & =\frac{1}{3}\E\big[\big(f(\mathbf{X}_{1})-f_{\theta}(\mathbf{X}_{1})\big)^{2}\big(1-\P(\vert\eps_{1}\vert>2\sigma\mid\mathbf{X}_{1})\big)\big]\\
 & \ge\frac{1}{4}\E\big[\big(f(\mathbf{X}_{1})-f_{\theta}(\mathbf{X}_{1})\big)^{2}\big]
\end{align*}
where we used Chebyshev's inequality in the last estimate. Hence,
$\frac{1}{4}\mathcal{E}(\theta)\le E_{1}\le\mathcal{E}(\theta).$
For $E_{2}$ we use $\E[\eps_{1}\psi'_{\rho}(\frac{\lambda}{n}\eps_{1}^{2})\mid\mathbf{X}_{1}]=0$
by symmetry together with $\ell_{1}(f)=\eps_{1}^{2}$ to obtain for
some random $\xi_{1}'\in[0,1]$
\begin{align*}
E_{2} & =2\E\big[\eps_{1}\big((f(\mathbf{X}_{1})-f_{\theta}(\mathbf{X}_{1})\big)\big(\psi_{\rho}'\big(\tfrac{\lambda}{n}\ell_{1}(f)+\xi_{1}\tfrac{\lambda}{n}\big(\ell_{1}(\theta)-\ell_{1}(f)\big)\big)-\psi_{\rho}'\big(\tfrac{\lambda}{n}\ell_{1}(f)\big)\big)\big]\\
 & =\frac{2\lambda}{n}\E\big[\eps_{1}\big(f(\mathbf{X}_{1})-f_{\theta}(\mathbf{X}_{1})\big)\xi_{1}\big(\ell_{1}(\theta)-\ell_{1}(f)\big)\psi_{\rho}''\big(\xi_{1}'\tfrac{\lambda}{n}\ell_{1}(\theta)+(1-\xi_{1}')\tfrac{\lambda}{n}\ell_{1}(f)\big)\big]\\
 & =\frac{\lambda}{n}\E\big[2\xi_{1}\big(\eps_{1}\big(f(\mathbf{X}_{1})-f_{\theta}(\mathbf{X}_{1})\big)^{3}+2\eps_{1}^{2}\big(f(\mathbf{X}_{1})-f_{\theta}(\mathbf{X}_{1})\big)^{2}\big)\psi_{\rho}''\big(\xi_{1}'\tfrac{\lambda}{n}\ell_{1}(\theta)+(1-\xi_{1}')\tfrac{\lambda}{n}\ell_{1}(f)\big)\big].
\end{align*}
Since $\max_{y\ge0}\frac{y}{(1+y)^{2}}=\frac{1}{4}$,
we have 
\[
\vert\psi_{\rho}''(x)\vert=\frac{(1-\rho)\e^{x}}{(1+(1-\rho)\e^{x})^{2}}\le\frac{1}{4}\qquad\text{for }x\ge0.
\]
Therefore, 
\begin{align*}
\vert E_{2}\vert & \le\frac{\lambda}{n}\big(\tfrac{1}{2}\E\big[\vert\eps_{1}\vert\vert f_{\theta}(\mathbf{X}_{1})-f(\mathbf{X}_{1})\vert^{3}+2\eps_{1}^{2}\big(f(\mathbf{X}_{1})-f_{\theta}(\mathbf{X}_{1})\big)^{2}\big]\big)\\
 & \le\frac{\lambda}{n}\big(\sigma \Cf+\sigma^{2}\big)\mathcal{E}(\theta).
\end{align*}
In combination with the bounds for $E_{1}$ we obtain
\[
\big(\tfrac{1}{4}-\tfrac{\lambda}{n}(\sigma \Cf+\sigma^{2})\big)\mathcal{E}(\theta)\le\E\big[\tilde{\mathcal{E}}_{n}(\theta)\big]\le\big(1+\tfrac{\lambda}{n}(\sigma \Cf+\sigma^{2})\big)\mathcal{E}(\theta).
\]
Define $Z_{i}(\theta)\coloneqq\frac{n}{\lambda}\big(\psi_{\rho}\big(\frac{\lambda}{n}\ell_{i}(\theta)\big)-\psi_{\rho}\big(\frac{\lambda}{n}\ell_{i}(f)\big)\big)$
such that $\tilde{\mathcal{E}}_{n}(\theta)=\frac{1}{n}\sum_{i=1}^{n}Z_{i}(\theta)$.
The previous bounds for $\E[\tilde{\mathcal{E}}_{n}(\theta)]$ yield
\begin{align*}
\E\big[\exp\big(\lambda\tilde{\mathcal{E}}_{n}(\theta)-\lambda\mathcal{E}(\theta)\big)\big] & =\E\big[\e^{\frac{\lambda}{n}\sum_{i=1}^n(Z_{i}(\theta)-\E[Z_{i}(\theta)])}\big]\e^{\lambda(\E[\tilde{\mathcal{E}}_{n}(\theta)]-\mathcal{E}(\theta))}\\
 & \le\E\big[\e^{\frac{\lambda}{n}\sum_{i=1}^n(Z_{i}(\theta)-\E[Z_{i}(\theta)])}\big]\e^{\frac{\lambda^{2}}{n}(\sigma \Cf+\sigma^{2})\mathcal{E}(\theta)}
\end{align*}
and
\begin{align*}
\E\big[\exp\big(-\lambda\tilde{\mathcal{E}}_{n}(\theta)+\lambda\mathcal{E}(\theta)\big)\big] & =\E\big[\e^{\frac{\lambda}{n}\sum_{i=1}^n(-Z_{i}(\theta)-\E[-Z_{i}(\theta)])}\big]\e^{\lambda(\mathcal{E}(\theta)-\E[\tilde{\mathcal{E}}_{n}(\theta)])}\\
 & \le\E\big[\e^{\frac{\lambda}{n}\sum_{i=1}^n(-Z_{i}(\theta)-\E[-Z_{i}(\theta)])}\big]\e^{(\frac{3\lambda}{4}+\frac{\lambda^{2}}{n}(\sigma \Cf+\sigma^{2}))\mathcal{E}(\theta)}.
\end{align*}
To bound the centered exponential moments, we use Bernstein's inequality. The required bounds for the regular moments are obtained as in \cite{alquier2013,Guedj2013}:
The second moments are bounded by
\begin{align*}
\E[Z_{i}^{2}]&= \E\big[\big(\tfrac{n}{\lambda}\big(\psi_{\rho}\big(\tfrac{\lambda}{n}\ell_{1}(\theta)\big)-\psi_{\rho}\big(\tfrac{\lambda}{n}\ell_{1}(f)\big)\big)\big)^{2}\big]\\
&=  \E\big[\big(\big(\ell_{1}(\theta)-\ell_{1}(f)\big)\psi_{\rho}'\big(\xi_{1}\tfrac{\lambda}{n}\ell_{1}(\theta)+(1-\xi_{1})\tfrac{\lambda}{n}\ell_{1}(f)\big)\big)^{2}\big]\\
&=  \E\big[\big((f_{\theta}(\mathbf{X}_{1})-f(\mathbf{X}_{1}))^{2}+2\eps_{1}(f_{\theta}(\mathbf{X}_{1})-f(\mathbf{X}_{1}))\big)^{2}(\psi_{\rho}')^{2}\big(\xi_{1}\tfrac{\lambda}{n}\ell_{1}(\theta)+(1-\xi_{1})\tfrac{\lambda}{n}\ell_{1}(f)\big)\big]\\
&\le  2\E\big[\big(f_{\theta}(\mathbf{X}_{1})-f(\mathbf{X}_{1})\big)^{4}+4\eps_{1}^{2}\big(f_{\theta}(\mathbf{X}_{1})-f(\mathbf{X}_{1})\big)^{2}\big]\\
&\le  8\big(\Cf^{2}+\sigma^{2}\big)\mathcal{E}(\theta)\eqqcolon U.
\end{align*}
Moreover, we have for $k\ge3$
\begin{align*}
\E\big[(Z_{i})_{+}^{k}\big] & \le\E\big[\big\vert\ell_{1}(\theta)-\ell_{1}(f)\big\vert^{k}\big\vert\psi_{\rho}'\big(\xi_{1}\tfrac{\lambda}{n}\ell_{1}(\theta)+(1-\xi_{1})\tfrac{\lambda}{n}\ell_{1}(f)\big)\big\vert^{k}\big]\\
 & \le\E\big[\big\vert\ell_{1}(\theta)-\ell_{1}(f)\big\vert^{k}\big]\\
 & =\E[\vert f(\mathbf{X}_{1})-f_{\theta}(\mathbf{X}_{1})+2\eps_{1}\vert^{k}\vert f(\mathbf{X}_{1})-f_{\theta}(\mathbf{X}_{1})\vert^{k-2}(f(\mathbf{X}_{1})-f_{\theta}(\mathbf{X}_{1}))^{2}]\\
 & \le(2\Cf)^{k-2}\E[\vert f(\mathbf{X}_{1})-f_{\theta}(\mathbf{X}_{1})+2\eps_{1}\vert^{k}(f(\mathbf{X}_1)-f_{\theta}(\mathbf{X}_1))^{2}]\\
 & \le(2\Cf)^{k-2}2^{k-1}((2\Cf)^{k}+k!2^{k-1}\sigma^{2}\Ceps^{k-2})\mathcal{E}(\theta)\\
 & \le(2\Cf)^{k-2}k!8^{k-2}\big((2\Cf)^{k-2}\vee\Ceps^{k-2}\big)U\\
 & =k!Uw^{k-2}.
\end{align*}
Hence, Bernstein's inequality \citep[inequality (2.21)]{Massart2007}
yields 
\[
\E\big[\e^{\frac{\lambda}{n}\sum_{i=1}^n(Z_{i}(\theta)-\E[Z_{i}(\theta)])}\big]\le\exp\Big(\frac{U\lambda^{2}}{n(1-w\lambda/n)}\Big)=\exp\big(C_{n,\lambda}\lambda\mathcal{E}(\theta)\big)
\]
for $C_{n,\lambda}$ as defined in \ref{prop:compatibility}. The
same bound remains true if we replace $Z_{i}$ by $-Z_{i}$. We conclude 
\[
\E\big[\exp\big(\lambda\tilde{\mathcal{E}}_{n}(\theta)-\lambda\mathcal{E}(\theta)\big)\big]\le\exp\big(\big(C_{n,\lambda}+\tfrac{\lambda}{n}(\sigma \Cf+\sigma^{2})\big)\lambda\mathcal{E}(\theta)\big)
\]
and
\[
\E\big[\exp\big(-\lambda\tilde{\mathcal{E}}_{n}(\theta)+\lambda\mathcal{E}(\theta)\big)\big]\le\exp\big(\big(C_{n,\lambda}+\tfrac{3}{4}+\tfrac{\lambda}{n}(\sigma \Cf+\sigma^{2})\big)\lambda\mathcal{E}(\theta)\big).\tag*{{\qedhere}}
\]
\end{proof}
\begin{rem}
\label{rem:stochMH}Replacing $\psi_{\rho}$ by $\bar{\psi}_{\rho}(x)\coloneqq-\log\big(\rho \e^{-x/\rho}+1-\rho\big)$,
$x\ge0$, and using 
\begin{align*}
1\ge\bar{\psi}_{\rho}'(x) & =(\rho+(1-\rho)\e^{x/\rho})^{-1}\\
 & \ge\frac{1}{\rho+3(1-\rho)}\ge\frac{1}{3}\qquad\text{for }x\in[0,\rho\log 3],
\end{align*}
we can analogously prove under \ref{assu:bounded} that $\bar{\mathcal{E}}_{n}(\theta)\coloneqq\bar{R}_{n,\rho}(\theta)-\bar{R}_{n,\rho}(f)$
with $\bar{R}_{n,\rho}$ from \ref{eq:rBar} satisfies for all $\lambda\in[0,n/w)\cap\big[0,\frac{n\log 3}{8(\Cf^2+\sigma^2)}\big],\rho\in(0,1]$
and $n\in\N$:
\begin{align*}
\E\big[\exp\big(\lambda\big(\bar{\mathcal{E}}_{n}(\theta)-\mathcal{E}(\theta)\big)\big)\big] & \le\exp\big(\big(C_{n,\lambda}+\tfrac{\lambda}{n\rho}4(\sigma \Cf+\sigma^{2})\big)\lambda\mathcal{E}(\theta)\big)\qquad\text{and}\\
\E\big[\exp\big(-\lambda\big(\bar{\mathcal{E}}_{n}(\theta)-\mathcal{E}(\theta)\big)\big)\big] & \le\exp\big(\big(C_{n,\lambda}+\tfrac{1}{4}+\tfrac{\lambda}{n\rho}4(\sigma \Cf+\sigma^{2})\big)\lambda\mathcal{E}(\theta)\big).
\end{align*}
\end{rem}

\subsection{A PAC-Bayes bound}\label{subsec:pacbayesbound}

Let $\mu,\nu$ be probability measures on a measurable space $(E,\mathscr{A})$.
The \emph{Kullback-Leibler divergence} of $\mu$ with respect to $\nu$
is defined via 
\begin{equation}
\KL(\mu\mid\nu)\coloneqq\begin{cases}
\int\log\big(\diff{\mu}{\nu}\big)\,\d\mu, & \text{if }\mu\ll\nu\\
\infty, & \text{otherwise}
\end{cases}.\label{eq:KL}
\end{equation}
The aforementioned Legendre transform of the Kullback-Leibler divergence is a key ingredient for PAC-Bayes bounds,
cf. \cite{csiszar1975}, \cite{donsker1976}, \cite{catoni2004,catoni2007}. We include
the short proof for the sake of completeness. 
\begin{lem}
\label{lem:classicallemma}Let $h\colon E\to\R$ be a measurable function
such that $\int\exp\circ h\,\d\mu<\infty$. With the convention $\infty-\infty=-\infty$
it then holds that 
\begin{equation}
\log\Big(\int\exp\circ h\,\d\mu\Big)=\sup_{\nu}\Big(\int h\,\d\nu-\KL(\nu\mid\mu)\Big),\label{eq:gibbsequality}
\end{equation}
where the supremum is taken over all probability
measures $\nu$ on $(E,\mathscr{A})$. If additionally, $h$ is bounded
from above on the support of $\mu$, then the supremum in \ref{eq:gibbsequality}
is attained for $\nu=g$ with the Gibbs distribution $g$, \emph{i.e.},  $\diff g{\mu}:\propto\exp\circ h$.
\end{lem}

\begin{proof}
For $D\coloneqq\int \e^{h}\,\d\mu$, we have $\d g=D^{-1}\e^{h}\d\mu$ and obtain
for all $\nu\ll\mu$:
\begin{align*}
0\le\KL(\nu\mid g)=\int\log\frac{\d\nu}{\d g}\,\d\nu & =\int\log\frac{\d\nu}{\e^{h}\d\mu/D}\,\d\nu\\
 & =\KL(\nu\mid\mu)-\int h\,\d\nu+\log\Big(\int \e^{h}\,\d\mu\Big).\tag*{{\qedhere}}
\end{align*}
\end{proof}
Note that no generality is lost by considering only those probability
measures $\nu$ on $(E,\mathscr{A})$ such that $\nu\ll\mu$ and thus
\[
\log\Big(\int\exp\circ h\,\d\mu\Big)=-\inf_{\nu\ll\mu}\Big(\KL(\nu\mid\mu)-\int h\,\d\nu\Big).
\]
In combination with \ref{prop:compatibility} we can verify a PAC-Bayes
bound for the excess risk. The basic proof strategy is in line with
the PAC-Bayes literature, see e.g.\ \citet{alquier2013}.
\begin{prop}[PAC-Bayes bound]
\label{prop:mainpaclemma} Grant \ref{assu:bounded} and  \ref{assu:F} \labelcref{assu:bounded_functions}. For any sample-dependent
(in a measurable way) probability measure $\varrho\ll\Pi$ and any $\lambda\in(0,n/w)$
and $\rho\in(0,1]$ such that $C_{n,\lambda}+\frac{\lambda}{n}(\sigma \Cf+\sigma^{2})\le\frac{1}{8}$,
we have
\begin{equation}
\mathcal{E}(\tilde{\theta}_{\lambda,\rho})\le9\int\mathcal{E}\,\d\varrho+\frac{16}{\lambda}\big(\KL(\varrho\mid\Pi)+\log(2/\delta)\big)\label{eq:mainpacbound}
\end{equation}
with probability
 at least $1-\delta$. 
\end{prop}

\begin{proof}
\ref{prop:compatibility} yields
\begin{align*}
\E\big[\exp\big(\lambda\tilde{\mathcal{E}}_{n}(\theta)-\big(1+C_{n,\lambda}+\tfrac{\lambda}{n}(\sigma \Cf+\sigma^{2})\big)\lambda\mathcal{E}(\theta)-\log\delta^{-1}\big)\big] & \le\delta\qquad\text{and}\\
\E\big[\exp\big(\lambda\big(\tfrac{1}{4}-C_{n,\lambda}-\tfrac{\lambda}{n}(\sigma \Cf+\sigma^{2})\big)\mathcal{E}(\theta)-\lambda\tilde{\mathcal{E}}_{n}(\theta)-\log\delta^{-1}\big)\big] & \le\delta.
\end{align*}
Integrating in $\theta$ with respect to the prior probability measure
$\Pi$ and applying Fubini's theorem, we conclude
\begin{align}
\E\Big[\int\exp\big(\lambda\tilde{\mathcal{E}}_{n}(\theta)-\big(1+C_{n,\lambda}+\tfrac{\lambda}{n}(\sigma \Cf+\sigma^{2})\big)\lambda\mathcal{E}(\theta)-\log\delta^{-1}\big)\,\d\Pi(\theta)\Big] & \le\delta\qquad\text{and}\label{eq:BernsteinConsequence1} \\
\E\Big[\int\exp\big(\lambda\big(\tfrac{1}{4}-C_{n,\lambda}-\tfrac{\lambda}{n}(\sigma \Cf+\sigma^{2})\big)\mathcal{E}(\theta)-\lambda\tilde{\mathcal{E}}_{n}(\theta)-\log\delta^{-1}\big)\,\d\Pi(\theta)\Big] & \le\delta.\label{eq:BersteinConsequence}
\end{align}
The Radon-Nikodym density of the posterior distribution $\tilde{\Pi}_{\lambda,\rho}(\cdot\mid\mathcal{D}_{n})\ll\Pi$
with respect to $\Pi$ is given by 
\begin{equation}
\frac{\d\tilde{\Pi}_{\lambda,\rho}(\theta\mid\mathcal{D}_{n})}{\d\Pi}=\tilde D_{\lambda}^{-1}\exp\Big(-\sum_{i=1}^{n}\psi_{\rho}\big(\tfrac{\lambda}{n}\ell_{i}(\theta)\big)\Big)\label{eq:posteriorDensity}
\end{equation}
with 
\[
\tilde D_{\lambda}\coloneqq\int \e^{-\lambda\tilde R_{n,\rho}(\theta)}\,\Pi(\d\theta)=\int\exp\Big(-\sum_{i=1}^{n}\psi_{\rho}\big(\tfrac{\lambda}{n}\ell_{i}(\theta)\big)\Big)\,\Pi(\d\theta).\label{eq:D_tilde}
\]
We obtain
\begin{align*}
\delta & \ge\E_{\mathcal{D}_{n}}\Big[\int\exp\big(\lambda\big(\tfrac{1}{4}-C_{n,\lambda}-\tfrac{\lambda}{n}(\sigma \Cf+\sigma^{2})\big)\mathcal{E}(\theta)-\lambda\tilde{\mathcal{E}}_{n}(\theta)-\log\delta^{-1}\big)\,\d\Pi(\theta)\Big]\\
 & =\E_{\mathcal{D}_{n},\tilde{\theta}\sim\tilde{\Pi}_{\lambda,\rho}(\cdot\mid\mathcal{D}_{n})}\Big[\exp\Big(\lambda\big(\tfrac{1}{4}-C_{n,\lambda}-\tfrac{\lambda}{n}(\sigma \Cf+\sigma^{2})\big)\mathcal{E}(\tilde{\theta})-\lambda\tilde{\mathcal{E}}_{n}(\tilde{\theta})\\
 & \qquad\qquad\qquad\qquad\qquad\qquad\qquad\qquad-\log\delta^{-1}-\log\Big(\frac{\d\tilde{\Pi}_{\lambda,\rho}(\tilde{\theta}\mid\mathcal{D}_{n})}{\d\Pi}\Big)\Big)\Big]\\
 & =\E_{\mathcal{D}_{n},\tilde{\theta}\sim\tilde{\Pi}_{\lambda,\rho}(\cdot \mid\mathcal{D}_{n})}\Big[\exp\Big(\lambda\big(\tfrac{1}{4}-C_{n,\lambda}-\tfrac{\lambda}{n}(\sigma \Cf+\sigma^{2})\big)\mathcal{E}(\tilde{\theta})-\lambda\tilde{\mathcal{E}}_{n}(\tilde{\theta})\\
 & \qquad\qquad\qquad\qquad\qquad\qquad\qquad\qquad-\log\delta^{-1}+\sum_{i=1}^{n}\psi_{\rho}\big(\tfrac{\lambda}{n}\ell_{i}(\tilde{\theta})\big)+\log\tilde D_{\lambda}\Big)\Big].
\end{align*}
Since $\1_{[0,\infty)}(x)\le \e^{\lambda x}$ for all $x\in\R$, we
deduce with probability not larger than $\delta$ that
\[
\big(\tfrac{1}{4}-C_{n,\lambda}-\tfrac{\lambda}{n}(\sigma \Cf+\sigma^{2})\big)\mathcal{E}(\tilde{\theta})-\tilde{\mathcal{E}}_{n}(\tilde{\theta})+\frac{1}{\lambda}\sum_{i=1}^{n}\psi_{\rho}\big(\tfrac{\lambda}{n}\ell_{i}(\tilde{\theta})\big)-\frac{1}{\lambda}\big(\log\delta^{-1}-\log\tilde D_{\lambda}\big)\ge0.
\]
Provided $C_{n,\lambda}+\frac{\lambda}{n}(\sigma \Cf+\sigma^{2})\le\frac{1}{8}$,
we thus have for  $\tilde{\theta}\sim \tilde{\Pi}_{\lambda,\rho}(\cdot\mid\mathcal{D}_{n})$ with probability
 at least $1-\delta$:
\begin{align*}
\mathcal{E}(\tilde{\theta}) & \le8\Big(\tilde{\mathcal{E}}_{n}(\tilde{\theta})-\frac{1}{\lambda}\sum_{i=1}^{n}\psi_{\rho}\big(\tfrac{\lambda}{n}\ell_{i}(\tilde{\theta})\big)+\frac{1}{\lambda}\big(\log\delta^{-1}-\log\tilde D_{\lambda}\big)\Big)\\
 & \le8\Big(-\frac{1}{\lambda}\sum_{i=1}^{n}\psi_{\rho}\big(\tfrac{\lambda}{n}\ell_{i}(f)\big)+\frac{1}{\lambda}\big(\log\delta^{-1}-\log\tilde D_{\lambda}\big)\Big)
\end{align*}
\ref{lem:classicallemma} with $h=-\sum_{i=1}^{n}\psi_{\rho}(\frac{\lambda}{n}\ell_{i}(\theta))$
yields
\begin{equation}
\log\tilde D_{\lambda}=\log\Big(\int\exp\Big(-\sum_{i=1}^{n}\psi_{\rho}\big(\tfrac{\lambda}{n}\ell_{i}(\theta)\big)\Big)\,\d\Pi(\theta)\Big)=-\!\inf_{\varrho\ll\Pi}\Big(\KL(\varrho\mid\Pi)+\int\sum_{i=1}^{n}\psi_{\rho}\big(\tfrac{\lambda}{n}\ell_{i}(\theta)\big)\,\d\varrho(\theta)\Big).\label{eq:constant}
\end{equation}
Therefore, we have with probability
 at least $1-\delta$:
\begin{align*}
\mathcal{E}(\tilde{\theta}) & \le8\inf_{\varrho\ll\Pi}\Big(\int\frac{1}{\lambda}\sum_{i=1}^{n}\big(\psi_{\rho}\big(\tfrac{\lambda}{n}\ell_{i}(\theta)\big)-\psi_{\rho}\big(\tfrac{\lambda}{n}\ell_{i}(f)\big)\big)\,\d\varrho(\theta)+\frac{1}{\lambda}\big(\log\delta^{-1}+\KL(\varrho\mid\Pi)\big)\Big)\\
 & \le8\inf_{\varrho\ll\Pi}\Big(\int\tilde{\mathcal{E}}_{n}(\theta)\,\d\varrho(\theta)+\frac{1}{\lambda}\big(\log\delta^{-1}+\KL(\varrho\mid\Pi)\big)\Big).
\end{align*}
In order to reduce the integral $\int\tilde{\mathcal{E}}_{n}(\theta)\,\d\varrho(\theta)$
to $\int\mathcal{E}(\theta)\,\d\varrho(\theta)$, we use $C_{n,\lambda}+\frac{\lambda}{n}(\sigma \Cf+\sigma^{2})\le\frac{1}{8}$,
Jensen's inequality and \ref{eq:BernsteinConsequence1} to obtain for
any probability measure $\varrho\ll\Pi$ (which may depend on $\mathcal{D}_{n}$)
\begin{align*}
\E_{\mathcal{D}_{n}}\Big[\exp\Big( & \int\big(\lambda\tilde{\mathcal{E}}_{n}(\theta)-\tfrac{9}{8}\lambda\mathcal{E}(\theta)\big)\,\d\varrho(\theta)-\KL(\varrho\mid\Pi)-\log\delta^{-1}\Big)\Big]\\
&=\E_{\mathcal{D}_{n}}\Big[\exp\Big(\int\lambda\tilde{\mathcal{E}}_{n}(\theta)-\tfrac{9}{8}\lambda\mathcal{E}(\theta)-\log\Big(\frac{\d\varrho}{\d\Pi}(\theta)\Big)-\log\delta^{-1}\,\d\varrho(\theta)\Big)\Big]\\
&\le\E_{\mathcal{D}_{n},\theta\sim\varrho}\Big[\exp\Big(\lambda\tilde{\mathcal{E}}_{n}(\theta)-\tfrac{9}{8}\lambda\mathcal{E}(\theta)-\log\Big(\frac{\d\varrho}{\d\Pi}(\theta)\Big)-\log\delta^{-1}\Big)\Big]\\
&\le \E_{\mathcal{D}_{n}}\Big[\int\exp\big(\lambda\tilde{\mathcal{E}}_{n}(\theta)-\big(1+C_{n,\lambda}+\tfrac{\lambda}{n}(\sigma \Cf+\sigma^{2})\big)\lambda\mathcal{E}(\theta)-\log\delta^{-1}\big)\,\d\Pi(\theta)\Big]\le\delta.
\end{align*}
Using $\1_{[0,\infty)}(x)\le \e^{\lambda x}$ again, we conclude with
probability at least $1-\delta$:
\[
\int\tilde{\mathcal{E}}_{n}(\theta)\,\d\varrho(\theta)\le\frac{9}{8}\int\mathcal{E}(\theta)\,\d\varrho(\theta)+\lambda^{-1}\big(\KL(\varrho\mid\Pi)+\log\delta^{-1}\big).
\]
Therefore, we conclude with probability  at least $1-2\delta$
\[
\mathcal{E}(\tilde{\theta})\le9\int\mathcal{E}(\theta)\,\d\varrho(\theta)+\frac{16}{\lambda}\big(\KL(\varrho\mid\Pi)+\log\delta^{-1}\big).\tag*{{\qedhere}}
\]
\end{proof}

\subsection{Proof of \ref{thm:oracleinequality}\label{subsec:ProofOracle}}

We fix a radius $\eta\in(0,1]$ and apply \ref{prop:mainpaclemma} with
$\varrho=\varrho_{\eta}$ defined via
\[
\diff{\varrho_{\eta}}{\Pi}(\theta)\propto\1_{\{\vert\theta-\theta^{\ast}\vert_{\Theta}\le\eta\}}
\]
with $\theta^{*}$ from \ref{eq:oracle}. Note that indeed $C_{n,\lambda}+\frac{\lambda}{n}(\sigma \Cf+\sigma^{2})\le\frac{1}{8}$
for $n$ sufficiently large. In order to control the integral
term, we decompose
\begin{align}
\int\mathcal{E}\,\mathrm{d}\varrho_\eta & =\mathcal{E}(\theta^{\ast})+\int\E\big[(f_{\theta}(\mathbf{X})-f(\mathbf{X}))^{2}-(f_{\theta^{*}}(\mathbf{X})-f(\mathbf{X}))^{2}\big]\,\d\varrho_\eta(\theta)\nonumber \\
 & =\mathcal{E}(\theta^{\ast})+\int\E\big[(f_{\theta^{\ast}}(\mathbf{X})-f_{\theta}(\mathbf{X}))^{2}\big]\,\mathrm{d}\varrho_\eta(\theta)+2\int\E\big[(f(\mathbf{X})-f_{\theta^{\ast}}(\mathbf{X}))(f_{\theta^{\ast}}(\mathbf{X})-f_{\theta}(\mathbf{X}))\big]\,\mathrm{d}\varrho_\eta(\theta)\nonumber \\
 & \le\mathcal{E}(\theta^{\ast})+\int\E\big[(f_{\theta^{\ast}}(\mathbf{X})-f_{\theta}(\mathbf{X}))^{2}\big]\,\mathrm{d}\varrho_\eta(\theta)\nonumber \\
 & \qquad+2\int\E\big[(f(\mathbf{X})-f_{\theta^{\ast}}(\mathbf{X}))^{2}\big]^{1/2}\E\big[(f_{\theta^{\ast}}(\mathbf{X})-f_{\theta}(\mathbf{X}))^{2}\big]^{1/2}\,\mathrm{d}\varrho_\eta(\theta)\nonumber \\
 & \le\frac{4}{3}\mathcal{E}(\theta^{\ast})+4\int\E\big[(f_{\theta^{\ast}}(\mathbf{X})-f_{\theta}(\mathbf{X}))^{2}\big]\,\mathrm{d}\varrho_\eta(\theta),\label{eq:intErho}
\end{align}
using $2ab\le\frac{a^{2}}{3}+3b^{2}$ in the last step. To bound the
remainder, we use the Lipschitz continuity of the map $\theta\mapsto f_\theta(\mathbf{x})$ from \ref{assu:F} \labelcref{assu:lipschitz}.
We obtain 
    \begin{equation}
    \int\mathcal{E}\,\d\varrho_\eta\le\frac{4}{3}\mathcal{E}(\theta^{\ast})+\frac{1}{n^{2}}\qquad\text{for}\qquad\eta=\frac{1}{2\Delta \,n}.\label{eq:integralterm}
    \end{equation}
    It remains to bound the Kullback-Leibler term in \ref{eq:mainpacbound}
which can be done with the following lemma: 
\begin{lem}
\label{lem:aux_a} Suppose $\mathcal B_\eta(\theta^*):=\{\theta\in\R^\pardim:|\theta-\theta^*|_\Theta\le\eta\}\subset\Theta$. Then the probability measure $\diff{\varrho_{\eta}}{\Pi}(\theta)\propto\1_{\mathcal B_\eta(\theta^*)}$ satisfies
\begin{equation}
    \KL(\varrho_{\eta}\mid\Pi)= \pardim\log(2B/\eta)-\log\operatorname{vol}(\mathcal B_1),
\end{equation}
where $\operatorname{vol}(\mathcal B_1)$ denotes the volume of the $|\cdot|_\Theta$-unit ball $\mathcal B_1(0)$.
\end{lem}
\noindent Indeed, $\mathcal{B}_\eta(\theta^\ast)\subseteq \Theta$ for sufficiently large $n$, so plugging \ref{eq:integralterm} and the bound from \ref{lem:aux_a}
into the PAC-Bayes bound \ref{eq:mainpacbound}, we conclude
\begin{align}
\mathcal{E}(\tilde{\theta}_{\lambda,\rho}) & \le12\mathcal{E}(\theta^{\ast})+\frac{9}{n^{2}}+\frac{16}{\lambda}\big( \pardim\log(4B\Delta n)-\log\operatorname{vol}(\mathcal B_1)+\log(2/\delta)\big)\nonumber \\
 & \le12\mathcal{E}(\theta^{\ast})+\frac{\K{1}}{n}\big( \pardim \log(B\Delta n)-\log\operatorname{vol}(\mathcal B_1)+\log(2/\delta)\big).\label{eq:preoracle}
\end{align}
for some constant $\K{1}$ only depending on $\Cf,\Ceps,\sigma$.\hfill\qed

\subsection{Proof of \ref{thm:oracleinequalityStochMH}}\label{subsec:uncorrectedposteriorproof}

Due to \ref{rem:stochMH} we can prove analogously to \ref{prop:mainpaclemma}
the following PAC-Bayes bound under \ref{assu:bounded}: For any sample-dependent
(in a measurable way) probability measure $\varrho\ll\Pi$ and any $\lambda\in(0,n/w)$
and $\rho\in(0,1]$ such that $C_{n,\lambda}+\frac{\lambda}{n\rho}4(\sigma \Cf+\sigma^{2})\le\frac{1}{4}$,
we have
\begin{equation}
\mathcal{E}(\widehat{\theta}_{\lambda})\le\frac{5}{2}\int\mathcal{E}\,\d\varrho+\frac{4}{\lambda}\big(\KL(\varrho\mid\Pi)+\log(2/\delta)\big)\label{eq:mainpacbound-2}
\end{equation}
with probability
 at least $1-\delta$. From here we can continue as in \ref{subsec:ProofOracle}.\hfill\qed

\subsection{Proof of \ref{thm:credibleSet}}\label{subsec:proofofcredibility}

Choosing $\lambda=\frac{n}{\qnorm\K{0}{}}$, \ref{thm:oracleinequality} and \ref{cor:mean} yield 
\[\min\big\{\E[\tilde{\Pi}_{\lambda,\rho}(\theta:\Vert f_{\theta}-f\Vert_{L^{2}(\P^{\mathbf{X}})}\le s_{n}\mid\mathcal{D}_{n})],\P(\Vert f-\bar{f}_{\lambda,\rho}\Vert_{L^{2}(\P^{\mathbf{X}})}\le s_{n})\big\}\ge1-\frac{\alpha^{2}}{2}\]
with $s_{n}^{2}\coloneqq \qnorm\rate^{2}+\frac{2\qnorm(\K{1}\vee\K{2})}{n}\log\frac{2}{\alpha}$. We conclude
\begin{align*}
\P\big(\diam(\widehat{C}(\tau_{\alpha}))\le4s{}_{n}\big) & =\P\Big(\sup_{g,h\in\widehat{C}(\tau_{\alpha})}\Vert g-h\Vert_{L^{2}(\P^{\mathbf{X}})}\le4s{}_{n}\Big)\\
 & \ge\P\Big(\sup_{g,h\in\widehat{C}(\tau_{\alpha})}\Vert g-\bar{f}_{\lambda,\rho}\Vert_{L^{2}(\P^{\mathbf{X}})}+\Vert\bar{f}_{\lambda,\rho}-h\Vert_{L^{2}(\P^{\mathbf{X}})}\le4s{}_{n}\Big)\\
 & \ge\P\big(\tau_{\alpha}\le2s{}_{n}\big)\\
 & =\P\big(\tilde{\Pi}_{\lambda,\rho}(\theta:\Vert f_{\theta}-\bar{f}_{\lambda,\rho}\Vert_{L^{2}(\P^{\mathbf{X}})}\le2s{}_{n}\mid\mathcal{D}_{n})>1-\alpha\big)\\
 & \ge\P\big(\tilde{\Pi}_{\lambda,\rho}(\theta:\Vert f_{\theta}-\bar{f}_{\lambda,\rho}\Vert_{L^{2}(\P^{\mathbf{X}})}>2s{}_{n}\mid\mathcal{D}_{n})<\alpha\big)\\
 & =1-\P\big(\tilde{\Pi}_{\lambda,\rho}(\theta:\Vert f_{\theta}-\bar{f}_{\lambda,\rho}\Vert_{L^{2}(\P^{\mathbf{X}})}>2s_{n}\mid\mathcal{D}_{n})\ge\alpha\big)\\
 & \ge1-\alpha^{-1}\E\big[\tilde{\Pi}_{\lambda,\rho}(\theta:\Vert f_{\theta}-\bar{f}_{\lambda,\rho}\Vert_{L^{2}(\P^{\mathbf{X}})}>2s{}_{n}\mid\mathcal{D}_{n})\big]\\
 & \ge1-\alpha^{-1}\big(\E\big[\tilde{\Pi}_{\lambda,\rho}(\theta:\Vert f_{\theta}-f\Vert_{L^{2}(\P^{\mathbf{X}})}>s{}_{n}\mid\mathcal{D}_{n})\big]+\P\big(\Vert \bar{f}_{\lambda,\rho}-f\Vert_{L^{2}(\P^{\mathbf{X}})}>s{}_{n}\big)\big)\\
 & \ge1-\alpha.
\end{align*}
The first statement in \ref{thm:credibleSet} is thus verified.

\noindent For the coverage statement, we denote $\bar\xi\coloneqq\xi\Delta$ and bound
\begin{align*}
\P\big(f\in\widehat{C}(\xi\tau_{\alpha}^{\theta_0})\big) & =\P\big(\Vert f-\bar{f}_{\lambda,\rho}\Vert_{L^{2}(\P^{\mathbf{X}})}\le\xi\tau_{\alpha}^{\theta_0}\big)\\
 & \ge\P\big(\tilde{\Pi}_{\lambda,\rho}(\theta:\vert\theta-\theta_0\vert_{\Theta}\le\bar\xi^{-1}\Vert f-\bar{f}_{\lambda,\rho}\Vert_{L^{2}(\P^{\mathbf{X}})}\mid\mathcal{D}_{n})<1-\alpha\big)\\
 & \ge\P\big(\tilde{\Pi}_{\lambda,\rho}(\theta:\vert\theta-\theta_0\vert_{\Theta}\le\bar\xi^{-1}s_{n}\mid\mathcal{D}_{n})<1-\alpha\big)-\alpha^{2}\\
 & =1-\alpha^{2}-\P\big(\tilde{\Pi}_{\lambda,\rho}(\theta:\vert\theta-\theta_0 \vert_{\Theta}\le\bar \xi^{-1}s_{n}\mid\mathcal{D}_{n})\ge 1-\alpha\big)\\
 & \ge1-\alpha^{2}-(1-\alpha)^{-1}\E\big[\tilde{\Pi}_{\lambda,\rho}(\mathcal{B}_{s_n/\bar\xi}(\theta_0)\mid\mathcal{D}_{n})\big]
\end{align*}
with balls
\[
\mathcal B_\kappa(\theta_0)=\big\{\theta\in\Theta:\vert \theta-\theta_0\vert_{\Theta}\le\kappa\big\},\qquad \kappa>0.
\]
In terms of $\tilde{\mathcal{E}}_{n}(\theta)=\tilde{R}_{n,\rho}(\theta)-\tilde{R}_{n,\rho}(f)$
and $\tilde{D}_{\lambda}=\int\exp\big(-\lambda\tilde{R}_{n,\rho}(\theta)\big)\,\Pi(\d\theta)$
the inequalities by Cauchy-Schwarz and H\"older imply for $\qnorm\ge2$
\begin{align*}
\E\big[\tilde{\Pi}_{\lambda,\rho}(\mathcal B_{s_n/\bar\xi}(\theta_0)&\mid \mathcal{D}_{n})\big]=  \E\Big[\tilde{D}_{\lambda}^{-1}\int_{\mathcal B_{s_n/\bar\xi}(\theta_0)}\e^{-\lambda\tilde{R}_{n,\rho}(\theta)}\,\Pi(\d\theta)\Big]\\
 &= \E\Big[\tilde{D}_{\lambda}^{-1}\e^{-\lambda\tilde{R}_{n,\rho}(f)}\int_{\mathcal B_{s_n/\bar\xi}(\theta_0)}\e^{-\lambda\tilde{\mathcal{E}}_{n}(\theta)}\,\Pi(\d\theta)\Big]\\
&\le \E\big[\tilde{D}_{\lambda}^{-2}\e^{-2\lambda\tilde{R}_{n,\rho}(f)}\big]^{1/2} \E\Big[\Big(\int_{\mathcal B_{s_n/\bar\xi}(\theta_0)}\e^{-\lambda\tilde{\mathcal{E}}_{n}(\theta)}\,\Pi(\d\theta)\Big)^{2}\Big]^{1/2}\\
&\le  \E\big[\tilde{D}_{\lambda}^{-2}\e^{-2\lambda\tilde{R}_{n,\rho}(f)}\big]^{1/2}\E\Big[\Pi(\mathcal B_{s_n/\bar\xi}(\theta_0))^{2(1-1/\qnorm)}\Big(\int_{\mathcal B_{s_n/\bar\xi}(\theta_0)}\e^{-\qnorm\lambda\tilde{\mathcal{E}}_{n}(\theta)}\,\Pi(\d\theta)\Big)^{2/\qnorm}\Big]^{1/2}.
\end{align*}
Abbreviating $\mathcal B_1=\mathcal B_1(0)$ as above, we note that the uniform prior yields 
\[
\Pi(\mathcal B_{s_n/\bar\xi}(\theta_0))\le\Pi(\mathcal B_{s_n/\bar\xi}(0))=\big(\frac{s_{n}}{\bar \xi}\big)^\pardim\Pi(\mathcal B_1)
\]
which is not random. To bound the expectation of the integral we extend the integration domain to $D=\Theta$ if $\theta_0$ is data dependent and keep $D=\mathcal B_{s_n/\bar\xi}(\theta_0)$ otherwise. 
The smaller choice of $\lambda=n/(\qnorm\K{0}{})$ instead of $n/\K{0}$ ensures $C_{n,\qnorm\lambda}+\frac{\qnorm\lambda}{n}(\sigma \Cf+\sigma^{2})\le\frac{1}{8}$ allowing us to apply \ref{prop:compatibility} with $\qnorm\lambda$. The second factor can thus be bounded using 
\begin{align*}
\E\Big[\int_{D}\e^{-\qnorm\lambda\tilde{\mathcal{E}}_{n}(\theta)}\,\Pi(\d\theta)\Big] & =\int_{D}\E\big[\e^{-\qnorm\lambda\tilde{\mathcal{E}}_{n}(\theta)}\big]\,\Pi(\d\theta)\\
 & \le\int_{D}\exp\big(\big(C_{n,\qnorm\lambda}+\tfrac{3}{4}+\tfrac{\qnorm\lambda}{n}(\sigma \Cf+\sigma^{2})-1\big)\qnorm\lambda\mathcal{E}(\theta)\big)\,\Pi(\d\theta)
  \\
  &\le\Pi(D).
\end{align*}
Based on \ref{eq:constant}, we conclude with $R_\qnorm=\big(\frac{s_{n}}{\bar \xi}\big)^{-\pardim/\qnorm}\Pi(\mathcal B_1)^{-1/\qnorm}$ for a data-dependent $\theta_0$ and $R_\qnorm=1$ otherwise that
\begin{align*}
\E\big[\tilde{\Pi}_{\lambda,\rho}(&\mathcal{B}_{s_n/\bar{\xi}}(\theta_0)\mid \mathcal{D}_{n})\big]\\
&\le  \Pi(\mathcal B_1)\big(\frac{s_{n}}{\bar \xi}\big)^\pardim R_q\E\big[\tilde{D}_{\lambda}^{-2}\e^{-2\lambda\tilde{R}_{n,\rho}(f)}\big]^{1/2}\\
&=  \Pi(\mathcal B_1)\big(\frac{s_{n}}{\bar \xi}\big)^\pardim R_q\E\Big[\exp\Big(\inf_{\varrho\ll\Pi}\Big(2\KL(\varrho\mid\Pi)+2\int\lambda\tilde{R}_{n,\rho}(\theta)\,\d\varrho(\theta)\Big)-2\lambda\tilde{R}_{n,\rho}(f)\Big)\Big]^{1/2}\\
&=  \Pi(\mathcal B_1)\big(\frac{s_{n}}{\bar \xi}\big)^\pardim R_q\E\Big[\exp\Big(\inf_{\varrho\ll\Pi}\Big(2\KL(\varrho\mid\Pi)+\int2\lambda\mathcal{\tilde{E}}_{n}(\theta)\,\d\varrho(\theta)\Big)\Big)\Big]^{1/2}.
\end{align*}
For $\varrho_{\eta'}$ defined via
\[
\diff{\varrho_{\eta'}}{\Pi}(\theta)\propto\1_{\{\vert\theta-\theta^{\ast}\vert_{\Theta}\le\eta'\}},\qquad\eta'=\frac{s_n}{4\Delta \sqrt{\log(B\Delta n)}}.
\]
we can moreover bound using \ref{eq:intErho}, \ref{assu:F}
and \ref{lem:aux_a}
\begin{align*}
  \inf_{\varrho\ll\Pi}\Big(\KL(\varrho\mid\Pi)+\int\lambda\mathcal{\tilde{E}}_{n}(\theta)\,\d\varrho(\theta)\Big)
 & \le\KL(\varrho_{\eta'}\mid\Pi)+\frac{4}{3}\lambda\mathcal{E}(\theta^{\ast})+4\lambda\int\E\big[\big(f_{\theta^{\ast}}(\mathbf{X})-f_{\theta}(\mathbf{X})\big)^{2}\big]\,\mathrm{d}\varrho_{\eta'}(\theta)\\
 & \qquad \qquad+\lambda\int\big(\tilde{\mathcal{E}}_{n}(\theta)-\mathcal{E}(\theta)\big)\,\d\varrho_{\eta'}(\theta)\\
 & \le \pardim\log\frac{1}{\eta'}-\log\Pi(\mathcal B_1)+\frac{4}{3}\lambda\mathcal{E}(\theta^{\ast})+\frac{\lambda s_{n}^2}{\log(B\Delta n)}\\
 &\qquad\qquad+\lambda\int\big(\mathcal{\tilde{E}}_{n}(\theta)-\mathcal{E}(\theta)\big)\,\d\varrho_{\eta'}(\theta).
\end{align*}
In the sequel $\K{i}>0,i=5,6,\dots$, are numerical constants which
may depend on $\Cf,\Ceps,\sigma,\inputdim$ and $\alpha$. Since $\log(B\Delta n)\mathcal{E}(\theta^{\ast})\le s_{n}^2\le\K{5} \pardim\log(B\Delta n)/\lambda$
by assumption, we obtain
\begin{align*}
\E\big[\tilde{\Pi}_{\lambda,\rho}(\mathcal{B}_{s_n/\bar{\xi}}(\theta_0)\mid \mathcal{D}_{n})\big] & \le R_q\exp\big(- \pardim\log\bar\xi+ \pardim\log\big(4\Delta\sqrt{\log(B\Delta n)}\big)+3\K{5} \pardim\big)\\
 & \qquad\qquad\times\E\Big[\exp\Big(2\lambda\int\big(\tilde{\mathcal{E}}_{n}(\theta)-\mathcal{E}(\theta)\big)\,\d\varrho_{\eta'}(\theta)\Big)\Big]^{1/2}\\
 & \le R_q\exp\big(- \pardim\log\xi+\pardim(\K{6}+\log\sqrt{\log(B\Delta n)})\big)\\
 &\qquad\qquad \times\E\Big[\int\exp\big(2\lambda\big(\tilde{\mathcal{E}}_{n}(\theta)-\mathcal{E}(\theta)\big)\big)\,\d\varrho_{\eta'}(\theta)\Big]^{1/2}
\end{align*}
applying Jensen's inequality in the last line. To bound the expectation
in the previous line, Fubini's theorem, \ref{prop:compatibility}
with $C_{n,2\lambda}+\frac{2\lambda}{n}(\sigma \Cf+\sigma^{2})\le\frac{1}{8}$
and \ref{lem:Lipschitz} imply
\begin{align*}
 \E\Big[\int\exp\big(2\lambda\big(\tilde{\mathcal{E}}_{n}(\theta)-\mathcal{E}(\theta)\big)\big)\,\mathrm{d}\varrho_{\eta'}(\theta)\Big]
 & =\int\E\big[\exp\big(2\lambda\big(\tilde{\mathcal{E}}_{n}(\theta)-\mathcal{E}(\theta)\big)\big)\big]\,\d\varrho_{\eta'}(\theta)\\
 & \le\int\exp\big(2\lambda\big(C_{n,2\lambda}+\tfrac{2\lambda}{n}(\sigma \Cf+\sigma^{2})\big)\mathcal{E}(\theta)\big)\,\d\varrho_{\eta'}(\theta)\\
 & \le\int\exp\big(\tfrac{1}{4}\lambda\mathcal{E}(\theta)\big)\,\d\varrho_{\eta'}(\theta)\\
 & \le\int\exp\big(\tfrac{1}{2}\lambda\big(\mathcal{E}(\theta^{*})+\Vert f_{\theta}-f_{\theta^{*}}\Vert_{L^{2}(\P^{\mathbf{X}})}^{2}\big)\big)\,\d\varrho_{\eta'}(\theta)\\
 & \le\int\exp\big(\tfrac{1}{2}\lambda\big(\mathcal{E}(\theta^{*})+s_{n}^2/(\log(B\Delta n)\big)\big)\,\d\varrho_{\eta'}(\theta)\\
 & \le \e^{\K{5} \pardim}.
\end{align*}
We conclude
\begin{align*}
\E\big[\tilde{\Pi}_{\lambda,\rho}(\mathcal{B}_{s_n/\bar{\xi}}(\theta_0)\mid \mathcal{D}_{n})\big] & \le R_\qnorm\exp\big(- \pardim\big(\log\xi-\K{6}-\K{5}-\log\sqrt{ \log(B\Delta n)}\big)\big).
\end{align*}
If $R_\qnorm=1$,  we obtain for a sufficiently large $\K{7}>0$ and $\xi=\K{7} \sqrt{\log(B\Delta  n)}$ that $\E\big[\tilde{\Pi}_{\lambda,\rho}(\mathcal{B}_{s_n/\bar{\xi}}(\theta_0)\mid \mathcal{D}_{n})\big]\le\alpha(1-\alpha)^{2}$
and thus
\[
\P\big(f\in\widehat{C}(\xi \tau_{\alpha}^{\theta_0})\big)\ge1-\alpha^{2}-\alpha(1-\alpha)\ge1-\alpha.
\]
If $R_\qnorm=\big(s_{n}/\bar \xi\big)^{-\pardim/\qnorm}\Pi(\mathcal B_1)^{-1/\qnorm}$, we obtain
\begin{equation*}
\E\big[\tilde{\Pi}_{\lambda,\rho}(\mathcal{B}_{s_n/\bar{\xi}}(\theta_0)\mid \mathcal{D}_{n})\big] \le \Pi(\mathcal B_1)^{-1/\qnorm}\exp\big(- \pardim\big(\log\xi-\log\big((\Delta\xi/s_n)^{1/\qnorm}\big)-\K{6}-\K{5}-\log\sqrt{ \log(B\Delta n)}\big)\big)
\end{equation*}
such that we need $\xi=\big(\Delta/s_n\big)^{1/\qnorm'}$ for any $\qnorm'<\qnorm$ and sufficiently large $n$. The claim the follows by choosing $\qnorm'=\qnorm/2$.\qed

\subsection{Proofs for \ref{sec:NN}}\label{subsec:proofsNN}

To prove \ref{thm:oracleinequalityShallow}, \ref{prop:UQ} and \ref{thm:oracleinequalityNN}, we need the following bound:
\begin{lem}
\label{lem:Lipschitz}Let $\theta,\tilde{\theta}\in[-B,B]^{ \pardim}$.
Then we have for $\mathbf{x}\in\R^{\inputdim}$ that
\[
\vert f_{\theta}(\mathbf{x})-f_{\tilde{\theta}}(\mathbf{x})\vert\le 4(2rB)^{L}(\vert\mathbf{x}\vert_{1}\lor1)\vert \theta-\tilde{\theta}\vert_\infty.
\]
If $L=1$, we also have for all $\mathbf{x}\in\R^{\inputdim}$
\[
\vert f_{\theta}(\mathbf{x})-f_{\tilde{\theta}}(\mathbf{x})\vert\le B(\vert\mathbf{x}\vert_{1}+1)\vert \theta-\tilde{\theta}\vert_1\le (\inputdim+3)rB(\vert\mathbf{x}\vert_{1}+1)\vert \theta-\tilde{\theta}\vert_\infty.
\]
\end{lem}
\subsubsection{Proof of \ref{thm:oracleinequalityShallow} and \ref{thm:oracleinequalityNN}}\label{subsec:networkproofs}
It remains to verify \ref{assu:F} for the considered class of neural networks where we choose $|\cdot|_\Theta=|\cdot|_\infty$. While the boundedness is ensured by construction of $\mathcal F(\inputdim,L,r,\Cf)$, the Lipschitz continuity with respect to $|\cdot|_\Theta=|\cdot|_\infty$ is verified by \ref{lem:Lipschitz}. In particular,
\[
\|f_\theta-f_{\tilde\theta}\|_{L^2(\P^{\mathbf X})}\le 4(2rB)^{L}|\theta-\tilde\theta|_\infty\E[|\mathbf X|_1^2\vee 1]^{1/2}\le 8\E[\vert \mathbf{X}\vert^2]\inputdim(2rB)^L|\theta-\tilde\theta|_\infty
\]
or $\|f_\theta-f_{\tilde\theta}\|_{L^2(\P^{\mathbf X})}\le  (\inputdim+3)rB(\E[\vert \mathbf{X}\vert^2]+1)|\theta-\tilde\theta|_\infty$ for $L=1$.
The statement is then an immediate consequence of \ref{thm:oracleinequality}.\hfill\qed
\subsubsection{Proof of \ref{prop:UQ}}
    This is a straight-forward application of \ref{thm:credibleSet} in combination with \ref{lem:Lipschitz}, where we used $|\cdot|_\Theta=|\cdot|_1$ with $\Delta=B
    \E[(\vert \mathbf{X}\vert_1+1)^2]^{1/2}$ for the credibility statement. \hfill \qed

\subsubsection{Proof of \ref{thm:learningthewidth}}

The outline of the proof \ref{thm:learningthewidth} is similar to that of \ref{thm:oracleinequality}. To bound the Kullback-Leibler term in \ref{eq:pac_adaptive}, we employ the following modification of \ref{lem:aux_a}:
\begin{lem}\label{lem:aux_adaptive}
    We have $\KL(\varrho_{r,\eta}\mid\widecheck{\Pi})\le {\pardim_r}\log(2B/\eta)+r$.
\end{lem}
\noindent Note that the only property of the prior that we used in the proof of  \ref{prop:mainpaclemma} is that $\Pi$ is a probability measure on the space of network weights. Hence, it is straightforward to see that the analogous statement still holds when replacing $\Pi$ with $\widecheck{\Pi}$. We obtain with probability at least $1-\delta$ 
\begin{equation}
\mathcal{E}(\widecheck{\theta}_{\lambda,\rho})\le9\int\mathcal{E}\,\d\varrho+\frac{16}{\lambda}\big(\KL(\varrho\mid\widecheck{\Pi})+\log(2/\delta)\big).\label{eq:pac_adaptive}
\end{equation}
For a width $r\in\N$ and some radius $\eta\in(0,1]$, we now choose  $\varrho=\varrho_{r,\eta}$ defined via
\[
\diff{\varrho_{r,\eta}}{\Pi_r}(\theta)\propto\1_{\{\vert\theta-\theta^{\ast}_r\vert_{\infty}\le\eta\}}
\]
with $\theta_r^{*}$ from \ref{eq:oracleL}. 
Replacing $\theta^{\ast}$ with $\theta^{\ast}_r$ in the arguments from \ref{subsec:ProofOracle} and verifying \ref{assu:F} \labelcref{assu:lipschitz} as in \ref{subsec:networkproofs}, we find
\begin{equation}
\int\mathcal{E}\,\d\varrho_{r,\eta}\le\frac{4}{3}\mathcal{E}(\theta^{\ast}_r)+\frac{1}{n^{2}}\qquad\text{for}\qquad\eta=\frac{1}{8\E[\vert\mathbf{X}\vert^2](2rB)^Lp\,n}.\label{eq:integralterm_adaptive}
\end{equation}
Owing to \ref{lem:aux_adaptive}, we thus have with probability $1-\delta$
\begin{equation}
\mathcal{E}(\widecheck{\theta}_{\lambda,\rho})\le 12\mathcal{E}(f_{\theta^{\ast}_r})+\frac{\K{1}{}}{n}\big({\pardim_r} L\log(rBn)+\log(2/\delta)\big),
\end{equation}
for some $\K{1}>0$ only depending on $\Cf,\Ceps,\sigma$. Choosing $r$ to minimize the upper bound in the last display yields the assertion.\hfill\qed

\subsubsection{Proof of \ref{lem:ApproxShallow}}
  By \citet[Theorem 2]{yang2024nonparametric} for any $f\in\mathcal C^\beta([0,1]^\inputdim,c_0)$ there is a shallow network $g_\theta(\mathbf x)=\sum_{i=1}^rW_i^{(2)}\phi(W_{i,\cdot}^{(1)}\mathbf x+v_i^{(1)})$ with $|W^{(1)}_{i,\cdot}|^2+|v_i^{(1)}|^2=1$ for all $i=1,\dots,r$ ($v^{(2)}=0$) and $|W^{(2)}|_1\le M$ such that
  \[\|g_\theta-f\|_\infty\le\cshallow r^{-\beta/\inputdim}\vee M^{-2\beta/(\inputdim+3-2\beta)}.\]
  Choosing $M=r^{(\inputdim+3-2\beta)/(2 \inputdim)}$, we obtain $\|g_\theta-f\|_\infty\le\cshallow r^{-\beta/\inputdim}$. Note that the normalization of $W^{(1)}_{i,\cdot}$ and $v_i^{(1)}$ implies $|W_{ij}^{(1)}|\le1$ and $|v_i^{(1)}|\le1$ for all $i=1,\dots,r,j=1,\dots,\inputdim$. 
  
  To obtain a uniform bound on the entries of the vector $W^{(2)}\in\R^\inputdim$, we note that for each entry with $|W^{(2)}_i|\ge B$ we can reproduce $W_i^{(2)}\phi(W_{i,\cdot}^{(1)}\mathbf x+v_i^{(1)})$ by summing over $\lceil |W^{(2)}_i|/B\rceil$ copies of the neuron $\phi(W_{i,\cdot}^{(1)}\mathbf x+v_i^{(1)})$ weighted with weights bounded by $B$. As a result we can reproduce $g_\theta$ with a shallow network with at most $r+|W^{(2)}|_1/B\le r+M/B$ neurons and weights uniformly bounded by $B$. Choosing $B=1\vee r^{(3-\inputdim-2\beta)/(2 \inputdim)}$ ensures that $M/B\le r$. \hfill\qed

\subsection{Remaining proofs for \ref{sec:Oracle-inequality}}\label{sec:remainingProofs}

\subsubsection{Proof of \ref{lem:KLApprox}}

Define
\begin{equation*}
D_{\lambda}\coloneqq\int\exp\big(-\lambda R_{n}(\theta)\big)\,\Pi(\d\theta),\qquad\bar{D}_{\lambda}\coloneqq\int\exp\big(-\lambda\bar{R}_{n,\rho}(\theta)\big)\,\Pi(\d\theta).\label{eq:D_standardize}
\end{equation*}
For the first part of the lemma, we write
\begin{align}
\KL\big(\bar{\Pi}_{\lambda,\rho}(\cdot\mid \mathcal{D}_{n})\,\big\vert\,\Pi_{\lambda}(\cdot\mid \mathcal{D}_{n})\big) & =\int\log\frac{\d\bar{\Pi}_{\lambda,\rho}(\theta\mid \mathcal{D}_{n})}{\d\Pi_{\lambda}(\cdot\mid \mathcal{D}_{n})}\,\bar{\Pi}_{\lambda,\rho}(\d\theta\mid \mathcal{D}_{n})\nonumber \\
 & =\lambda\int S_{n}(\theta)\,\bar{\Pi}_{\lambda,\rho}(\d\theta\mid \mathcal{D}_{n})+\log\frac{D_{\lambda}}{\bar{D}_{\lambda}}\qquad\text{with}\nonumber \\
S_{n}(\theta) & \mspace{-3.7mu}\coloneqq R_{n}(\theta)-\bar{R}_{n,\rho}(\theta).\label{eq:Sn}
\end{align}
By concavity of the logarithm we have
\[
\frac{1}{\lambda}\sum_{i=1}^{n}\log\big(\rho \e^{-\frac{\lambda}{n\rho}\ell_{i}(\theta)}+1-\rho\big)\ge\frac{1}{\lambda}\sum_{i=1}^{n}\rho\log \e^{-\frac{\lambda}{n\rho}\ell_{i}(\theta)}+(1-\rho)\log1=-\frac{1}{n}\sum_{i=1}^{n}\ell_{i}(\theta)=-R_{n}(\theta).
\]
Hence, $S_{n}(\theta)\ge\text{0}$ and $D_{\lambda}\le\bar{D}_{\lambda}$.
We conclude
\[
\KL\big(\bar{\Pi}_{\lambda,\rho}(\cdot\mid \mathcal{D}_{n})\,\big\vert\,\Pi_{\lambda}(\cdot\mid \mathcal{D}_{n})\big)\le\lambda\int S_{n}(\theta)\,\bar{\Pi}_{\lambda,\rho}(\d\theta\mid \mathcal{D}_{n}).
\]
Moreover, $\log(x+1)\le x$ for all $x>-1$ and a second order Taylor
expansion of $x\mapsto \e^{x}$ yields
\begin{align*}
S_{n}(\theta) & =\frac{1}{\lambda}\sum_{i=1}^{n}\big(\log\big(\rho(\e^{-\frac{\lambda}{n\rho}\ell_{i}(\theta)}-1)+1\big)+\frac{\lambda}{n}\ell_{i}(\theta)\big)\\
 & \le\frac{\rho}{\lambda}\sum_{i=1}^{n}\big(\e^{-\frac{\lambda}{n\rho}\ell_{i}(\theta)}-1+\tfrac{\lambda}{n\rho}\ell_{i}(\theta)\big)\\
 & \le\frac{\rho}{2\lambda}\sum_{i=1}^{n}\big(\tfrac{\lambda}{n\rho}\ell_{i}(\theta)\big)^{2}\e^{-\frac{\lambda}{n\rho}\ell_{i}(\theta)}\\
 & \le\frac{\lambda}{n\rho}\cdot\frac{1}{2n}\sum_{i=1}^{n}\vert\ell_{i}(\theta)\vert^{2}.
\end{align*}
For $\ell_{i}(\theta)=\vert Y_{i}-f_{\theta}(\mathbf{X}_{i})\vert^{2}\le2\vert f(\mathbf{X}_{i})-f_{\theta}(\mathbf{X}_{i})\vert^{2}+2\eps_{i}^{2}\le8\Cf^{2}+2\eps_{i}^{2}$
we obtain 
\[
S_{n}(\theta)\le\frac{\lambda}{n\rho}\Big(64\Cf^{4}+\frac{4}{n}\sum_{i=1}^{n}\eps_{i}^{4}\Big)
\]
and thus
\begin{align*}
\frac{1}{\lambda}\KL\big(\bar{\Pi}_{\lambda,\rho}(\cdot\mid \mathcal{D}_{n})\,\big\vert\,\Pi_{\lambda}(\cdot\mid \mathcal{D}_{n})\big) & \le\frac{\lambda}{n\rho}\Big(64\Cf^{4}+\frac{4}{n}\sum_{i=1}^{n}\eps_{i}^{4}\Big)\int\bar{\Pi}(\d\theta\mid \mathcal{D}_{n})\\
 & =\frac{\lambda}{n\rho}\Big(64\Cf^{4}+\frac{4}{n}\sum_{i=1}^{n}\eps_{i}^{4}\Big).
\end{align*}

\noindent In the regime $\rho\to0$, define 
\[
T_{n}(\theta)\coloneqq-\rho n\frac{1}{n}\sum_{i=1}^{n}\e^{-\frac{\lambda}{n\rho}\ell_{i}(\theta)}\qquad\text{and}\qquad D_{\varpi,\lambda}\coloneqq\int\exp\big(-T_n(\theta)\big)\,\Pi(\d \theta)
\]
such that
\[
\KL\big(\bar{\Pi}_{\lambda,\rho}(\cdot\mid \mathcal{D}_{n})\,\big\vert\,\varpi_{\lambda,\rho}(\cdot\mid \mathcal{D}_{n})\big)=\int\big(T_{n}(\theta)-\lambda\bar{R}_{n,\rho}(\theta)\big)\,\bar{\Pi}_{\lambda,\rho}(\d\theta\mid \mathcal{D}_{n})+\log\frac{D_{\varpi,\lambda}}{\bar{D}_{\lambda}}.
\]
We have 
\begin{align*}
\lambda\bar{R}_{n,\rho}(\theta)-T_{n}(\theta) & =-\sum_{i=1}^{n}\log\big(\rho \e^{-\frac{\lambda}{n\rho}\ell_{i}(\theta)}+1-\rho\big)-T_{n}(\theta)\\
 & =-n\log(1-\rho)-\sum_{i=1}^{n}\big(\log\big(\rho \e^{-\frac{\lambda}{n\rho}\ell_{i}(\theta)}+1-\rho\big)-\log(1-\rho)\big)-T_{n}(\theta)\\
 & =-n\log(1-\rho)-\sum_{i=1}^{n}\rho \e^{-\frac{\lambda}{n\rho}\ell_{i}(\theta)}\int_{0}^{1}\big(t\rho \e^{-\frac{\lambda}{n\rho}\ell_{i}(\theta)}+1-\rho\big)^{-1}\,\d t-T_{n}(\theta)\\
 & =-n\log(1-\rho)-\sum_{i=1}^{n}\rho \e^{-\frac{\lambda}{n\rho}\ell_{i}(\theta)}\int_{0}^{1}\Big(\frac{1}{t\rho \e^{-\frac{\lambda}{n\rho}\ell_{i}(\theta)}+1-\rho}-1\Big)\,\d t,
\end{align*}
where $(t\rho \e^{-\frac{\lambda}{n\rho}\ell_{i}(\theta)}+1-\rho)^{-1}-1\in[0,\frac{\rho}{1-\rho}]$.
Therefore,
\begin{equation}
-\frac{\rho^{2}}{(1-\rho)}\sum_{i=1}^{n}\e^{-\frac{\lambda}{n\rho}\ell_{i}(\theta)}\le\lambda\bar{R}_{n,\rho}(\theta)-T_{n}(\theta)+n\log(1-\rho)\le0.\label{eq:kullbackEst}
\end{equation}
This implies $\log\frac{D_{\varpi,\lambda}}{\bar{D}_{\lambda}}\le-n\log(1-\rho)$
and thus
\begin{equation}
\KL\big(\bar{\Pi}_{\lambda,\rho}(\cdot\mid \mathcal{D}_{n})\,\big\vert\,\varpi_{\lambda,\rho}(\cdot\mid \mathcal{D}_{n})\big)  \le\frac{\rho^{2}}{1-\rho}\int\sum_{i=1}^{n}\e^{-\frac{\lambda}{n\rho}\ell_{i}(\theta)}\,\bar{\Pi}_{\lambda,\rho}(\d\theta\mid \mathcal{D}_{n})   \le\frac{\rho^{2}n}{1-\rho}.\tag*{{\qed}}
\end{equation}

\subsubsection{Proof of \ref{lem:KLtilde}}
Recall $\psi_{\rho}(x)=-\log(\e^{-x}+1-\rho)$, $\psi_{\rho}'(x)=\frac{1}{1+(1-\rho)\e^{x}}$ and $\psi_{\rho}''(x)=-\frac{(1-\rho)\e^{x}}{(1+(1-\rho)\e^{x})^{2}}\in[-1/4,0]$. Since
\begin{align*}
\tilde R_{n,\rho}(\theta)	
    &=\frac{1}{\lambda}\sum_{i=1}^{n}\psi_{\rho}\big(\tfrac{\lambda}{n}\ell_{i}(\theta)\big)\\
	&=\frac{n}{\lambda}\psi_{\rho}(0)+\frac{1}{\lambda}\sum_{i=1}^{n}\tfrac{\lambda}{n}\ell_{i}(\theta)\psi'_{\rho}\big(\xi_{i}\tfrac{\lambda}{n}\ell_{i}(\theta)\big)\\
	&=\frac{n}{\lambda}\psi_{\rho}(0)+\frac{\psi_{\rho}'(0)}{n}\sum_{i=1}^{n}\ell_{i}(\theta)+\frac{1}{n}\sum_{i=1}^{n}\ell_{i}(\theta)\big(\psi'_{\rho}(\xi_{i}\tfrac{\lambda}{n}\ell_{i}(\theta))-\psi_{\rho}'(0)\big)\\
	&=-\frac{n}{\lambda}\log(2-\rho)+\frac{1}{2-\rho}R_{n}(\theta)+\frac{\lambda}{n^{2}}\sum_{i=1}^{n}\ell_{i}(\theta)^{2}\xi_i\psi''_{\rho}\big(\xi_{i}'\tfrac{\lambda}{n}\ell_{i}(\theta)\big),
\end{align*} 
we have
\[-\frac{\lambda^{2}}{4n^{2}}\sum_{i=1}^{n}\ell_{i}(\theta)^{2}\le\lambda\tilde R_{n,\rho}(\theta)-\frac{\lambda}{2-\rho}R_{n}(\theta)+n\log(2-\rho)\le0.\]
Therefore, we have with $\tilde{D}_\lambda$ from \ref{eq:D_tilde} that
\begin{align*}
\KL\big(\tilde{\Pi}_{\lambda,\rho}(\cdot\mid \mathcal{D}_{n})\,\big\vert\,\Pi_{\lambda/(2-\rho)}(\cdot\mid \mathcal{D}_{n})\big)
    &=\int\Big(\frac{\lambda}{2-\rho}R_{n}(\theta)-\lambda\tilde R_{n,\rho}(\theta)\Big)\,\tilde{\Pi}_{\lambda,\rho}(\d\theta\mid \mathcal{D}_{n})+\log\frac{D_{\lambda/(2-\rho)}}{\tilde D_{\lambda}}\\
	&\le\int\Big(\frac{\lambda}{2-\rho}R_{n}(\theta)-\lambda\tilde R_{n,\rho}(\theta)-n\log(2-\rho)\Big)\,\tilde{\Pi}_{\lambda,\rho}(\d\theta\mid \mathcal{D}_{n})\\
	&\le\frac{\lambda^{2}}{4n}\int\frac{1}{n}\sum_{i=1}^{n}\ell_{i}(\theta)^{2}\,\tilde{\Pi}_{\lambda,\rho}(\d\theta\mid \mathcal{D}_{n})\\
	&\le\frac{\lambda^{2}}{n}\Big(32\Cf^{4}+\frac{2}{n}\sum_{i=1}^{n}\eps_{i}^{4}\Big).\tag*{{\qed}}
\end{align*}
\subsubsection{Proof of \ref{cor:mean}}

Jensen's and Markov's inequality yield for $\rate^{2}$ from \ref{eq:r2}
that
\begin{align*}
\P\Big(\mathcal{E}(\bar{f}_{\lambda,\rho})>\rate^{2}+\frac{\K{1}}{n}+\frac{\K{1}}{n}\log(2/\delta)\Big) & =\P\Big(\Vert \E[f_{\tilde{\theta}_{\lambda,\rho}}\mid \mathcal{D}_{n}]-f\Vert_{L^{2}(\P^{\mathbf{X}})}^{2}>\rate^{2}+\frac{\K{1}}{n}+\frac{\K{1}}{n}\log(2/\delta)\Big)\\
 & \le\P\Big(\E\big[\Vert f_{\tilde{\theta}_{\lambda,\rho}}-f\Vert_{L^{2}(\P^{\mathbf{X}})}^{2}\,\big\vert\,\mathcal{D}_{n}\big]>\rate^{2}+\frac{\K{1}}{n}+\frac{\K{1}}{n}\log(2/\delta)\Big)\\
 & =\P\Big(\int_{\frac{\K{1}}{n}\log(2/\delta)}^{\infty}\tilde{\Pi}_{\lambda,\rho}\big(\Vert f_{\tilde{\theta}_{\lambda,\rho}}-f\Vert_{L^{2}(\P^{\mathbf{X}})}^{2}>\rate^{2}+t\,\big\vert\,\mathcal{D}_{n}\big)\,\d t>\frac{\K{1}}{n}\Big)\\
 & \le\frac{n}{\K{1}{}}\int_{\frac{\K{1}{}}{n}\log(2/\delta)}^{\infty}\E\big[\tilde{\Pi}_{\lambda,\rho}\big(\Vert f_{\tilde{\theta}_{\lambda}}-f\Vert_{L^{2}(\P^{\mathbf{X}})}^{2}>\rate^{2}+t\,\big\vert\,\mathcal{D}_{n}\big)\big]\,\d t.
\end{align*}
Using \ref{thm:oracleinequality}, we thus obtain
\[
\P\Big(\mathcal{E}(\bar{f}_{\lambda,\rho})>\rate^{2}+\frac{\K{1}{}}{n}+\frac{\K{1}{}}{n}\log(2/\delta)\Big)\le\frac{2n}{\K{1}{}}\int_{\frac{\K{1}{}}{n}\log(2/\delta)}^{\infty} \e^{-n t/\K{1}{}}\,\d t=\delta.\tag*{{\qed}}
\]

\subsubsection{Proof of \ref{prop:rate}}

We combine arguments from \cite{schmidthieber2020} with the approximation
results from \cite{KohlerLanger2021}. By rescaling,
we can rewrite
\[
f=f_{\hdepth}\circ\cdots\circ f_{0}=h_{\hdepth}\circ\cdots\circ h_{0}
\]
with $h_{i}=(h_{ij})_{j=1,\dots,d_{i+1}}$, where $\tilde h_{0j}\in\mathcal{C}_{t_{0}}^{\beta_{0}}([0,1]^{t_{0}},1)$,
$\tilde h_{ij}\in\mathcal{C}_{t_{i}}^{\beta_{i}}\big([0,1]^{t_{i}},(2\cholder)^{\beta_{i}}\big)$
for $i=1,\dots,\hdepth-1$ and $\tilde h_{\hdepth j}\in\mathcal{C}_{t_{\hdepth}}^{\beta_{q}}\big([0,1]^{t_{\hdepth}},\cholder(2\cholder)^{\beta_{\hdepth}}\big)$
and $h_{ij}$ is $\tilde h_{ij}$ understood as a function in $d_{i}$
instead of $t_{i}$ arguments. 

\bigskip{}

We want to show that there exists a constant $C_i$ such that for any $M_{i}\in\N$ we can find sufficiently
large $L_{i},r_{i}\in\N$ and
a neural network $\tilde g_{ij}\in\mathcal{G}(t_{i},L_{i},r_{i})$
with ${\pardim_{L_i,r_i}}=c_{i}M^{t_{i}}$ parameters and
\begin{equation}
\Vert\tilde h_{ij}-\tilde g_{ij}\Vert_{L^{\infty}([0,1]^{t_{i}})}\le C_{i}M_{i}^{-2\beta_{i}}.\label{eq:building_blocks_networks}
\end{equation}
To construct such $g_{ij}$, we use Theorem 2(a) from \cite{KohlerLanger2021}.
Their conditions
\begin{enumerate}
\item $L_{i}	\ge5+\lceil\log_{4}(M^{2\beta_{i}})\rceil\big(\lceil\log_{2}(\max\{\lfloor\beta_{i}\rfloor,t_{i}\}+1\rceil+1\big)$ and
\item $r_{i}	\ge2^{t_{i}+6}\binom{t_{i}+\lfloor\beta_{i}\rfloor}{t_{i}}t_{i}^{2}(\lfloor\beta_{i}\rfloor+1)M_{i}^{t_{i}}$
\end{enumerate}
can be satisfied for $L_{i}=C_{i}\log(M_i),r_i=C_iM_{i}^{t_{i}}$, where $C_{i}$ only depends on upper
bounds for $t_{i}$ and $\beta_{i}$. Hence, there exists a neural
network $\tilde g_{ij}\in\mathcal{G}(t_{i},L_{i},r_{i})$ with \ref{eq:building_blocks_networks}. Careful inspection of the proof of this  theorem reveals that the weights and shifts of $\tilde{g}_{ij}$ grow at most polynomially in $M$.
Since $t_{i}\le d_{i},r_{i}$, we can easily embed $\tilde g_{ij}$
into the class $\mathcal{G}(d_{i},L_{i},r_{i})$ by setting $g_{ij}=\tilde g_{ij}(W_{ij}\, \cdot )$,
where the matrix $W_{ij}\in\R^{t_{i}\times d_{i}}$ is chosen such
that $g_{ij}$ depends on the same $t_{i}$ many arguments as $h_{ij}$.
Note that the approximation accuracy of $\tilde g_{ij}$ carries over
to $g_{ij}$, that is 
\begin{equation}
\Vert h_{ij}-g_{ij}\Vert_{L^{\infty}([0,1]^{d_{i}})}\le\Vert\tilde h_{ij}-\tilde g_{ij}\Vert_{L^{\infty}([0,1]^{t_{i}})}\le C_{i}M_{i}^{-2\beta_{i}}.\label{eq:build_blocks_embedded}
\end{equation}

\noindent Setting $g=g_{\hdepth}\circ\cdots\circ g_{0}$ with $g_{i}=(g_{ij})_{j}$
we obtain a neural network $g\in\mathcal{G}(\inputdim,L,r)$ with $r=\max_{i=0,\dots,\hdepth}r_{i}d_{i+1}$
and $L=\sum_{i=0}^{\hdepth}L_{i}$.

\noindent Counting the number of parameters of $g$ and using $L_{i}=C_{i}M_{i}^{t_{i}}$,
we get
\[
\pardim_{L,r}\le \K{8}\sum_{i=0}^{\hdepth}L_ir_i^2
\]
for some $\K{8}>0$.

It follows from \citet[Lemma 3]{schmidthieber2020} and \ref{eq:build_blocks_embedded}
that 
\[
\Vert f-g\Vert_{L^{\infty}([0,1]^{\inputdim})}\le \cholder\prod_{l=0}^{\hdepth-1}(2\cholder)^{\beta_{l+1}}\sum_{i=0}^{\hdepth}\big\Vert\vert h_{i}-g_{i}\vert_{\infty}\big\Vert_{L^{\infty}([0,1]^{d_{i}})}^{\prod_{l=i+1}^{\hdepth}\beta_{l}\land1}\le \K{9}\sum_{i=0}^{\hdepth}M_{i}^{-2\beta_{i}},
\]
for some $\K{9} >0$.

\noindent Applying \ref{thm:oracleinequality} together with $\mathcal{E}(f_{\theta^{\ast}})\le\Vert f-g\Vert_{L^{\infty}([0,1]^{\inputdim})}^{2}$ we now obtain
\begin{equation*}
\mathcal{E}(\tilde f_{\lambda,\rho})\le \K{10}\sum_{i=0}^\hdepth M_i^{-4\beta_i}
+\frac{\K{10}}{n}\sum_{i=0}^{\hdepth}M_i^{2t_i}(\log n)^3+\K{10}\frac{\log(2/\delta)}{n}\label{eq:balance_decomposition}
\end{equation*}
with probability  at least $1-\delta$. Choosing 
\begin{equation}
    M_i=\Big\lceil\Big(\frac{n}{(\log n)^3}\Big)^{1/(4\beta_i+2t_i)}\Big\rceil
\end{equation}
ensures $L,r\le n$ for sufficiently large $n$, balances the first two terms in the upper bound \ref{eq:balance_decomposition} and thus yields the asserted convergence rate for $\tilde{f}_{\lambda,\rho}$.

\noindent The convergence rate for the posterior mean can be proved analogously using \ref{cor:mean}. \hfill\qed

\subsubsection{Proof of \ref{cor:rate_adaptive}}
The statement follows by choosing $L$ in the upper bound from \ref{thm:learningthewidth} as in the statement of \ref{prop:rate} and then using the same approximation result to control excess-risk of the corresponding oracle choice $\theta^\ast_r$.\hfill\qed

\subsection{Proofs of the auxiliary results\label{sec:RemainingProofs}}\label{sec:AuxProofs}

\subsubsection{Proof of \ref{lem:aux_a}}
Denoting $\mathcal B_\eta(\theta^*)=\{\theta\in\R^{\pardim}:|\theta-\theta^*|_\Theta\le\eta\}$, we have $\diff{\varrho_{\eta}}{\Pi}(\theta)=\1_{\mathcal B_\eta(\theta^*)}/\Pi(\mathcal B_\eta(\theta^*))$. If $\mathcal B_\eta(\theta^*)\subset[-B,B]^\pardim$, the uniformity of $\Pi$ yields
\[
\KL(\varrho_{\eta}\mid\Pi)=\int\log\Big(\diff{\varrho_{\eta}}{\Pi}\Big)\d\varrho_\eta=-\log\big(\Pi(\mathcal B_\eta(\theta^*)\big)
=-\log\big(\Pi(\mathcal B_\eta(0)\big)=\pardim\log(2B/\eta)-\log\operatorname{vol}(\mathcal B_1(0)).\tag*{\qed}
\]

\subsubsection{Proof of \ref{lem:Lipschitz}}
Set $\eta\coloneqq\vert\theta-\tilde{\theta}\vert_{\infty}$ and let $W^{(1)},\dots,W^{(L+1)},v^{(1)},\dots,v^{(L+1)}$
and $\tilde W^{(1)},\dots,\tilde W^{(L+1)},\tilde v^{(1)},\dots,\tilde v^{(L+1)}$
be the weights and shifts associated with $\theta$ and $\tilde{\theta}$,
respectively. Define $\tilde{\mathbf{x}}^{(l)}$, $l=0,\dots,L+1$, analogously
to \ref{eq:neurons}. We can recursively deduce from the Lipschitz-continuity
of $\phi$ that for $l=2,\dots,L$: 
\begin{align}
\vert\mathbf{x}^{(1)}\vert_{1} & \le\vert W^{(1)}\mathbf{x}\vert_{1}+\vert v^{(1)}\vert_{1}\\
 & \le2rB(\vert\mathbf{x}\vert_{1}\lor1),\\
\vert\mathbf{x}^{(1)}-\tilde{\mathbf{x}}^{(1)}\vert_{1} & \le\vert W^{(1)}\mathbf{x}^{(0)}+v^{(1)}-\tilde{W}^{(1)}\tilde{\mathbf{x}}^{(0)}-\tilde{v}^{(1)}\vert_{1}\\
 & \le\eta2r(\vert\mathbf{x}\vert_{1}\lor1),\\
\vert\mathbf{x}^{(l)}\vert_{1} & \le\vert W^{(l)}\mathbf{x}^{(l-1)}\vert_{1}+\vert v^{(l)}\vert_{1}\\
 & \le2rB(\vert\mathbf{x}^{(l-1)}\vert_{1}\vee1)\qquad\text{and}\\
\vert\mathbf{x}^{(l)}-\tilde{\mathbf{x}}^{(l)}\vert_{1} & \le\vert W^{(l)}\mathbf{x}^{(l-1)}+v^{(l)}-\tilde{W}^{(l)}\tilde{\mathbf{x}}^{(l-1)}-\tilde{v}^{(l)}\vert_{1}\\
 & \le\vert(W^{(l)}-\tilde{W}^{(l)})\mathbf{x}^{(l-1)}\vert_{1}+\vert\tilde{W}^{(l)}(\mathbf{x}^{(l-1)}-\tilde{\mathbf{x}}^{(l-1)})\vert_{1}+\vert v^{(l)}-\tilde{v}^{(l)}\vert_{1}\\
 & \le\eta2r(\vert\mathbf{x}^{(l-1)}\vert_{1}\lor1)+rB\vert\mathbf{x}^{(l-1)}-\tilde{\mathbf{x}}^{(l-1)}\vert_{1}.
\end{align}
Therefore, 
\begin{align}
\vert\mathbf{x}^{(L)}\vert_{1} & \le(2rB)^{L-1}(\vert\mathbf{x}^{(1)}\vert_{1}\lor1)\\
 & \le(2rB)^{L}(\vert\mathbf{x}\vert_{1}\lor1)\qquad\text{and}\\
\vert\mathbf{x}^{(L)}-\tilde{\mathbf{x}}^{(L)}\vert_{1} & \le\eta2r\sum_{k=1}^{L-1}(rB)^{k-1}(\vert\mathbf{x}^{(L-k)}\vert_{1}\lor1)+(rB)^{L-1}\vert\mathbf{x}^{(1)}-\tilde{\mathbf{x}}^{(1)}\vert_{1}\\
 & \le\eta2^{(L+1)}r(\vert\mathbf{x}\vert_{1}\lor1)(rB)^{L-1}.
\end{align}
Since the clipping function $y\mapsto(-\Cf)\vee(y\wedge \Cf)$ has Lipschitz
constant $1$, we conclude 
\begin{align}
\vert f_{\theta}(\mathbf{x})-f_{\tilde{\theta}}(\mathbf{x})\vert & \le\vert g_{\theta}(\mathbf{x})-g_{\tilde{\theta}}(\mathbf{x})\vert\\
 & =\vert\mathbf{\mathbf{x}}^{(L+1)}-\tilde{\mathbf{x}}^{(L+1)}\vert\\
 & =\vert W^{(L+1)}\mathbf{x}^{(L)}+v^{(L+1)}-\tilde{W}^{(L+1)}\tilde{\mathbf{x}}^{(L)}-\tilde{v}^{(L+1)}\vert\\
 & \le\vert(W^{(L+1)}-\tilde{W}^{(L+1)})\mathbf{x}^{(L)}\vert+\vert\tilde{W}^{(L+1)}(\mathbf{x}^{(L)}-\tilde{\mathbf{x}}^{(L)})\vert+\vert v^{(L+1)}-\tilde{v}^{(L+1)}\vert\\
 & \le r\vert W^{(L+1)}-\tilde{W}^{(L+1)}\vert_{\infty}\vert\mathbf{x}^{(L)}\vert_{1}+r\vert\tilde{W}^{(L+1)}\vert_{\infty}\vert\mathbf{x}^{(L)}-\tilde{\mathbf{x}}^{(L)}\vert_{1}+\vert v^{(L+1)}-\tilde{v}^{(L+1)}\vert\\
 & \le\eta r(2rB)^{L}(\vert \mathbf{x}\vert_{1}\lor1)+\eta(rB)^{L}2^{L+1}(\vert\mathbf{x}\vert_{1}\lor1)+\eta\\
 & \le\eta4(2rB)^{L}(\vert\mathbf{x}\vert_{1}\lor1).
\end{align}
For shallow neural networks we proceed slightly differently. We have
\begin{align*}
    |\mathbf x^{(1)}|_\infty&\le|W^{(1)}|_\infty|\mathbf x|_1+|v^{(1)}|_\infty\le B(|\mathbf x|_1+1),\\
     |\mathbf x^{(1)}-\tilde{\mathbf x}^{(1)}|_1&\le|W^{(1)}-\tilde W^{(1)}|_1|\mathbf x|_\infty+|v^{(1)}-\tilde v^{(1)}|_1
\end{align*}
and thus
\begin{align*}
    \vert f_{\theta}(\mathbf{x})-f_{\tilde{\theta}}(\mathbf{x})\vert & \le\vert g_{\theta}(\mathbf{x})-g_{\tilde{\theta}}(\mathbf{x})\vert\\
 &\le\vert(W^{(2)}-\tilde{W}^{(2)})\mathbf{x}^{(1)}\vert+\vert\tilde{W}^{(2)}(\mathbf{x}^{(1)}-\tilde{\mathbf{x}}^{(1)})\vert+\vert v^{(2)}-\tilde{v}^{(2)}\vert\\
 &\le \vert W^{(2)}-\tilde{W}^{(2)}\vert_1|\mathbf{x}^{(1)}|_\infty+\vert\tilde{W}^{(2)}\vert_\infty\vert \mathbf{x}^{(1)}-\tilde{\mathbf{x}}^{(1)}\vert_1+\vert v^{(2)}-\tilde{v}^{(2)}\vert\\
 &\le B(|\mathbf x|_1+1)|\theta-\tilde\theta|_1.\tag*{\qed}
\end{align*}

\subsubsection{Proof of \ref{lem:aux_adaptive}}
We will show that
\begin{equation}
    \frac{\d \varrho_{r,\eta}}{\d\widecheck{\Pi}}=2^r(1-2^{-n})\frac{\d \varrho_{r,\eta}}{\d\Pi_r},\label{eq:rhodensity}
\end{equation}
from which we can deduce
\begin{align}
\KL(\varrho_{r,\eta}\mid\widecheck{\Pi})=\int\log\Big(\frac{\d\varrho_{r,\eta}}{\d\widecheck{\Pi}}\Big)\,\d\varrho_{r,\eta}=\int\log\Big(\frac{\d\varrho_{r,\eta}}{\d\Pi_r}\Big)\,\d\varrho_{r,\eta}+\log(2^{r}(1-2^{-n}))\le\KL(\varrho_{r,\eta}\mid\Pi_{r})+r.
\end{align}
Since $\varrho_{r,\eta}$ and $\Pi_r$ are product measures, their KL-divergence is equal to the sum of the KL-divergences in each of the $ \pardim_r$ factors. For each such factor, we are comparing
\[
\mathcal{U}([(\theta^{\ast}_r)_{i}-\eta,(\theta^{\ast}_r)_{i}+\eta]\cap[-B,B])\qquad\text{with}\qquad\mathcal{U}([-B,B]),
\]
where $(\theta^{\ast}_r)_i$ denotes the $i$-th entry of $\theta^\ast_r$.
The KL-divergence of these distributions is equal to
\[
\log\Big(\frac{\lebesgue([-B,B])}{\lebesgue([(\theta^{\ast}_r)_{i}-\eta,(\theta^{\ast}_r)_{i}+\eta]\cap[-B,B])}\Big)\le\log\Big(\frac{\lebesgue([-B,B])}{\lebesgue([0,\eta])}\Big)=\log(2B/\eta)
\]
and the lemma follows.

For \ref{eq:rhodensity}, note that $\varrho_{r, \eta}$ can only assign a positive probability to subsets $A\subseteq[-B, B]^{\pardim_r}$. Hence,
\begin{align}
    \varrho_{r,\eta}(A)=\int_A\frac{\d \varrho_{r,\eta}}{\d \widecheck{\Pi}}\,\d\widecheck{\Pi}=(1-2^{-n })^{-1}\sum_{l=1}^n 2^{-l}\int_A \frac{\d \varrho_{r,\eta}}{\d\widecheck{\Pi}}\,\d\Pi_l=(1-2^{-n })^{-1}2^{-r}\int_A \frac{\d\varrho_{r,\eta}}{\d\widecheck{\Pi}}\,\d\Pi_r.\tag*{{\qed}}
\end{align}
\bibliographystyle{apalike2}
\bibliography{references}

\end{document}

%% file: preamble.tex
\usepackage[T1]{fontenc}
\usepackage[latin9]{inputenc}
\usepackage{xcolor}
\usepackage[english]{babel}
\usepackage{float}
\usepackage{mathrsfs}
\usepackage{mathtools}
\usepackage{amsmath}
\usepackage{amsthm}
\usepackage{amssymb}
\usepackage{natbib}
\usepackage[unicode=true,pdfusetitle,
 bookmarks=true,bookmarksnumbered=false,bookmarksopen=false,
 breaklinks=false,pdfborder={0 0 1},backref=false,colorlinks=false]
 {hyperref}
\usepackage{capt-of}

\makeatletter

\floatstyle{ruled}
\newfloat{algorithm}{tbp}{loa}
\providecommand{\algorithmname}{Algorithm}
\floatname{algorithm}{\protect\algorithmname}

\numberwithin{equation}{section}
\newenvironment{lyxcode}
	{\par\begin{list}{}{
		\setlength{\rightmargin}{\leftmargin}
		\setlength{\listparindent}{0pt}
		\raggedright
		\setlength{\itemsep}{0pt}
		\setlength{\parsep}{0pt}
		\normalfont\ttfamily}%
	 \item[]}
	{\end{list}}
\theoremstyle{plain}
\newtheorem{thm}{\protect\theoremname}
\theoremstyle{plain}
\newtheorem{lem}[thm]{\protect\lemmaname}
\theoremstyle{plain}
\newcounter{assumptioncounter}
\newtheorem{assumption}{\protect\assumptionname}

\theoremstyle{remark}
\newtheorem{rem}[thm]{\protect\remarkname}
\theoremstyle{plain}
\newtheorem{cor}[thm]{\protect\corollaryname}
\theoremstyle{plain}
\newtheorem{prop}[thm]{\protect\propositionname}
\theoremstyle{plain}
\newtheorem{example}[thm]{\protect\examplename}

\@ifundefined{date}{}{\date{}}

\usepackage{bbm}

\usepackage{mathtools}

\allowdisplaybreaks

\usepackage{hyperref}

\usepackage[capitalise,nameinlink]{cleveref}

\AtBeginDocument{
\let\ref\cref
}

\usepackage{autonum}

\makeatletter
\newcommand{\restore@Environment}[1]{%
  \AtBeginDocument{%
    \csletcs{#1*}{#1}%
    \csletcs{end#1*}{end#1}%
  }%
}
\forcsvlist\restore@Environment{alignat,equation,gather,multline,flalign,align}
\makeatother

\usepackage{tikz}

\crefname{equation}{}{}

\crefalias{thm}{theorem}

\usepackage{etoolbox}

\ifcsmacro{defn}{}{\newtheorem{defn}{test}}
\ifcsmacro{thm}{}{\newtheorem{thm}{test}}
\ifcsmacro{lem}{}{\newtheorem{lem}{test}}
\ifcsmacro{prop}{}{\newtheorem{prop}{test}}
\ifcsmacro{cor}{}{\newtheorem{cor}{test}}
\ifcsmacro{assumption}{}{\newtheorem*{assumption*}{test}}
\ifcsmacro{example}{}{\newtheorem{example}{Example}}

\usepackage{enumitem}
\setlist[enumerate,1]{label=(\alph*)}
\setlist[enumerate,2]{label=(\roman*)}

\let\oldthm\thm
\renewcommand{\thm}{%
\crefalias{thm}{theorem}
\oldthm
}

\let\oldlem\lem
\renewcommand{\lem}{%
\crefalias{thm}{lemma}
\oldlem
}

\let\oldprop\prop
\renewcommand{\prop}{%
\crefalias{thm}{proposition}
\oldprop
}

\let\oldcor\cor
\renewcommand{\cor}{%
\crefalias{thm}{corollary}
\oldcor
}

\let\oldexample\example
\renewcommand{\example}{%
\crefalias{thm}{example}
\oldexample
}

\let\oldassumption\assumption
\renewcommand{\assumption}{%
\crefalias{thm}{assumption}
\oldassumption
}

\let\oldrem\rem
\renewcommand{\rem}{%
\crefalias{thm}{remark}
\oldrem
}

\usepackage{crossreftools}
\let\ORGhypersetup\hypersetup
\protected\def\hypersetup{\ORGhypersetup}
\crefname{assumption}{Assumption}{Assumptions}

\makeatother

\addto\captionsamerican{\renewcommand{\assumptionname}{Assumption}}
\addto\captionsamerican{\renewcommand{\corollaryname}{Corollary}}
\addto\captionsamerican{\renewcommand{\lemmaname}{Lemma}}
\addto\captionsamerican{\renewcommand{\propositionname}{Proposition}}
\addto\captionsamerican{\renewcommand{\remarkname}{Remark}}
\addto\captionsamerican{\renewcommand{\examplename}{Example}}
\addto\captionsamerican{\renewcommand{\theoremname}{Theorem}}
\addto\captionsenglish{\renewcommand{\algorithmname}{Algorithm}}
\addto\captionsenglish{\renewcommand{\assumptionname}{Assumption}}
\addto\captionsenglish{\renewcommand{\corollaryname}{Corollary}}
\addto\captionsenglish{\renewcommand{\lemmaname}{Lemma}}
\addto\captionsenglish{\renewcommand{\propositionname}{Proposition}}
\addto\captionsenglish{\renewcommand{\remarkname}{Remark}}
\addto\captionsenglish{\renewcommand{\theoremname}{Theorem}}
\addto\captionsenglish{\renewcommand{\examplename}{Example}}
\providecommand{\assumptionname}{Assumption}
\providecommand{\corollaryname}{Corollary}
\providecommand{\lemmaname}{Lemma}
\providecommand{\propositionname}{Proposition}
\providecommand{\remarkname}{Remark}
\providecommand{\theoremname}{Theorem}
\providecommand{\examplename}{Example}

%% file: math_commands.tex
\AtBeginDocument

\DeclareFontFamily{U}{mathx}{\hyphenchar\font45}
\DeclareFontShape{U}{mathx}{m}{n}{
	<5> <6> <7> <8> <9> <10>
	<10.95> <12> <14.4> <17.28> <20.74> <24.88>
	mathx10
}{}
\DeclareSymbolFont{mathx}{U}{mathx}{m}{n}
\DeclareFontSubstitution{U}{mathx}{m}{n}
\DeclareMathAccent{\widecheck}{0}{mathx}{"71}
\DeclareMathAccent{\wideparen}{0}{mathx}{"75}

\global\long\def\N{\mathbb{N}}%
\global\long\def\R{\mathbb{R}}%
\global\long\def\P{\mathbb{P}}%
\global\long\def\E{\mathbb{E}}%
\global\long\def\1{\mathbbm1}%
\global\long\def\as{\ a.s.}%
\global\long\def\d{\mathrm{d}}%
\global\long\def\e{\mathrm{e}}%
\global\long\def\diff#1#2{\frac{\mathrm{d}#1}{\mathrm{d}#2}}%
\global\long\def\lebesgue{\lambda\mkern-13mu  \lambda}%
\global\long\def\argmin#1{\operatorname*{arg\,\min}_{#1}}%
\global\long\def\eps{\varepsilon}%
\global\long\def\theta{\vartheta}%
\global\long\def\le{\leqslant}%
\global\long\def\ge{\geqslant}%
\global\long\def\KL{\operatorname{KL}}%
\foreignlanguage{english}{}
\global\long\def\subset{\subseteq}%
\global\long\def\diam{\operatorname{diam}}%
\global\long\def\comp{\mathsf{c}}%
\foreignlanguage{english}{}
\global\long\def\tilde#1{\widetilde{#1}}%
\foreignlanguage{english}{}
\global\long\def\hat#1{\widehat{#1}}%
\renewcommand{\Xi}{Q}

\newcommand{\rate}{\varrho_n}
\newcommand{\id}[1]{I_{#1}}
\newcommand{\pardim}{Q}
\newcommand{\inputdim}{d}

\newcommand{\Cf}{C_f} 
\newcommand{\Ceps}{C_\varepsilon}  
\newcommand{\cshallow}{C_\mathrm{shallow}}  

\newcommand{\hdepth}{p} 

\newcommand{\qnorm}{p} 

\newcommand{\Cxi}{C_\tau} 

\newcommand{\K}[2]{K_{#1}#2}

\newcommand{\cholder}{C_0}

%% file: main.bbl
\begin{thebibliography}{}

\bibitem[Alexos et~al., 2022]{ssgMCMC2022}
Alexos, A., Boyd, A.~J., \& Mandt, S. (2022).
\newblock Structured stochastic gradient {MCMC}.
\newblock In {\em International Conference on Machine Learning}  (pp.\
  414--434).

\bibitem[Alquier, 2024]{Alquier2021}
Alquier, P. (2024).
\newblock User-friendly introduction to {PAC}-{B}ayes bounds.
\newblock {\em Foundations and Trends in Machine Learning}, 17(2), 174--303.

\bibitem[Alquier \& Biau, 2013]{alquier2013}
Alquier, P. \& Biau, G. (2013).
\newblock Sparse single-index model.
\newblock {\em Journal of Machine Learning Research}, 14, 243--280.

\bibitem[Alquier \& Lounici, 2011]{Alquier2011}
Alquier, P. \& Lounici, K. (2011).
\newblock {PAC}-{B}ayesian bounds for sparse regression estimation with
  exponential weights.
\newblock {\em Electronic Journal of Statistics}, 5.

\bibitem[Andrieu \& Roberts, 2009]{AndrieuRoberts2009}
Andrieu, C. \& Roberts, G.~O. (2009).
\newblock {The pseudo-marginal approach for efficient {M}onte {C}arlo
  computations}.
\newblock {\em The Annals of Statistics}, 37(2), 697--725.

\bibitem[Anthony \& Bartlett, 1999]{AnthonyBartlett1999}
Anthony, M. \& Bartlett, P.~L. (1999).
\newblock {\em Neural network learning: Theoretical foundations}.
\newblock Cambridge University Press.

\bibitem[Audibert, 2004]{Audibert2004}
Audibert, J.-Y. (2004).
\newblock Aggregated estimators and empirical complexity for least square
  regression.
\newblock {\em Annales de l'Institut Henri Poincar\'{e}. Probabilit\'{e}s et
  Statistiques}, 40(6), 685--736.

\bibitem[Audibert, 2009]{Audibert2009}
Audibert, J.-Y. (2009).
\newblock Fast learning rates in statistical inference through aggregation.
\newblock {\em The Annals of Statistics}, 37(4), 1591--1646.

\bibitem[Audibert \& Catoni, 2011]{Audibert2011}
Audibert, J.-Y. \& Catoni, O. (2011).
\newblock Robust linear least squares regression.
\newblock {\em The Annals of Statistics}, 39(5), 2766--2794.

\bibitem[Bach, 2017]{bach2017breaking}
Bach, F. (2017).
\newblock Breaking the curse of dimensionality with convex neural networks.
\newblock {\em Journal of Machine Learning Research}, 18(19), 1--53.

\bibitem[Bardenet et~al., 2017]{BardenetEtAl2017}
Bardenet, R., Doucet, A., \& Holmes, C. (2017).
\newblock On {M}arkov chain {M}onte {C}arlo methods for tall data.
\newblock {\em Journal of Machine Learning Research}, 18(47).

\bibitem[Barron, 1993]{barron1993universal}
Barron, A.~R. (1993).
\newblock Universal approximation bounds for superpositions of a sigmoidal
  function.
\newblock {\em IEEE Transactions on Information theory}, 39(3), 930--945.

\bibitem[Bauer \& Kohler, 2019]{BauerKohler2019}
Bauer, B. \& Kohler, M. (2019).
\newblock On deep learning as a remedy for the curse of dimensionality in
  nonparametric regression.
\newblock {\em The Annals of Statistics}, 47(4), 2261--2285.

\bibitem[Besag, 1994]{Besag1994}
Besag, J. (1994).
\newblock Comments on ``{R}epresentations of knowledge in complex systems'' by
  {U}. {G}renander and {M.I.} {M}iller.
\newblock {\em Journal of the Royal Statistical Society. Series B.
  Methodological}, 56(4), 549--581.

\bibitem[Biggs \& Guedj, 2021]{biggs2021}
Biggs, F. \& Guedj, B. (2021).
\newblock Differentiable {PAC}-{B}ayes objectives with partially aggregated
  neural networks.
\newblock {\em Entropy}, 23(10), 1280.

\bibitem[Biggs \& Guedj, 2022]{biggs2022}
Biggs, F. \& Guedj, B. (2022).
\newblock Non-vacuous generalisation bounds for shallow neural networks.
\newblock In {\em International Conference on Machine Learning}  (pp.\
  1963--1981).: PMLR.

\bibitem[Biggs \& Guedj, 2023]{biggs2023}
Biggs, F. \& Guedj, B. (2023).
\newblock Tighter {PAC}-{B}ayes generalisation bounds by leveraging example
  difficulty.
\newblock In {\em International Conference on Artificial Intelligence and
  Statistics}  (pp.\ 8165--8182).: PMLR.

\bibitem[Bissiri et~al., 2016]{Bissiri2016}
Bissiri, P.~G., Holmes, C.~C., \& Walker, S.~G. (2016).
\newblock A general framework for updating belief distributions.
\newblock {\em Journal of the Royal Statistical Society. Series B. Statistical
  Methodology}, 78(5), 1103--1130.

\bibitem[Blei et~al., 2017]{BleiEtal2017}
Blei, D.~M., Kucukelbir, A., \& McAuliffe, J.~D. (2017).
\newblock Variational inference: A review for statisticians.
\newblock {\em Journal of the American Statistical Association}, 112(518),
  859--877.

\bibitem[Castillo \& Egels, 2024]{castillo2024posterior}
Castillo, I. \& Egels, P. (2024).
\newblock Posterior and variational inference for deep neural networks with
  heavy-tailed weights.
\newblock {\em arXiv preprint arXiv:2406.03369}.

\bibitem[Castillo \& Nickl, 2014]{CastilloNickl2014}
Castillo, I. \& Nickl, R. (2014).
\newblock On the {B}ernstein-von {M}ises phenomenon for nonparametric {B}ayes
  procedures.
\newblock {\em The Annals of Statistics}, 42(5), 1941--1969.

\bibitem[Catoni, 2004]{catoni2004}
Catoni, O. (2004).
\newblock {\em Statistical learning theory and stochastic optimization}.
\newblock Springer.

\bibitem[Catoni, 2007]{catoni2007}
Catoni, O. (2007).
\newblock {\em {PAC}-{B}ayesian supervised classification: The thermodynamics
  of statistical learning}, volume~56 of {\em Lecture Notes-Monograph Series}.
\newblock Institute of Mathematical Statistics.

\bibitem[Cheng \& Bartlett, 2018]{cheng18a}
Cheng, X. \& Bartlett, P. (2018).
\newblock Convergence of {L}angevin {MCMC} in {KL}-divergence.
\newblock In {\em Proceedings of Algorithmic Learning Theory}, volume~83  (pp.\
  186--211).

\bibitem[Ch{\'e}rief-Abdellatif, 2020]{cherief2020}
Ch{\'e}rief-Abdellatif, B.-E. (2020).
\newblock Convergence rates of variational inference in sparse deep learning.
\newblock In {\em International Conference on Machine Learning}  (pp.\
  1831--1842).

\bibitem[Chewi, 2024]{chewi2024}
Chewi, S. (2024).
\newblock {\em Log-Concave Sampling}.
\newblock unfinished draft, \url{https://chewisinho.github.io/main.pdf}.

\bibitem[Cobb \& Jalaian, 2021]{cobb2020scaling}
Cobb, A.~D. \& Jalaian, B. (2021).
\newblock Scaling hamiltonian monte carlo inference for bayesian neural
  networks with symmetric splitting.
\newblock {\em Uncertainty in Artificial Intelligence}.

\bibitem[Csiszar, 1975]{csiszar1975}
Csiszar, I. (1975).
\newblock {$I$-Divergence Geometry of Probability Distributions and
  Minimization Problems}.
\newblock {\em The Annals of Probability}, 3(1), 146--158.

\bibitem[Dalalyan \& Riou-Durand, 2020]{Dalalyan2020}
Dalalyan, A.~S. \& Riou-Durand, L. (2020).
\newblock On sampling from a log-concave density using kinetic {L}angevin
  diffusions.
\newblock {\em Bernoulli. Official Journal of the Bernoulli Society for
  Mathematical Statistics and Probability}, 26(3), 1956--1988.

\bibitem[Dalalyan \& Tsybakov, 2007]{dalalyan2007}
Dalalyan, A.~S. \& Tsybakov, A.~B. (2007).
\newblock Aggregation by exponential weighting and sharp oracle inequalities.
\newblock In {\em International Conference on Computational Learning Theory}
  (pp.\ 97--111).: Springer.

\bibitem[Deng et~al., 2020a]{Deng2020replicaexchangeSGMCMC}
Deng, W., Feng, Q., Gao, L., Liang, F., \& Lin, G. (2020a).
\newblock Non-convex learning via replica exchange stochastic gradient {MCMC}.
\newblock In {\em Proceedings of the 37th International Conference on Machine
  Learning}, volume 119 of {\em Proceedings of Machine Learning Research}
  (pp.\ 2474--2483).

\bibitem[Deng et~al., 2022]{Deng2022ISCGLD}
Deng, W., Liang, S., Hao, B., Lin, G., \& Liang, F. (2022).
\newblock Interacting contour stochastic gradient {L}angevin dynamics.
\newblock In {\em The Tenth International Conference on Learning
  Representations}.

\bibitem[Deng et~al., 2020b]{Deng2020contourSGMCMC}
Deng, W., Lin, G., \& Liang, F. (2020b).
\newblock A contour stochastic gradient {L}angevin dynamics algorithm for
  simulations of multi-modal distributions.
\newblock In {\em Advances in Neural Information Processing Systems 33: Annual
  Conference on Neural Information Processing Systems}.

\bibitem[DeVore et~al., 2021]{devoreEtAl2021}
DeVore, R., Hanin, B., \& Petrova, G. (2021).
\newblock Neural network approximation.
\newblock {\em Acta Numerica}, 30, 327--444.

\bibitem[Donsker \& Varadhan, 1976]{donsker1976}
Donsker, M.~D. \& Varadhan, S.~S. (1976).
\newblock Asymptotic evaluation of certain markov process expectations for
  large time--{III}.
\newblock {\em Communications on pure and applied Mathematics}, 29(4),
  389--461.

\bibitem[Duane et~al., 1987]{duaneEtAl1987}
Duane, S., Kennedy, A.~D., Pendleton, B.~J., \& Roweth, D. (1987).
\newblock Hybrid {M}onte {C}arlo.
\newblock {\em Physics Letters B}, 195(2), 216--222.

\bibitem[Dziugaite \& Roy, 2017]{Dziugaite2017}
Dziugaite, G.~K. \& Roy, D.~M. (2017).
\newblock Computing nonvacuous generalization bounds for deep (stochastic)
  neural networks with many more parameters than training data.
\newblock {\em Proceedings of the Thirty-Third Conference on Uncertainty in
  Artificial Intelligence}.

\bibitem[Franssen \& Szab{\'o}, 2022]{franssenSzabo2022}
Franssen, S. \& Szab{\'o}, B. (2022).
\newblock Uncertainty quantification for nonparametric regression using
  empirical {B}ayesian neural networks.
\newblock {\em arXiv preprint arXiv:2204.12735}.

\bibitem[Freund et~al., 2022]{Freund2022}
Freund, Y., Ma, Y.-A., \& Zhang, T. (2022).
\newblock When is the convergence time of {L}angevin algorithms dimension
  independent? {A} composite optimization viewpoint.
\newblock {\em Journal of Machine Learning Research}, 23, 1--32.

\bibitem[Ghosal \& van~der Vaart, 2017]{GhosalvanderVaart2017}
Ghosal, S. \& van~der Vaart, A. (2017).
\newblock {\em Fundamentals of nonparametric {B}ayesian inference}, volume~44
  of {\em Cambridge Series in Statistical and Probabilistic Mathematics}.
\newblock Cambridge University Press.

\bibitem[Goodfellow et~al., 2016]{goodfellowEtAl2016}
Goodfellow, I., Bengio, Y., \& Courville, A. (2016).
\newblock {\em Deep learning}.
\newblock MIT Press.

\bibitem[Guedj, 2019]{Guedj2019}
Guedj, B. (2019).
\newblock A primer on {PAC}-{B}ayesian learning.
\newblock {\em Proceedings of the 2nd congress of the Soci{\'e}t{\'e}
  Math{\'e}matique de France}, (pp.\ 391--414).

\bibitem[Guedj \& Alquier, 2013]{Guedj2013}
Guedj, B. \& Alquier, P. (2013).
\newblock {PAC}-{B}ayesian estimation and prediction in sparse additive models.
\newblock {\em Electronic Journal of Statistics}, 7, 264--291.

\bibitem[Guedj \& Robbiano, 2018]{Guedj2018}
Guedj, B. \& Robbiano, S. (2018).
\newblock {PAC}-{B}ayesian high dimensional bipartite ranking.
\newblock {\em Journal of Statistical Planning and Inference}, 196, 70--86.

\bibitem[Hellstr{\"o}m et~al., 2025]{hellstroem2025}
Hellstr{\"o}m, F., Durisi, G., Guedj, B., \& Raginsky, M. (2025).
\newblock Generalization bounds: Perspectives from information theory and
  {PAC}-{B}ayes.
\newblock {\em Foundations and Trends in Machine Learning}, 18(1), 1--223.

\bibitem[Hoffmann \& Nickl, 2011]{HoffmannNickl2011}
Hoffmann, M. \& Nickl, R. (2011).
\newblock On adaptive inference and confidence bands.
\newblock {\em The Annals of Statistics}, 39(5), 2383--2409.

\bibitem[{Kingma} \& {Ba}, 2014]{Adam2014}
{Kingma}, D.~P. \& {Ba}, J. (2014).
\newblock Adam: A method for stochastic optimization.
\newblock {\em arXiv preprint arXiv:1412.6980}.

\bibitem[Knapik et~al., 2011]{KnapiketAl2011}
Knapik, B.~T., van~der Vaart, A.~W., \& van Zanten, J.~H. (2011).
\newblock {B}ayesian inverse problems with {G}aussian priors.
\newblock {\em The Annals of Statistics}, 39(5), 2626--2657.

\bibitem[Kohler \& Langer, 2021]{KohlerLanger2021}
Kohler, M. \& Langer, S. (2021).
\newblock On the rate of convergence of fully connected deep neural network
  regression estimates.
\newblock {\em The Annals of Statistics}, 49(4), 2231--2249.

\bibitem[Li et~al., 2016]{Li2016SGLD}
Li, C., Chen, C., Carlson, D.~E., \& Carin, L. (2016).
\newblock Preconditioned stochastic gradient {L}angevin dynamics for deep
  neural networks.
\newblock In {\em Proceedings of the Thirtieth {AAAI} Conference on Artificial
  Intelligence}  (pp.\ 1788--1794).

\bibitem[Maclaurin \& Adams, 2014]{MaclaurinAdams2014}
Maclaurin, D. \& Adams, R.~P. (2014).
\newblock Firefly {M}onte {C}arlo: Exact {MCMC} with subsets of data.
\newblock In {\em Proceedings of the 30th Conference on Uncertainty in
  Artificial Intelligence}.

\bibitem[Massart, 2007]{Massart2007}
Massart, P. (2007).
\newblock {\em Concentration inequalities and model selection}, volume 1896 of
  {\em Lecture Notes in Mathematics}.
\newblock Springer.

\bibitem[McAllester, 1999a]{McAllester1999a}
McAllester, D.~A. (1999a).
\newblock {PAC}-{B}ayesian model averaging.
\newblock In {\em Proceedings of the {T}welfth {A}nnual {C}onference on
  {C}omputational {L}earning {T}heory}  (pp.\ 164--170).

\bibitem[McAllester, 1999b]{McAllester1999b}
McAllester, D.~A. (1999b).
\newblock Some {PAC}-{B}ayesian theorems.
\newblock {\em Machine Learning}, 37(3), 355--363.

\bibitem[Neal, 2011]{neal2011}
Neal, R.~M. (2011).
\newblock {MCMC} using {H}amiltonian dynamics, {I}n: \textit{Handbook of
  {M}arkov chain {M}onte {C}arlo}.
\newblock (pp.\ 113--163).

\bibitem[Nickl \& Wang, 2022]{nickl2022}
Nickl, R. \& Wang, S. (2022).
\newblock On polynomial-time computation of high-dimensional posterior measures
  by {L}angevin-type algorithms.
\newblock {\em Journal of the European Mathematical Society}.

\bibitem[Patterson \& Teh, 2013]{Patterson2013SGRLD}
Patterson, S. \& Teh, Y.~W. (2013).
\newblock Stochastic gradient {R}iemannian {L}angevin dynamics on the
  probability simplex.
\newblock In {\em Advances in Neural Information Processing Systems 26}  (pp.\
  3102--3110).

\bibitem[P{\'e}rez-Ortiz et~al., 2021]{perezEtAl2021}
P{\'e}rez-Ortiz, M., Rivasplata, O., Shawe-Taylor, J., \& Szepesv{\'a}ri, C.
  (2021).
\newblock Tighter risk certificates for neural networks.
\newblock {\em Journal of Machine Learning Research}, 22.

\bibitem[Polson \& Ro{\v{c}}kov{\'a}, 2018]{polsonRockova2018}
Polson, N.~G. \& Ro{\v{c}}kov{\'a}, V. (2018).
\newblock Posterior concentration for sparse deep learning.
\newblock {\em Advances in Neural Information Processing Systems}, 31,
  938--949.

\bibitem[Ray \& Szab{\'o}, 2022]{raySzabo2022}
Ray, K. \& Szab{\'o}, B. (2022).
\newblock Variational {B}ayes for high-dimensional linear regression with
  sparse priors.
\newblock {\em Journal of the American Statistical Association}, 117(539),
  1270--1281.

\bibitem[Robert \& Casella, 2004]{RobertCasella2004}
Robert, C.~P. \& Casella, G. (2004).
\newblock {\em {M}onte {C}arlo statistical methods}.
\newblock Springer, second edition.

\bibitem[Roberts \& Tweedie, 1996a]{roberts1996b}
Roberts, G.~O. \& Tweedie, R.~L. (1996a).
\newblock {Exponential convergence of of {L}angevin distributions and their
  discrete approximations}.
\newblock {\em Bernoulli. Official Journal of the Bernoulli Society for
  Mathematical Statistics and Probability}, 2(4), 341--363.

\bibitem[Roberts \& Tweedie, 1996b]{Roberts1996}
Roberts, G.~O. \& Tweedie, R.~L. (1996b).
\newblock Geometric convergence and central limit theorems for multidimensional
  {H}astings and {M}etropolis algorithms.
\newblock {\em Biometrika}, 83(1), 95--110.

\bibitem[Rousseau \& Szab\'{o}, 2020]{RousseauSzabo2020}
Rousseau, J. \& Szab\'{o}, B. (2020).
\newblock Asymptotic frequentist coverage properties of {B}ayesian credible
  sets for sieve priors.
\newblock {\em The Annals of Statistics}, 48(4), 2155--2179.

\bibitem[Schmidhuber, 2015]{schmidhuber2015}
Schmidhuber, J. (2015).
\newblock Deep learning in neural networks: {A}n overview.
\newblock {\em Neural networks}, 61, 85--117.

\bibitem[Schmidt-Hieber, 2020]{schmidthieber2020}
Schmidt-Hieber, J. (2020).
\newblock Nonparametric regression using deep neural networks with {R}e{LU}
  activation function.
\newblock {\em The Annals of Statistics}, 48(4), 1875--1897.

\bibitem[Shawe-Taylor \& Williamson, 1997]{ShaweTaylor1997}
Shawe-Taylor, J. \& Williamson, R.~C. (1997).
\newblock A {PAC} analysis of a {B}ayesian estimator.
\newblock In {\em Proceedings of the {T}enth {A}nnual {C}onference on
  {C}omputational {L}earning {Theory}}  (pp.\ 2--9).

\bibitem[Steffen \& Trabs, 2025]{SteffenTrabs2025}
Steffen, M.~F. \& Trabs, M. (2025).
\newblock A {PAC}-{B}ayes oracle inequality for sparse neural networks.
\newblock {\em Springer Proceedings in Mathematics \& Statistics, to appear}.

\bibitem[Szab\'{o} et~al., 2015]{SzaboEtAl2015}
Szab\'{o}, B., van~der Vaart, A.~W., \& van Zanten, J.~H. (2015).
\newblock Frequentist coverage of adaptive nonparametric {B}ayesian credible
  sets.
\newblock {\em The Annals of Statistics}, 43(4), 1391--1428.

\bibitem[Tinsi \& Dalalyan, 2022]{tinsiDalalyan2022}
Tinsi, L. \& Dalalyan, A. (2022).
\newblock Risk bounds for aggregated shallow neural networks using {G}aussian
  priors.
\newblock In {\em Proceedings of Thirty Fifth Conference on Learning Theory},
  volume 178 of {\em Proceedings of Machine Learning Research}  (pp.\
  227--253).: PMLR.

\bibitem[Tung-Yu~Wu \& Wong, 2022]{wu2022}
Tung-Yu~Wu, Y. X. R.~W. \& Wong, W.~H. (2022).
\newblock Mini-batch {M}etropolis-{H}astings with reversible {SGLD} proposal.
\newblock {\em Journal of the American Statistical Association}, 117(537),
  386--394.

\bibitem[Welling \& Teh, 2011]{Welling2011bayesianSGLD}
Welling, M. \& Teh, Y.~W. (2011).
\newblock {B}ayesian learning via stochastic gradient {L}angevin dynamics.
\newblock In {\em Proceedings of the 28th International Conference on Machine
  Learning}  (pp.\ 681--688).

\bibitem[Yang \& Zhou, 2024]{yang2024nonparametric}
Yang, Y. \& Zhou, D.-X. (2024).
\newblock Nonparametric regression using over-parameterized shallow {ReLU}
  neural networks.
\newblock {\em Journal of Machine Learning Research}, 25, 1--35.

\bibitem[Yarotsky, 2017]{yarotsky2017}
Yarotsky, D. (2017).
\newblock Error bounds for approximations with deep {R}e{LU} networks.
\newblock {\em Neural Networks}, 94, 103--114.

\bibitem[Zhang \& Zhou, 2020]{ZhangZhou2020}
Zhang, A.~Y. \& Zhou, H.~H. (2020).
\newblock Theoretical and computational guarantees of mean field variational
  inference for community detection.
\newblock {\em The Annals of Statistics}, 48(5), 2575--2598.

\bibitem[Zhang \& Gao, 2020]{ZhangGao2020}
Zhang, F. \& Gao, C. (2020).
\newblock Convergence rates of variational posterior distributions.
\newblock {\em The Annals of Statistics}, 48(4), 2180--2207.

\bibitem[Zhang et~al., 2020]{Zhang2020cyclicSGMCMC}
Zhang, R., Li, C., Zhang, J., Chen, C., \& Wilson, A.~G. (2020).
\newblock Cyclical stochastic gradient {MCMC} for {B}ayesian deep learning.
\newblock In {\em 8th International Conference on Learning Representations}.

\bibitem[Zhang, 2006]{zhang2006}
Zhang, T. (2006).
\newblock Information-theoretic upper and lower bounds for statistical
  estimation.
\newblock {\em IEEE Transactions on Information Theory}, 52(4), 1307--1321.

\bibitem[Zhou et~al., 2019]{Zhou2018}
Zhou, W., Veitch, V., Austern, M., Adams, R.~P., \& Orbanz, P. (2019).
\newblock Non-vacuous generalization bounds at the imagenet scale: A
  {PAC}-{B}ayesian compression approach.
\newblock In {\em The Seventh International Conference on Learning
  Representations}.

\end{thebibliography}
